%% file: main.tex
\documentclass[11pt]{article} 

\usepackage[letterpaper,margin=1in]{geometry}

\usepackage[letterpaper,margin=1in]{geometry}

\def\colorful{1}
\ifnum\colorful=1

\else

\fi

\newif\ifred
\redtrue 

\usepackage[T1]{fontenc}
\usepackage{microtype}

\usepackage{titlesec}
\titleformat*{\paragraph}{\bfseries}

\usepackage{natbib}
\usepackage{amsthm,amsmath,amssymb,amsfonts,amssymb,mathtools}
\usepackage{amsmath,amssymb,amsfonts,amssymb,mathtools}
\usepackage{xcolor}
\usepackage{empheq}
\usepackage{dsfont}
\usepackage{mathrsfs}
\usepackage{graphicx}
\usepackage{enumitem}
\usepackage{aliascnt} 
\usepackage{xspace}
\usepackage{bbm}

\usepackage{caption}
\usepackage{tikz}
\usetikzlibrary{patterns}
\usetikzlibrary{patterns}
\usetikzlibrary{arrows,shapes,automata,backgrounds,petri,positioning}
\usetikzlibrary{shadows}
\usetikzlibrary{calc}
\usetikzlibrary{spy}
\usetikzlibrary{angles, quotes}
\usetikzlibrary{matrix}
\usepackage{pgf,pgfplots}
\usepackage{pgfmath,pgffor}
\pgfplotsset{compat=1.17}

\usepackage{framed}
 \usepackage{algorithm}
\usepackage[noend]{algpseudocode}

\definecolor[named]{ACMBlue}{cmyk}{1,0.1,0,0.1}
\definecolor[named]{ACMYellow}{cmyk}{0,0.16,1,0}
\definecolor[named]{ACMOrange}{cmyk}{0,0.42,1,0.01}
\definecolor[named]{ACMRed}{cmyk}{0,0.90,0.86,0}
\definecolor[named]{ACMLightBlue}{cmyk}{0.49,0.01,0,0}
\definecolor[named]{ACMGreen}{cmyk}{0.20,0,1,0.19}
\definecolor[named]{ACMPurple}{cmyk}{0.55,1,0,0.15}
\definecolor[named]{ACMDarkBlue}{cmyk}{1,0.58,0,0.21}

\usepackage[colorlinks,citecolor=blue,linkcolor=magenta,bookmarks=true]{hyperref}
 \usepackage[nameinlink]{cleveref}
\creflabelformat{ineq}{#2{\upshape(#1)}#3}
\crefname{sub}{Subsection}{Subsection}
\creflabelformat{Subsection}{#2{\upshape(#1)}#3}
\crefname{sdp}{SDP}{SDP}
\creflabelformat{sdp}{#2{\upshape(#1)}#3}
\crefname{lp}{LP}{LP}
\creflabelformat{lp}{#2{\upshape(#1)}#3}
\usepackage{aliascnt}
\usepackage{cleveref}
\crefname{ineq}{Inequality}{Inequality}
\creflabelformat{ineq}{#2{\upshape(#1)}#3}
\crefname{sub}{Subsection}{Subsection}
\creflabelformat{Subsection}{#2{\upshape(#1)}#3}
\crefname{sdp}{SDP}{SDP}
\creflabelformat{sdp}{#2{\upshape(#1)}#3}
\crefname{lp}{LP}{LP}
\creflabelformat{lp}{#2{\upshape(#1)}#3}



\newtheorem{theorem}{Theorem}[section]

\newtheorem{lemma}{Lemma}
\newtheorem{informal theorem}[theorem]{Theorem (informal statement)}

\newtheorem{fact}[theorem]{Fact}

\newtheorem{assumption}[theorem]{Assumption}
\newtheorem{definition}[theorem]{Definition}

\newcommand{\lp}{\left}
\newcommand{\rp}{\right}
\newcommand\norm[1]{\left\| #1 \right\|}

\renewcommand\vec[1]{\mathbf{#1}}
\DeclareMathOperator*{\pr}{\mathbf{Pr}}
\DeclareMathOperator*{\E}{\mathbf{E}}
\newcommand{\proj}{\mathrm{proj}}

\newcommand{\normal}{\mathcal{N}}

\newcommand{\bx}{\mathbf{x}}
\newcommand{\by}{\mathbf{y}}
\newcommand{\bv}{\mathbf{v}}
\newcommand{\bu}{\mathbf{u}}
\newcommand{\bz}{\mathbf{z}}
\newcommand{\bw}{\mathbf{w}}

\newcommand{\bI}{\mathbf{I}}
\newcommand{\bV}{\mathbf{V}}

\newcommand{\sym}{\mathrm{sym}}

\newcommand{\err}{\mathrm{err}}

\newcommand{\B}{\mathbb{B}}

\newcommand{\p}{\mathbf{P}}

\newcommand{\R}{\mathbb{R}}

\newcommand{\Z}{\mathbb{Z}}
\newcommand{\N}{\mathbb{N}}
\newcommand{\eps}{\epsilon}

\newcommand{\poly}{\mathrm{poly}}

\newcommand{\D}{\mathcal{D}}

\newcommand{\sgn}{\mathrm{sign}}
\newcommand{\sign}{\mathrm{sign}}

\newcommand{\Ind}{\mathds{1}}

\newcommand{\littlesum}{\mathop{\textstyle \sum}}

\newcommand{\bb}{\mathbf{b}}
\newcommand{\be}{\mathbf{e}}

\newcommand{\bs}{\mathbf{s}}
\newcommand{\bg}{\mathbf{g}}

\newcommand{\x}{\vec x}

\newcommand{\lamdba}{\lambda}

\newcommand{\iid}{{i.i.d.}\ }
\newcommand{\abs}[1]{\lp| #1 \rp|}
\newcommand{\A}{\mathcal{A}}

\newcommand{\card}[1]{|#1|}

\newcommand{\algparbox}[1]{\parbox[t]{\dimexpr\linewidth-\algorithmicindent}{#1\strut}}

\renewcommand\Pr{\pr}

\begin{document}

\title{Learning Intersections of Two Margin Halfspaces\\ under Factorizable Distributions \footnote{Appeared in COLT 2025}}

\author{
Ilias Diakonikolas\thanks{Supported by NSF Medium Award CCF-2107079 and an H.I. Romnes Faculty Fellowship.}\\
University of Wisconsin-Madison\\
\texttt{ilias@cs.wisc.edu}
\and
Mingchen Ma\thanks{Supported by NSF Award  CCF-2144298 (CAREER).}\\
University of Wisconsin-Madison\\
\texttt{mingchen@cs.wisc.edu}
\and
Lisheng Ren\thanks{Supported in part by NSF Medium Award CCF-2107079.}\\
University of Wisconsin-Madison\\
\texttt{lren29@wisc.edu}
\and
Christos Tzamos\thanks{Supported in part by NSF Award  CCF-2144298 (CAREER).}\\
University of Athens and Archimedes AI\\
\texttt{ctzamos@gmail.com}
}

\maketitle
\thispagestyle{empty}

\begin{abstract}%

Learning intersections of halfspaces is a central problem in Computational Learning Theory. Even for just two halfspaces, it remains a major open question whether learning is possible in polynomial time with respect to the margin $\gamma$ of the data points and their dimensionality $d$. The best-known algorithms run in quasi-polynomial time $d^{O( \log{1/\gamma} )}$, and it has been shown that this complexity is unavoidable for any algorithm relying solely on correlational statistical queries (CSQ).

In this work, we introduce a novel algorithm that provably circumvents 
the CSQ hardness barrier. Our approach applies to a broad class of 
distributions satisfying a natural, previously studied, factorizability assumption. Factorizable distributions lie between distribution-specific and distribution-free settings, and significantly extend previously known tractable cases.
Under these distributions, 
we show that CSQ-based methods still require quasipolynomial time even for weakly learning, whereas our algorithm achieves $\text{poly}(d,1/\gamma)$ 
time by leveraging more general statistical queries (SQ), establishing a strong separation between CSQ and SQ for this simple realizable PAC learning problem. 

Our result is grounded in a rigorous analysis utilizing a novel duality framework that characterizes the moment tensor structure induced by the marginal distributions. Building on these structural insights, we propose new, efficient learning algorithms. These algorithms combine a refined variant of Jennrich’s Algorithm with PCA over random projections of the moment tensor, along with a gradient-descent-based non-convex optimization framework.




\end{abstract}

\newpage
\setcounter{page}{1}

\input{intro}
\input{structure-result}

\input{SQ-relevent-direction}

\input{CSQ-relevent-direction}

\input{localization}



\bibliography{mydb}
\bibliographystyle{alpha}
\newpage

\appendix



\input{appendix}

\input{CSQ-lb}

\end{document}

%% file: intro.tex

\section{Introduction}

A halfspace $h=\sign(\bu^*\cdot \bx +t_1):\R^d \to \{\pm 1\}$ is a Boolean function defined by its weight vector $\bu^* \in \R^d$ 
and threshold $t_1 \in \R$. Halfspace learning is one of the oldest and most fundamental problems in Machine Learning \citep{rosenblatt1958perceptron, block1962perceptron}. 
While learning a single halfspace is well-understood, 
learning intersections of halfspaces is significantly more challenging. 
Even for intersections of two halfspaces, polynomial-time algorithms 
are known only under strong  distributional assumptions about the datapoints $\bx$, which are e.g., assumed to be drawn from a Gaussian or log-concave distribution \citep{blum1997learning, vempala2010learning, vempala2010random}.  
Beyond these assumptions, little is known about the 
problem's complexity. Prior work \citep{klivans2007unconditional,klivans2009cryptographic,daniely2016complexity,tiegel2024improved} establish hardness results for learning intersections of $\omega_d(1)$ halfspaces. It is a central open question in computational 
learning theory whether an intersection of even two halfspaces can be 
efficiently learned in the distribution-free setting.
 
To design efficient learning algorithms in the more challenging distribution-free setting, 
a popular approach is to assume that the underlying distribution has a margin with respect to the target hypothesis; see, e.g.,~\citep{arriaga2006algorithmic,klivans2004learningmargin}. Under the $\gamma$-margin assumption, it is well-known that the Perceptron algorithm   properly learns a single halfspace in time $\Tilde{O}(d/(\gamma^2\epsilon))$; see, e.g.,~\cite{cristianini2000introduction}. 
Unfortunately, a similar result does not hold for learning an intersection of two halfspaces. Even under a margin assumption, it is computationally hard to output 
an intersection of any constant number of halfspaces 
with an error better than $1/2$~\citep{khot2008hardness}.
The best known algorithm~\citep{klivans2004learningmargin} in this setting, developed over 20 years ago, 
runs in time $d^{O(\log(1/\gamma))}$ and 
outputs a polynomial threshold function 
with degree $O(\log(1/\gamma))$. 
Such a learning algorithm not only has super-polynomial 
time complexity, but also needs super-polynomial 
time to evaluate its hypothesis on a single example. 
An important open question, 
posed in~\cite{klivans2004perceptron}, 
is whether 
a $\poly(d,1/\gamma,1/\eps)$ time learning algorithm 
exists under only a $\gamma$-margin assumption.

Specifically, the algorithm developed by \cite{klivans2008learning} closely relates to
the notion of Correlation Statistical Queries (CSQ), 
which are queries of the form $\E_{(x,y)\sim D}[yq(x)]$, 
where $q$ is an arbitrary bounded function. 
The main idea of this type of algorithm relies 
on the fact that the target hypothesis 
can be represented as a high-degree 
polynomial threshold function, 
and thus one can make CSQ queries---independent 
of the marginal distribution---to obtain a weak hypothesis; 
a strong hypothesis can then be obtained via boosting. Algorithms of this type usually do not leverage 
useful structural properties of the underlying 
learning problem, 
and are hard to adapt to obtain more efficient algorithms.
In fact, even for weak learning, 
$d^{\Omega(\log(1/\gamma))}$ complexity is the best 
one can hope for via a CSQ algorithm. 
This suggests that, to make progress toward 
a polynomial time algorithm, 
a new algorithmic framework is needed. 
In particular, one needs to design 
\emph{instance-dependent} statistical queries 
by learning information about
the marginal distribution $D_X$. 
In this work, we introduce a novel algorithm 
that provably circumvents the CSQ-hardness barrier 
under the $\gamma$-margin assumption. 
Our algorithm runs in fully-polynomial time 
for a broad class of  distributions satisfying 
a factorizability assumption. 
We now formally define the problem we study in this paper.

\begin{definition}[Learning Intersections of Margin Halfspaces 
Under Factorizable Distributions]\label{def problem}
Let $V \subseteq \R^d$ be an unknown two-dimensional subspace 
and $W=V^\perp \subseteq \R^d$ be the orthogonal complement of $V$. 
Let $h^*(\bx) = \sgn(\bu^*\cdot \bx+t_1) \wedge \sgn(\bv^* \cdot \bx+t_2): \R^d \to \{\pm 1\}$, 
where $\bu^*,\bv^* \in V \cap S^{d-1}$ be the target directions and $t_1,t_2 \in \R$ are the thresholds of the defining halfspaces.   
Let $D$ be a distribution over $\R^d \times \{\pm 1\}$ 
satisfying the following:
\begin{enumerate}[leftmargin=*,nosep]
    \item The distribution $D$ is consistent with an instance of learning intersections of two halfspaces $h^*$, i.e., for $(\bx,y)\sim D$, $y=h^*(\bx)$ holds almost surely.
    \item  The distribution $D$ satisfies the $\gamma$-margin assumption, i.e., for  
$\bx\sim D_X$, it holds $\norm{\bx}_2 \le 1$ and $\abs{\bu^*\cdot \bx_V+t_1} \ge \gamma$, 
$\abs{\bv^*\cdot \bx_V+t_2} \ge \gamma$ holds almost surely. 
Here $\bx_V$ is the projection of $\bx$ on $V$.

\item We say that $D$ is factorizable if $D_X=D_V\times D_W$, 
where $D_V$ is the marginal distribution of $D_X$ over $D_V$ 
and $D_W$ is the marginal distribution of $D_X$ over $W$.
\end{enumerate}
Given parameters $\epsilon,\delta \in (0,1)$, a learning algorithm $\A$ 
draws a set $S=\{(\bx^{(i)},y^{(i)})\}_{i=1}^m$ of $m$ examples \iid from $D$ 
and outputs a hypothesis $\hat{h}:\R^d \to \{\pm 1\}$ 
such that with probability at least $1-\delta$,
$\err(\hat{h}):= \Pr_{(\bx,y) \sim D}\left(\hat{h}(\bx) \neq y\right) \le \eps. $  
\end{definition}
Throughout this paper, we will use $D_X$ for the marginal distribution of $D$ on 
the feature space, $D^+$ to denote the marginal distribution of $D_X$ on positive examples, and $D^-$ to denote the marginal distribution 
of $D_X$ on negative examples. 

\paragraph{Discussion} 
Factorizable distributions lie between the distribution-specific 
and the distribution-free settings. 
Specifically, an efficient learning algorithm 
for such distributions would significantly 
extend previously known tractable settings, 
such as under the Gaussian or uniform distribution over the unit 
sphere, 
as no assumptions are made over $D_V$ and $D_W$. 
Factorizable distributions are not new in this learning context.
The original motivation of studying them 
can be traced back at least 
to \cite{blum1994relevant} 
in the context of learning $k$-juntas under the uniform distribution on the hypercube, 
and to \citep{kOS2008learning,vempala2010learning} 
for learning convex concepts under the Gaussian distribution.  
In both of these settings, the target hypothesis 
only depends on the projection of the points 
on some unknown low-dimensional subspace, 
the marginal distributions are factorizable 
and satisfy additional strong assumptions. 
With this motivation, 
\cite{vempala2011structure} 
first formally proposed the setting 
of learning $k$-subspace juntas
(functions that only depend on the projection on a $k$-dimensional subspace) 
under factorizable distributions. 
The original observation of \cite{vempala2011structure} 
was that if there are $k$ directions in $V$ 
along which the moments of $D_V$ 
are different from those of a standard Gaussian, 
then one can information-theoretically recover the subspace $V$. 
However, as we will discuss in detail in \Cref{sec direction}, 
this approach incurs an exponential dependence on $d$ 
and the accuracy parameter. 
To obtain computationally efficient algorithms, 
\cite{vempala2011structure} additionally assume 
that $D_W$ is the standard Gaussian. 
Under this assumption, if $D_V$ satisfies 
the aforementioned moment conditions 
and $H$ satisfies some additional robustness assumptions, 
they gave an algorithm that approximately recovers 
the relevant subspace and 
then learns over a low-dimensional space. 
In contrast, our work focuses on the original setting 
where no additional assumptions are made over $D_V,D_W$ 
for the basic case of intersections of two large-margin halfspaces. 
In this context, we establish novel structural results 
for intersections of two halfspaces, 
and design a fully-polynomial time learning algorithm. 
We summarize our results below.

\paragraph{Our Contribution and Technical Overview}
We start by establishing a quasi-polynomial 
$d^{\Omega(\log(1/\gamma))}$
Correlational Statistical Query (CSQ) lower bound  
(\Cref{thm:csq-lb}) for our learning task (\Cref{def problem}). 
Our CSQ lower bound shows that, unlike learning under the Gaussian distribution 
where there exists a fully-polynomial CSQ algorithm, 
learning in the factorizable setting is more challenging and CSQ algorithms even fail to efficiently weakly learn.
Furthermore, the CSQ lower bound we obtain matches the running time 
of the algorithm developed by \cite{klivans2008learning} for learning 
intersections of two $\gamma$-margin halfspaces (even without the factorizable assumption).  
This suggests that a new algorithmic framework is required to obtain 
more efficient algorithms. Due to space limitations, the proof of this lower bound is deferred to Appendix~\ref{app lb}.


\begin{theorem}[CSQ Lower Bound] \label{thm:csq-lb}
    Let $\gamma>0$, $q, d \in \N$, $\tau\in (0,1)$ and $d'=\min(d,1/\gamma^2)$. 
    Any CSQ algorithm that learns intersections of two halfspaces with $\gamma$-margin in $d$ dimensions 
    under factorizable distributions 
    to error 
    $1/2-\max(d'^{-\Omega(\log(1/\gamma))},2^{-d'^{\Omega(1)}})$ requires 
    $q$ queries of tolerance at most $\tau$, 
    where $q/\tau^2\geq \min(d'^{\Omega(\log(1/\gamma))},2^{d'^{\Omega(1)}})$.  
\end{theorem}

Our main algorithmic result 
bypasses the CSQ-hardness by giving 
a polynomial-time algorithm 
for learning any intersection of two halfspaces 
with $\gamma$-margin, 
as long as $D_X$ is factorizable. 
This implies the first strong separation between 
CSQ and SQ algorithms for \emph{weak} (realizable) 
PAC learning of a natural concept class. 
We refer the reader to the related work 
for more detailed discussion.

\begin{theorem}[Main Result]\label{th main}
There is an algorithm that for any distribution $D$ over $\B^d(1) \times \{\pm 1\}$ 
satisfying the conditions of \Cref{def problem}, for any $0< \eps, \delta <1$, 
has the following guarantees: 
it draws $n=\poly(d,1/\gamma,1/\epsilon,\log(1/\delta))$ labeled examples from $D$, 
runs in $\poly(n, d)$ time, 
and outputs a hypothesis 
$\hat{h}:\B^d(1) \to \{\pm 1\}$ 
such that with probability at least $1-\delta$, 
$\err(\hat{h}) \le \eps$.
\end{theorem}

The full proof of \Cref{th main} and 
the main learning algorithm 
are deferred to Appendix~\ref{app main}. 
In the rest of this section, we give a detailed 
outline of the techniques involved in proving this result.

Intuitively, the CSQ lower bound of \Cref{thm:csq-lb} implies 
that without learning some information about the marginal distribution $D_X$, 
it is impossible to efficiently find a CSQ, $q$, 
such that $\abs{\E_{(\bx,y)}yq(\bx)}>\poly(\gamma/d)$; 
as otherwise, one could 
output $\sign(q(\bx)-t), t\sim[-1,1]$,  
as a weak hypothesis with $1/2-\poly(\gamma/d)$ error. 
This suggests that a plausible approach is to first learn 
(some information about) the marginal distribution $D_X$, 
and use it to design \emph{instance-dependent} 
statistical queries. 
The construction of these queries hinges 
on the following observation. 
If we are given a direction $\bw$ that is $\poly(\gamma)$-close to $V$, 
then restricted over bands $B_i:=\{\bx \mid \bw\cdot\bx \in [i\gamma,(i+1)\gamma]\}$, 
the labels $y$ are consistent with a degree-$2$ polynomial threshold function. 
This in turn implies that a CSQ of the form 
$q(\bx)=\Ind(\bx \in B_i)p(\bx)$, 
for some degree-$2$ polynomial $p$, 
can be used to give a weak hypothesis with $\poly(\gamma)$ advantage. 
Once we have a weak hypothesis, we can run a standard boosting algorithm 
to get a strong hypothesis. With this goal, the question 
is how to efficiently find such a direction $\bw$. 
To achieve this, we start with an easier case, 
where 
$\norm{(\E_{\bx\sim D^+}-\E_{\bx\sim D^-})\bx^{\otimes m}}_F$
is large. 
Since $D_X$ is factorizable, 
we show in \Cref{thm:csq-direction-extraction} 
that any local maximum or local minimum 
of the objective function 
$f(\bu)=(\E_{\bx\sim D^+}-\E_{\bx\sim D^-})(\bu\cdot \bx)^m$ 
must be in $V$. 
Though finding exact locally optimal solutions for $f$ 
is computationally intractable, 
we show that an approximate solution, 
obtained by running a standard gradient-descent method, 
suffices for our purposes.

The more challenging case is when the underlying instance is 
indeed CSQ-hard. 
In this case, a low-degree polynomial function 
$q$ will also satisfy 
$\abs{\E_{(\bx,y)\sim D}yq(\bx)}\approx 0$, 
which implies that 
$\E_{\bx\sim D^+}\bx^{\otimes m}\approx \E_{\bx\sim D^-}\bx^{\otimes m}$ 
for small $m \in \Z_+$. 
Leveraging the fact that the label $y$ only depends on the subspace $V$, 
we establish the following novel structural property for $D_V$. 
Our key structural result can be summarized as follows.

\begin{theorem}[Informal statement of \Cref{th:intersection-structure-main}]
For any distribution $D$ 
satisfying the conditions of \Cref{def problem}, 
if for $m \in [3]$, $\norm{(\E_{\bx \sim D^+}-\E_{\bx \sim D^-})\bx^{\otimes m} }_F \le \poly(\gamma)$ and $\norm{\E_{\bx\sim D_X}\bx}_F \le \poly(\gamma)$, then $\norm{\E_{\bx\sim D_V}\bx_V^{\otimes3}}_F \ge \poly(\gamma)$.
\end{theorem}

That is, for any distribution $D$ consistent 
with our learning task, 
if the first three moments of $D^+,D^-$ 
are nearly matching
and the mean of $D_X$ is close to $0$, 
then the third moment tensor of $D_V$ 
must significantly deviate from $0$. 
The proof of this structural result is rather technical, 
because the only condition we assumed about $D_V$ 
is the margin assumption. 
To prove this result, we develop a novel technique 
which we term \emph{one-sided polynomial approximation.}
Roughly speaking, if we are able to appropriately 
construct a polynomial function $p(\bx)$ 
such that for some $z \in \{\pm 1\}$, 
$\sign(p(\bx))=z, \forall \bx, h^*(\bx)=z$, 
then such a polynomial can be used as a certificate 
to show that distributions satisfying certain moment conditions do not exist. 
One-sided polynomial approximation 
is a powerful tool for proving moment properties of distributions. In \Cref{sec structure}, 
we will carefully design these polynomials to prove useful 
results for the marginal distribution $D_V$. 

The next step is to efficiently find a 
direction $\bw$ close to $V$, 
by leveraging our structural result.
Our initial attempt was inspired by the observation of \cite{vempala2011structure} 
on generalized independent component analysis: for $m\ge 3$ 
if $D_V$ has $m$-th moment different from that of a standard Gaussian, but the first $m-1$ moments are the same as those of a Gaussian, 
then a locally optimal solution to $f(\bu) = \E_{\bx\sim D_X}(\bu \cdot \bx)^m$ 
must be either in $V$ or $W$. Unfortunately, such an idea does not lead to an efficient algorithm for the following reason. Since we make no distributional assumptions 
over $D_W$, $D_W$ could also be factorized as $\Omega(d)$ many distributions 
that are isomorphic to $D_V$. This implies that information-theoretically 
we are only able to find a list of $O(d)$ unit vectors such 
that one of them is close to $V$. 
The natural approach is to optimize $f$ in order to find 
one direction, and optimize the same function over the orthogonal subspace 
to find the next directions. As we explain in \Cref{sec direction}, 
such an approach has exponential sample complexity 
(as also demonstrated in \cite{vempala2011structure}). 
Interestingly, we show that by carefully adapting 
such an idea, 
we are able to obtain an efficient SQ algorithm, 
which in turn leads to $\poly(d/\gamma)$ sample complexity. 
Still, 
it is unclear how to get a computationally 
efficient learning algorithm. 
To overcome this difficulty, we develop 
a completely different approach 
inspired by the tensor-decomposition literature. 
Though the third moment tensor $T$ of $D_X$ 
is in general very high rank 
and does not admit a low-rank decomposition, 
we know that the third-moment tensor $T_V$ of $D_V$ significantly deviates from $0$. 
This implies that if we take the product of $T$ 
with a Gaussian random vector $\bv$, 
then with good probability 
one of the eigenvalues of $T_V \cdot \bv$ 
must be significantly different 
from the other eigenvalues of $T\cdot \bv$. 
By the factorizable assumption, 
an eigenvector of $T \cdot \bv$ must be close to $V$. 
Thus, even for the hard instance mentioned above, 
we are able to efficiently find a direction close to $V$. 
We give \Cref{alg sketch}, a sketch of our approach 
for learning intersections of two halfspaces, 
and give the full algorithm in Appendix~\ref{app main}. 

\begin{algorithm}[h]
		\caption{\textsc{LearningIntersections} (Sketch)}\label{alg sketch}
		\begin{algorithmic} [1]
\State Find a list $\cal O$ of unit vectors as follows: if $D^+$ and $D^-$ have nearly matching first three moments, run the algorithm in \Cref{sec matched}; otherwise, run the algorithm in \Cref{sec:direction-extraction-csq}.
\State For each $\bw \in \mathcal{O}$, run Adaboost with the algorithm in \Cref{sec local} to get a hypothesis $h_{\bw}$. 
\State Return the best hypothesis $h \in \{h_{\bw} \mid \bw \in \mathcal{O}\}$ via a standard hypothesis testing approach.
\end{algorithmic}
\end{algorithm}
\vspace{-0.1cm}

\subsection{Related Work} 

\paragraph{Learning Intersections of Halfspaces}

Learning intersections of halfspaces is one of the central problems in learning theory. Despite a long line of work studying this problem algorithmically \citep{baum1990polynomial, long1994composite,kwek1996pac,blum1997learning,klivans2004learning,klivans2008learning,kOS2008learning,klivans2009baum,vempala2010learning,vempala2010random}, relatively 
little is known about its complexity.
In the distribution-specific setting, \cite{blum1997learning} 
first showed that under the uniform distribution over the unit 
sphere (or under the Gaussian distribution), 
for any fixed constant $k$, one can learn 
an intersection of $k$ halfspaces in polynomial time via PCA.  \cite{vempala2010random} later extend this result 
to isotropic log-concave distributions 
via a random sampling method. For discrete distributions, \cite{klivans2004learning}, developed Fourier-based 
algorithms for learning an intersection of any constant number 
of halfspaces under the uniform distribution over the Boolean hypercube via Fourier analysis (albeit with complexity exponential in the inverse of the accuracy parameter $\eps$). Many subsequent works in the distribution-specific setting 
developed algorithms with improved complexity under the 
Gaussian or uniform distribution over the hypercube. 
It remains an open question whether, 
under these strong distributional assumptions, 
an intersection of $k$ halfspaces can be learned 
in fully-polynomial time. 
For the special case of two halfspaces, \citep{baum1990polynomial,klivans2009baum}  
developed polynomial-time algorithms 
for learning an intersection of two homogeneous halfspaces 
under mean zero symmetric/isotropic log-concave distributions. We emphasize that the distributional assumptions 
of these two works are fairly strong 
and the underlying algorithms do not work 
if the defining halfspaces do not go through 
the origin.

In the distribution-free setting, much less is understood. \citep{klivans2007unconditional,klivans2009cryptographic,daniely2016complexity,tiegel2024improved} showed that 
if the number of halfspaces $k=\omega_d(1)$, 
then distribution-free learning is hard. \cite{tiegel2024improved} recently gave 
an SQ lower bound of $d^{\Omega(k)}$ for distribution-free learning intersections of $k$ halfspaces. 
However, no super-polynomial SQ lower bound 
(or any other representation-independent hardness result) 
is known for learning intersections of a 
constant number of halfspaces.

\paragraph{Independent Component Analysis and Its Generalization}
The learning setting where the marginal distribution is factorizable is closely related to the work of \cite{vempala2011structure} on 
an unsupervised learning setting, known as 
generalized independent component analysis. 
Independent component analysis \citep{jutten1991blind}, originally considered the following problem: 
given examples $y \in \R^d$ generated by $y= A\bx$, 
where $A \in \R^{d \times d}$ is an unknown matrix 
and $\bx \in \R^d$ is a random vector such that 
$\bx_i, i \in [d]$ are independent, 
recover the underlying direction $\bx_1,\dots,\bx_d$. 
ICA is a natural generalization of PCA, 
another task that identifies the source components 
given only their linear combinations. 
PCA can be viewed as the task of finding vectors
on the unit sphere that are local optima 
of the second moment of the observed data. 
Such an approach fails when eigenvalues repeat 
and ICA bypasses the difficulty by considering finding 
a local optimum of functions related to higher-order moments of observed data. It is known that the original ICA problem 
can be efficiently solved via second-order method  \citep{frieze1996learning,arora2012provable}. 
Generalized ICA \citep{vempala2011structure} instead 
aims to recover the distribution $D_V$ 
given examples generated from a factorizable distribution $D_X=D_V \times D_W$. Such a problem can also be viewed 
as a generalization of Non-Gaussian Component Analysis(NGCA)~\citep{tan2018polynomial,goyal2019non}. 
However, unlike the original ICA problem, 
the generalized ICA suffers issues of numerical instability, 
and no fully polynomial-time algorithm 
is known for the problem so far.

\paragraph{SQ Model and CSQ Model} The CSQ model~\citep{BshoutyFeldman:02} 
is a subset of the SQ model~\citep{Kearns:98}, where the oracle access
is of a special form (see Appendix~\ref{sec notation}).
In particular, any SQ query function $q_{\mathrm{sq}}:X\times \{0,1\}\to [-1,1]$ 
can always be decomposed to $q_{\mathrm{sq}}(\bx,y)=q_1(\bx)+q_2(\bx,y)$, 
where $q_1(\bx)$ 
is a query function independent of the label $y$, 
and $q_2(\bx,y)$ 
is a CSQ query.
An intuitive interpretation is that, compared with the SQ model, 
the CSQ model loses exactly the power to make label-independent queries 
about the distribution, i.e., the power to ask queries about the marginal 
distribution of $\bx$. 
In the context of learning Boolean-valued functions, the two models are 
known to be equivalent in the distribution-specific setting 
(i.e., when the marginal distribution on feature vectors 
is known to the learner)~\citep{BshoutyFeldman:02}. 
However, they are not in general 
equivalent in the distribution-free PAC model. 
In the realizable PAC setting, there are known 
natural separations between the CSQ and SQ models. 
Notably,~\cite{feldman2011distribution} showed that Boolean halfspaces 
are not efficiently learnable 
up to an arbitrarily small accuracy 
via CSQ algorithms (even though they are efficiently SQ learnable). Intuitively, such a separation exists 
because even though CSQ algorithms can be used 
to learn a weak hypothesis, 
without using stronger SQs, we cannot implement 
boosting algorithms to get a strong hypothesis. 
Our results for learning intersections of two 
halfspaces exhibit a new separation between 
CSQ and SQ models in the realizable PAC learning setting, 
in terms of weak learning. That is to say, 
under the assumption of factorizable distributions, 
it is hard to efficiently find a hypothesis 
with error $1/2-o(1)$ using CSQ algorithms; 
but this is possible via efficient SQ algorithms. 
This shows the necessity of using SQ queries 
for PAC learning is not only due to boosting, 
but also due to the hardness of finding a weak hypothesis.

\paragraph{Organization}
In \Cref{sec structure}, we present our main structural result 
for learning intersections of two halfspaces. In \Cref{sec direction}, 
we make use of the structural results developed in \Cref{sec structure} 
to design computationally efficient algorithms that find a direction close to $V$. 
In \Cref{sec local}, we show how to use the direction we find in \Cref{sec direction} 
to obtain an efficient weak learner. 

%% file: structure-result.tex
\section{Structural Result for Distribution-Free Learning Intersections of Two Halfspaces}\label{sec structure}




By \Cref{thm:csq-lb}, 
we know that without looking at the marginal distribution $D_X$, 
it is impossible to find a Correlational Statistical Query (CSQ) 
that can detect the correlation between $D_X$ and $D_y$; 
thus, even performing weak learning efficiently
via CSQs is impossible.
To bypass this inherent limitation of CSQ algorithms, 
we need to design ``instance-dependent'' CSQs by first 
looking at the marginal $D_X$.
Motivated by this intuition, we 
provide a novel structural result 
for learning intersections of two halfspaces, 
which we will make essential use of in \Cref{sec direction} 
to design instance-dependent statistical queries. 
To start with, we present the following $(\alpha,m)$-moment matching condition that will be heavily discussed throughout the paper.
\begin{definition}[$(\alpha,m)$-moment matching condition]\label{def moment match}
Let $\alpha \ge 0$, $m \in \Z_+$, 
and $D$ be a distribution of $(\bx,y)$ 
over $\R^d\times \{\pm 1\}$.
We say that $D$ satisfies the $(\alpha,m)$-moment matching condition 
if 
$\norm{\left(\E_{\bx \sim D^+} - \E_{\bx \sim D^-}\right) \bx^{\otimes m}  }_F \le \alpha$.
\end{definition}

To simplify the notation, in this section 
we make the following assumption about the subspace $V$.
We assume that $V$, the subspace spanned by $\bu^*,\bv^*$,  
is exactly equal to $\textbf{span}\{\be_1,\be_2\}$. 
Such an assumption can be made without loss of generality, as 
applying a rotation transformation will not affect 
the $\gamma$-margin assumption. 
Since the label of an example $\bx$ only depends 
on its projection $\bx_V$ onto $V$, to simplify notation 
in this section, we restrict attention to the dimensions 
of $V$ and consider examples $\bx\in \R^2$ drawn 
from $D^+$ (or $D^-$).
Our main structural result, \Cref{th:intersection-structure-main}, 
shows that if $D^+,D^-$ have nearly matched their first three moments, 
and $\E_{\bx \sim D_X}(\bx)$ is close to $0$, then the third-moment 
tensors of $D^+, D^-$ must significantly deviate from $0$.



\begin{theorem}\label{th:intersection-structure-main}
Let $D$ be a distribution over $\B^2(1) \times \{\pm 1\}$ that is 
consistent with an instance of learning intersections of two 
halfspaces with $\gamma$-margin. Let $c>0$ be \emph{any} 
suitably large constant. Suppose that $D$ satisfies the 
$(\gamma^c,m)$-moment matching condition for $m \in [3]$, 
and $\norm{\E_{\bx\sim D^+} \bx}_F \le \gamma^c$. 
Then $\norm{\E_{\bx\sim D^+}\bx^{\otimes 3}}_F, \norm{\E_{\bx\sim D^-}\bx^{\otimes 3}}_F = \Omega(\gamma^{15})$.
\end{theorem}

Here we present some high-level intuition 
for the proof of \Cref{th:intersection-structure-main}.  
The formal proof is given in Appendix~\ref{app structure main}. 
To prove \Cref{th:intersection-structure-main}, 
we establish two technical lemmas.
The first lemma, \Cref{lm 2nd}, 
shows that the conditions in 
\Cref{th:intersection-structure-main}
imply that the smallest 
eigenvalue of the covariance matrix of both $D^+$ and $D^-$ 
will be at least $\gamma^c$. 
Thus, by properly rescaling $D$ 
(by at most a $\poly(\gamma)$ factor), 
we can create a new distribution $D'$ 
such that both $(D')^+$ and 
$(D')^-$ have isotropic covariance matrices 
and the distribution $D'$ 
still satisfies a $\gamma'$-margin condition 
with respect to an 
intersection of two halfspaces, 
where $\gamma'=\gamma^{c'}$. 
Now assuming, for the purpose of contradiction, 
that $\norm{\E_{\bx\sim D^+}\bx^{\otimes 3}}_F, \norm{\E_{\bx\sim D^-}\bx^{\otimes 3}}_F < \poly(\gamma)$,
we must also have 
$\norm{\E_{\bx\sim D'^+}\bx^{\otimes 3}}_F, \norm{\E_{\bx\sim D'^-}\bx^{\otimes 3}}_F < \poly(\gamma)$ 
on the new distribution $D'$ as well.
This assumption implies that both $(D')^+$ and $(D')^-$ 
must have their first and third moment tensors roughly $0$.
Our second lemma, \Cref{lm 3rd moment}, 
shows that the scale of the 
covariance matrices of $D'^+$ and $D'^-$ 
must be significantly different,  
which contradicts the previous assumption 
that the first 
three moments of $D'^+$ and $D'^-$ are nearly matched. 
The proofs of 
these technical lemmas rely on a novel technique, 
which we term \emph{polynomial one-sided approximation}, 
that leverages weak duality between distributions 
with specific 
moments 
and 
polynomial certificates. We summarize the property of one-sided 
polynomial approximation in the following theorem 
and defer its proof to Appendix~\ref{app one-side}.

    
    

\begin{theorem}[Polynomial One-sided Approximation]\label{th one-side}
For any $d,m\in \N$, $C\subseteq\R^d$, 
$T_{i}\in (\R^d)^{\otimes i}$ for $i\in [m]$, 
and $\tau\in \R_{\geq 0}$, at most one of the following conditions can be satisfied:
\begin{enumerate}[leftmargin=*,nosep]
    \item [a)] There is a distribution $D$ supported on $C$ such that $\|\E_{\bx\sim D}(\bx^{\otimes i})-T_{i}\|_F\leq \tau$ for any $i\in [m]$; 
    \item [b)] There is a degree-$m$ polynomial $p:\R^d\to \R$, defined as $p(\bx)=\littlesum_{i=1}^m  A_i\cdot\bx^{\otimes i}$, 
    where $p(\bx)\geq 0$ for any $\bx\in C$,
    $\littlesum_{i=1}^m\|A_i\|_F\leq 1$ and $\littlesum_{i=1}^m  A_i\cdot T_i<-\tau$.
\end{enumerate}
We call such a polynomial a {\em one-sided approximation for $C$ w.r.t. to moments $T_i$ and tolerance $\tau$}.
\end{theorem}

Given \Cref{th one-side}, in order to certify the non-existence of 
certain distributions on $C$ with specific moments, 
it suffices to construct a one-sided approximating 
polynomial, as 
stated in \Cref{th one-side}. However, a direct application of 
\Cref{th one-side} may not be intuitive. 
In the rest of the section, we explain how to use 
this idea 
to prove our two technical lemmas.
To give a clearer intuition, it is convenient to parameterize the 
instance of learning intersections of two halfspaces 
as in \Cref{as parameter} 
(see \Cref{fig geometry} for a geometric illustration). 
We defer further discussion to Appendix~\ref{app parameter} showing 
that such a parameterization can be made without loss of generality.

\begin{assumption}\label{as parameter}
    Given an intersection of two halfspaces 
    $h^*=\sign(\bu^*\cdot \bx + t_1) \wedge \sign(\bv^* \cdot \bx + t_2)$ 
    and a distribution $D$ over $\B^2(1) \times \{\pm 1\}$ 
    satisfying the $\gamma$-margin condition w.r.t. $h^*$, 
    we parameterize $h^*$ by an angle $\theta \in (0,\pi/2)$, 
    and thresholds $t \ge 0, \sigma\ge 0$, where
    $\bu^* = \sin \theta \be_1 - \cos \theta \be_2, t_1 = t\sin\theta$ 
    and $\bv^* = \sin \theta \be_1 + \cos \theta \be_2, t_2 = (1+\sigma)t\sin\theta$ such that 
    $\norm{\E_{\bx\sim D^+} \bx} \le \gamma^c$, $t_1,t_2 \ge \gamma$, $|t_1|,|t_2|\leq 1$.
\end{assumption} 

\begin{figure}
    \centering
    \includegraphics[width=0.6\linewidth]{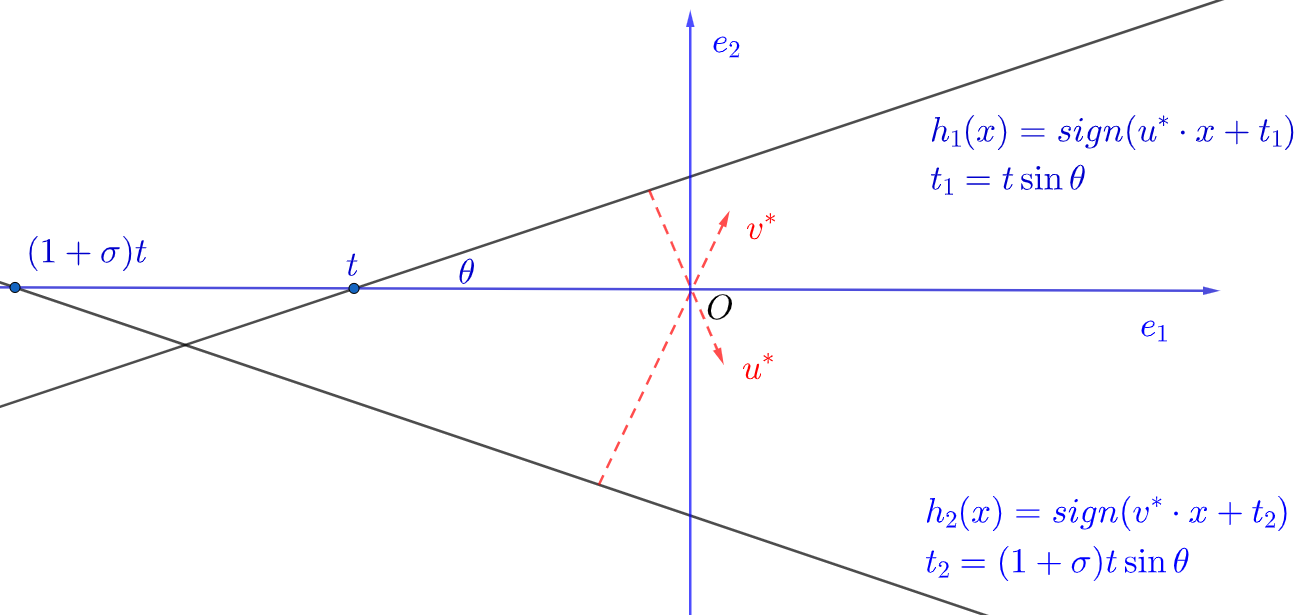}
    \caption{Geometry of Intersection of Two Halfspaces under \Cref{as parameter}.}
    \label{fig geometry}
\end{figure}

\paragraph{Moment-matched Distributions Cannot Have Ill-Conditioned Covariance Matrices}
Our first technical lemma (\Cref{lm 2nd}) shows that 
if $D^+,D^-$ have nearly matched 
their first two moments, then the covariance matrices of $D^+,D^-$ 
cannot have small eigenvalues. 

\begin{lemma}\label{lm 2nd}
     Let $D$ be a distribution over $\B^2(1) \times \{\pm 1\}$ that is consistent 
     with an instance of learning intersections of two halfspaces 
     with $\gamma$-margin, where $\gamma$ is smaller than 
     some sufficiently small constant. 
     Let $c>0$ be any suitably large constant.  
    Suppose that $D$ satisfies the $(\gamma^c,m)$-moment matching 
    condition for $m \in [2]$ and $\norm{\E_{\bx\sim D^+} \bx}_F\le \gamma^c,\norm{\E_{\bx\sim D^-} \bx}_F\le \gamma^c$. 
Then $\norm{(\Sigma^+)^{-1}}_2\le O(1/\gamma^{4})$ and 
$\norm{(\Sigma^-)^{-1}}_2\le O(1/\gamma^{4})$, where 
$\Sigma^+:=\E_{\bx\sim D^+}\bx\bx^T$ and $\Sigma^-:=\E_{\bx\sim D^-}\bx\bx^T$.
\end{lemma}

The proof strategy of \Cref{lm 2nd} is to construct a suitable polynomial function 
as a certificate, as stated in \Cref{th one-side}. We will use the polynomial 
function $f(\bx) = (\bu^*\cdot \bx+t_1)(\bv^*\cdot \bx+t_2)$ as a certificate. Intuitively, if the first two moments of $D^+,D^-$ 
are nearly matched, then $\E_{\bx\sim D^+} f(x) \approx \E_{\bx\sim D^-} f(\bx)$. Furthermore, by the $\gamma$-margin assumption, 
for every positive example $\bx$, $f(\bx) \ge \gamma^2$ 
and for every negative example $\bx$ with $h_1(\bx)h_2(\bx)=-1 $, 
$f(\bx) \le -\gamma^2$. This implies that if the probability 
$\Pr_{\bx\sim D^-}(h_1(\bx)=h_2(\bx))$ is small, 
then  $\E_{\bx\sim D^+} f(\bx) - \E_{\bx\sim D^-} f(\bx) \ge \Omega(\gamma^2)$, 
which gives a contradiction. Geometrically, if 
$\Pr_{\bx\sim D^-}(h_1(\bx)=h_2(\bx))$ is large, 
then examples in this region must have a significant contribution to $\Sigma^-$ 
along every direction $\bv$ to make $\E_{\bx \sim D^-} \bx$ close to $0$, which 
contradicts the fact that $\Sigma^-$ has a direction with small variance.




\paragraph{Moment-matched Distributions Must Have Large Third Moments}

By \Cref{lm 2nd}, it is safe to assume that the marginal distribution $D_X$ 
has an isotropic covariance matrix. Our main structural result, 
\Cref{lm 3rd moment}, shows that if $D^+,D^-$ have nearly matched first three moments and their covariance matrices are nearly isotropic, 
then their third moments must significantly deviate from $0$. 
We defer the proof of \Cref{lm 3rd moment} to Appendix~\ref{app lm 3rd}.

\begin{lemma}\label{lm 3rd moment}
    Let $D$ be a distribution over $\B^2(1) \times \{\pm 1\}$ 
    that is consistent with an instance of learning intersections of two halfspaces with $\gamma$-margin. 
    Let $c>0$ be any suitably large constant. Suppose
    \begin{enumerate} [leftmargin=*, nosep]
        \item $\norm{\E_{\bx\sim D^+} \bx}_F\le \gamma^c, \norm{\E_{\bx\sim D^-} \bx}_F\le \gamma^c$.
        \item $\E_{\bx\sim D^+}\bx\bx^T = \alpha^2 I+\Delta_+, \E_{\bx\sim D^-}\bx\bx^T = \alpha^2 I+\Delta_-$, where $\Delta_+,\Delta_- \in \R^{2\times 2}$ are symmetric matrices such that $\norm{\Delta_+}_F \le \gamma^c, \norm{\Delta_-}_F \le \gamma^c$ and $\alpha^2 > 0$.
        \item $\norm{\left(\E_{\bx\sim D^+}- \E_{\bx\sim D^-}\right)\bx^{\otimes 3}}_F\le \gamma^c$. 
    \end{enumerate}
    Then we have $\norm{\E_{\bx\sim D^+}\bx^{\otimes 3}}_F \ge \Omega(\gamma^2), \norm{\E_{\bx\sim D^-}\bx^{\otimes 3}}_F \ge \Omega(\gamma^2)$.
\end{lemma}

The intuition behind \Cref{lm 3rd moment} relies on 
\Cref{fact pos} and \Cref{fact neg} below that characterize 
the covariance matrix of \emph{any} pair of distributions $D^+,D^-$ 
with zero mean and zero third moment tensor. 

\begin{fact}
\label{fact pos}
    Let $h^*=\sign(\bu^*\cdot \bx + t_1) \wedge \sign(\bv^* \cdot \bx + t_2)$ 
    be the target hypothesis of an instance of the problem of learning 
    intersections of two halfspaces, 
    and $D$ be a distribution that is consistent with $h^*$. 
    Under \Cref{as parameter}, if 
    $\E_{\bx\sim D^+}  (\bx) = 0, \E_{\bx\sim D^+} \bx^{\otimes 3} =0$ 
    and for every $\bv \in S^{d-1}\cap V$, it holds
    $\E_{\bx\sim D^+}(\bv\cdot \bx)^2 = \alpha^2$, 
    then $\alpha^2 \le t^2 \sin^2\theta$.
\end{fact}

\begin{fact}\label{fact neg}
        Let $h^*=\sign(\bu^*\cdot \bx + t_1) \wedge \sign(\bv^* \cdot \bx + t_2)$ 
        be the target hypothesis of an instance of the problem of learning 
        intersections of two halfspaces, 
        and $D$ be a distribution that is consistent with $h^*$. 
        Under \Cref{as parameter},
        if $\E_{\bx\sim D^-}  \bx = 0, \E_{\bx\sim D^-} \bx^{\otimes 3} =0$ 
        and for every $\bv \in S^{d-1}\cap V$, it holds 
        $\E_{\bx\sim D^-}(\bv\cdot \bx)^2 = \beta^2$, 
        then $\beta^2 \ge (1+\sigma)t^2\tan^2\theta$.
\end{fact}

Here we give an overview of the proof techniques 
behind \Cref{fact pos} and \Cref{fact neg}, and defer 
the full proofs to Appendix~\ref{app fact pos neg}. 
We take \Cref{fact pos} as an example and show how the 
certificate one-sided approximating 
polynomial used in \Cref{th one-side} 
is derived. For every $\bx$ that is labeled positive by $h^*$, 
denote by $p(\bx)$ the variable of the density of a distribution $D^+$ 
over $\R^2$. Notice that any distribution $D^+$ that satisfies 
the statement of \Cref{fact pos} gives a feasible solution 
to the following LP~\eqref{LP max}. 
Thus, to upper bound the variance of $D^+$, 
it is equivalent to upper bound the optimal value of LP~\eqref{LP max}.
\begin{align}\label{LP max}
\begin{split}
\max \  \alpha^2 \quad  
\text{s.t. } & \littlesum_{\bx}p(\bx)\bx = 0, \quad 
 \littlesum_\bx p(\bx)\bx\bx^T = \alpha^2 I, \\
& \littlesum_\bx p(\bx)\bx^{\otimes 3} = 0, \quad 
 \littlesum_\bx p(\bx) =1, \quad p(\bx) \ge 0, \forall \bx \in   \textbf{supp}(D^+)
\end{split}
\end{align}
To upper bound the optimal value of LP~\eqref{LP max}, 
we use LP duality theory \citep{bertsimas1997introduction,Shapiro2001}. 
The dual linear program to LP~\eqref{LP max} is defined by LP~\eqref{LPdual max}, 
whose variable is defined over the coefficients of $f(\bx)$, 
a degree-3 polynomial over $\R^2$, 
and the objective function is given by its constant term, namely 
\begin{align}\label{LPdual max}
\begin{split}
\min \  a_{0}  \quad
\text{s.t. } & \forall \bx \in \textbf{supp}(D^+), \ f(\bx) \ge 0 , \quad a_{11}+a_{22} =-1.
\end{split}
\end{align}
Here, $a_{11}, a_{22}$ are the coefficients of $f$ 
with respect to monomials $\bx_1^2,\bx_2^2$.
Thus, to give tight bounds for $\alpha^2,\beta^2$, 
the key technical difficulty is to design a pair of polynomials 
that are feasible to the dual LPs with nearly optimal objective values. 
The polynomials we used here are given by \Cref{lm positive polynomial} 
and \Cref{lm negative polynomial} 
(illustrated in \Cref{fig geometry}), 
the proofs of which are deferred to Appendix~\ref{app exact match}.

\begin{lemma}\label{lm positive polynomial}
    Let $h^*=\sign(\bu^*\cdot \bx + t_1) \wedge \sign(\bv^* \cdot \bx + t_2)$ 
    be the target hypothesis of an instance of the problem of learning intersections of two halfspaces and $D$ be a distribution that is consistent with $h^*$. 
    Under \Cref{as parameter}, the polynomial 
    $f^*(\bx) = \frac{1}{t\sin\theta}(\bu^*\cdot \bx-t\sin\theta)^2(\bu^*\cdot \bx+t\sin \theta)$
    satisfies $f^*(\bx) \ge 0, \forall \bx \in \textbf{supp}(D^+)$.
\end{lemma}

\begin{lemma}\label{lm negative polynomial}
    Let $h^*=\sign(\bu^*\cdot \bx + t_1) \wedge \sign(\bv^* \cdot \bx + t_2)$ be the target hypothesis of an instance of the problem of learning intersections of two halfspaces and $D$ be a distribution that is consistent with $h^*$.
    Under \Cref{as parameter}, the polynomial 
    $f^*(\bx) = a_0 + a_1 \bx_1 + a_2 \bx_2 - \bx_2^2$, where
    $a_0 = (1+\sigma)\tan^2 \theta t^2$,
    $a_1 = (2+\sigma)\tan^2\theta t$ and
    $a_2 =-\sigma \tan \theta t$
    satisfies $f^*(\bx) \le 0, \forall \bx \in \textbf{supp}(D^-)$.
\end{lemma}

Given \Cref{fact pos} and \Cref{fact neg}, under the $\gamma$-margin assumption, 
we know that if $D^+,D^-$ have zero means, zero third moments, 
and isotropic covariance matrices, 
then the variances of $D^+,D^-$ 
must differ by at least $\gamma^2$. 
In other words, if $D^+,D^-$ have zero means 
and matched second, third moments, 
then their third moments must differ from $0$ 
by $\Omega(\gamma^c)$. 
However, in general, we are not able to guarantee 
that the moments of $D^+,D^-$ satisfy the condition of \Cref{fact pos} 
and \Cref{fact neg} exactly. 
Thus, we need \Cref{lm 3rd moment}, 
a robust version of the above argument. 
Importantly, the polynomials we construct in \Cref{lm positive polynomial} 
and \Cref{lm negative polynomial} are stable enough 
and can still be used in our proof even if the moment conditions are perturbed. 
Thus, using these one-sided approximating polynomials 
as certificates, 
we are able to prove \Cref{lm 3rd moment}.

%% file: SQ-relevent-direction.tex

\section{Relevant Direction Extraction for Intersections of Two Halfspaces}\label{sec direction}
In \Cref{sec structure}, we developed structural results for the marginal 
distribution $D_V$ of any instance of learning intersections of two halfspaces with a margin assumption.
In this section, we will show that with these structural results, 
we are able to efficiently find a unit vector $\bw$ 
that is close to $V$ for \emph{any} factorizable distribution $D$ 
consistent with an instance of learning intersections 
of two halfspaces. As we will show later, 
with such a unit vector $\bw$, we are able to design 
non-smoothed statistical queries that can be used for weak learning.
Recall by \Cref{thm:csq-lb} that a necessary condition that makes 
CSQ algorithms not work is that low-degree moments of $D^+,D^-$ 
are nearly the same. So, in \Cref{sec matched}, we will present 
algorithms that work under this ``hard'' condition, 
while in \Cref{sec:direction-extraction-csq}, 
we will give an algorithm under the easier condition 
where the low-degree moments of $D^+,D^-$ are mismatched.


\subsection{Relevant Direction Extraction with Matched Moments}\label{sec matched}

When $D^+,D^-$ have nearly the same low-degree moments, 
the moments of both of $D^+, D^-$ look like those of $D_X$. 
As we mentioned in \Cref{sec structure}, in this case, 
the third moment of $D_X$ must deviate from $0$ significantly. 
In this step, we will make use of this property 
to perform certain unsupervised learning tasks over $D_X$ 
to find some $\bw \in S^{d-1}$ close to $V$.




\paragraph{An Efficient SQ Algorithm for Relevant Direction Extraction with Matched Moments}
Our initial attempt was inspired by independent component analysis (ICA) \citep{frieze1996learning,arora2012provable} 
and its generalization \citep{vempala2011structure}. 
The generalized ICA studies the following problem. 
Consider a distribution $D_X$ over $\R^d$, where there 
is a $k$-dimensional subspace $V$ such that $D_V$ and $D_W$ ($D_{V^\perp}$) 
are independent. Given sample access to $D_X$, ICA is asked to recover 
the subspaces $V$ and $W$. 
The core idea of generalized ICA is 
to solve some non-convex optimization task 
based on higher moments of $D_X$.
\cite{vempala2011structure} observed that if $D_X$ has 
the same first $m-1$ moments as the standard Gaussian, 
for $m \ge 3$, but has a different $m$th moment, 
then any local maximum (minimum) 
$\bu^*$ of $f^*(\bu^*)$ over $S^{d-1}$ 
with $f^*(\bu^*)> \gamma_m$ ($f^*(\bu^*)< \gamma_m$) 
must be either in $V$ or $W$. 
Here, $f^*(\bu) = \E_{\bx \sim D_X}(\bu\cdot \bx)^m$ 
and $\gamma_m$ is the $m$-th moment 
of the standard normal distribution.
In particular, if the distribution $D_W$ is a standard Gaussian, 
then any $\bu^*$ obtained above must be in $V$, 
which gives a reasonable method that finds one direction in $V$. 

However, for two general distributions $D_V, D_W$, this observation 
does not immediately give a method to find a direction $\bu \in V$, 
because we are not able to guarantee whether the local optimum 
of $f^*(\bu)$ is in $V$ or $W$. In fact, for the problem of learning 
intersections of two halfspaces, $V$ only has dimension $2$; 
but it is possible that $D_W$ is also a factorizable distribution 
that can be factorized into $\Omega(d)$ distributions, each of which is isomorphic to $D_V.$ 
Thus, information-theoretically, finding a list of $O(d)$ directions 
such that one of them is close to $V$ is the best one can achieve. 
To do this, the direct attempt is to find the first local optimum 
$\bu^{(1)}$, look at the subspace $(\bu^{(1)})^\perp$, 
find $\bu^{(2)}$, the next local optimum of $f^*$ within 
$(\bu^{(1)})^\perp$, and perform this procedure recursively $d$ times.  
This can be done because every time we make a projection, the resulted distribution is still factorizable. 
Unfortunately, such a direct approach cannot be turned into 
an efficient algorithm, and no fully polynomial time algorithm for generalized ICA is known so far. 
This is because, due to the sampling error and optimization error for optimizing $f^*$, we are only able to find an approximate solution 
for $\bu^{(1)}$, which is not in $V$ or $W$ exactly. 
Thus, the local optimum of $f^*$ restricted at $(\bu^{(1)})^\perp$ 
is not guaranteed to be a local optimum of $f^*$ over $S^{d-1}$. 
Such an error can accumulate exponentially fast with respect 
to the order in the output list (as demonstrated in \cite{vempala2011structure}).
Since we are not able to guarantee which $\bu$ 
in the list is close to $V$, in the worst case, 
our target direction could be the last few discovered 
directions in the list. 
To guarantee that these vectors are still close to $V$, 
the first several solutions must be found with error $\gamma^{-\Omega(d)}$. 
Our first result in this section shows that, 
although such a framework cannot give us a computationally efficient algorithm, we can modify it to get an
SQ-efficient algorithm that outputs a list of $\poly(d/\gamma)$ 
unit vectors such that at least one of them is $\poly(\gamma/d)$-close 
to $V$. In other words, we show that extracting one relevant direction 
can be done in a sample-efficient manner. 
Formally, we establish the following theorem 
(we defer the algorithm and the proof to Appendix~\ref{app SQ efficient}).


\begin{theorem}\label{th SQ efficient}
There is a Statistical Query learning algorithm $\A$ such that for $c>0$, a suitably large constant, 
and for an instance of learning intersections of two $\gamma$-margin 
halfspaces under factorizable distributions, 
if the input distribution $D$ satisfies the $(\gamma^c,m)$-moment matching condition for $m \in [3]$,
the algorithm makes $\poly(d/\gamma)$ many statistical queries, each of which has tolerance $\poly(\gamma/d)$, 
and outputs a direction $\bw \in \R^d$ such that with probability at least $\poly(\gamma/d)$, $\norm{\bw_W}_2 \le \poly(\gamma/d)$.    
\end{theorem}


\paragraph{Computationally Efficient Algorithm for Relevant Direction Extraction with Matched Moments}

Given the above discussion, finding a relevant direction 
via a direct non-convex optimization method is technically challenging. 
In summary, since there is no structural assumption over $D_W$, 
the function $f(\bu) = \E_{\bx \sim D_X}(\bu \cdot \bx)^3$ 
could have too many locally optimal solutions 
and some of them (including the ones that are close to $V$) 
are hard to find; this makes the error accumulate fast 
when sequentially finding each local optimum.
Thus, to avoid such an issue of error accumulation 
and get a computationally efficient algorithm, 
one hope is to find directions in $V$ and $W$ simultaneously. 
Following this idea, we give a fully-polynomial time algorithm 
that solves this task using techniques from the 
tensor decomposition literature. 
Formally, we have the following theorem.

\begin{theorem}\label{th tensor}
There is a learning algorithm $\A$ such that for every $c$, 
a suitably large constant, and any instance of learning 
intersections of two $\gamma$-margin halfspaces 
under factorizable distributions, 
if the input distribution $D$ satisfies the $(\gamma^c,m)$-moment matching condition for $m \in [3]$,
$\A$ runs in $\poly(d,1/\gamma)$ time and outputs a list of $d$ 
unit vectors $\mathcal{O}$ such that at least one direction 
$\bw \in \mathcal{O}$ satisfies $\norm{\bw_W}_2 \le \poly(\gamma)$ 
with probability $\Omega(\gamma/d)$.    
\end{theorem}

Tensor decomposition techniques usually deal with problems 
of the following type. 
Given a tensor $T \in \R^{d\times d\times d}$ 
of the form $T=\littlesum_{i=1}^k (\bv^{(i)})^{\otimes 3}$, 
recover $\bv^{(i)}, i \in [k]$ for some small $k$. A number of prior works 
address this problem from a computational point of view. 
Unfortunately, for the moment tensor of a general distribution, $k$ can be large and it can be challenging to compute the decomposition. 
This also happens for our problem. 
However, our goal is to find a direction $\bw \in V$, 
instead of doing a complete tensor decomposition. 
Assuming that $D_X$ has zero mean, then we can write 
$T^*:=\E_{\bx \sim D} \bx^{\otimes3} = \E_{\bx \sim D_V} \bx_V^{\otimes3} + \E_{\bx \sim D_W} \bx_W^{\otimes3}.$ 
Notice that for every $\bv \in \R^d$, we have 
$
    M=    T^* \cdot \bv = \E_{\bx \sim D_V} \bx_V\bx_V^T (\bx_V \cdot \bv) + \E_{x \sim D_W} \bx_W\bx_W^T(\bx_W \cdot \bv) = M_V + M_W \;,
$
where
$M_V = \E_{\bx \sim D_V} \bx_V\bx_V^T (\bx_V \cdot \bv)$ 
and $M_W = \E_{\bx \sim D_W} \bx_W\bx_W^T(\bx_W \cdot \bv)$. 
Since $V \perp W$, every eigenvector $\bw$ of $M_V$ 
must also be an eigenvector of $M$. Thus, as long as $M_W$ 
does not have a common eigenvalue as $M_V$, 
we are able to find one direction $\bw \in V$ 
using eigendecomposition algorithms.
On the other hand, if the eigenvalues of $M_V$ 
are the same as (or close to) the eigenvalues of $M_W$, 
vectors that have heavy components in both $V$ and $W$ 
can also be eigenvectors of $M$, which makes finding $\bw \in V$ hard. 
To overcome this difficulty, we choose $\bv \sim N(0,\frac{1}{d}I)$.
Such a choice makes the eigenvalues of $M_V$ and $M_W$ are independent. Importantly, by \Cref{th:intersection-structure-main}, 
$T_V$ significantly deviates from $0$. 
Thus, if we write $\bv_V = \alpha^2 \bv_V^0$, 
where $\bv_V^0$ is the direction of $\bv_V$, 
then with constant probability $M_V/\alpha^2$ 
has at least one eigenvalue $\sigma_1$ with magnitude at least $\gamma^c$. 
On the other hand, the corresponding eigenvalue $\alpha^2\sigma_1$ 
of $M_V$ is a random variable that satisfies an anti-concentration property. 
In the proof of \Cref{th tensor}, we will show that this 
anti-concentration property can make
$\alpha^2\sigma_1$ far away from any eigenvalue of $M_W$ 
with a non-trivial probability. Thus, as long as we estimate 
the moment tensor of $T^*$ up to $\poly(\gamma/d)$ accuracy, 
we are able to find a direction $\bw$ close to $V$ 
with a non-trivial probability. We defer the algorithm and the proof of \Cref{th tensor} 
to Appendix~\ref{app tensor}. 


%% file: CSQ-relevent-direction.tex

\subsection{Relevant Direction Extraction with Mismatched Moments} \label{sec:direction-extraction-csq}

In \Cref{sec matched}, we described 
an efficient algorithm that outputs a direction 
$\bw$ that is close to $V$ when $D^+$ and $D^-$ have nearly matched 
low-degree moments.
In this section, we focus on a different regime, 
where the low-degree moments of $D^+,D^-$ are not matched.
Recall that in \Cref{def moment match}, we use the $(\alpha,m)$-moment matching condition to measure the level of mismatch 
of the low-degree moments of $D^+,D^-$. 
This characterizes the difficulty of using polynomials 
to detect the correlation between the labels and the unlabeled examples. 
If for small $m$, the $(\alpha,m)$ moment matching condition 
always does not hold, then one can use the polynomial regression method 
to output a degree-$m$ Polynomial Threshold Function (PTF) 
with $\poly(\alpha)$ advantage. However, this does not imply that 
we are able to run a boosting algorithm with PTFs.  
Indeed, this would require the moment-matching condition to not 
hold throughout the process.
However, when the distribution $D$ is factorizable, 
instead of boosting using polynomials, 
our strategy will be to extract a direction $\bu$ in the relevant subspace $V$ 
by solving a carefully defined non-convex optimization problem. 
The main result we obtain is summarized 
in \Cref{thm:csq-direction-extraction}. 
We defer the full proof of \Cref{thm:csq-direction-extraction} and the corresponding algorithm to Appendix~\ref{app csq}.




\begin{theorem} \label{thm:csq-direction-extraction}
There is an algorithm $\mathcal{A}$ (\Cref{alg csq}) such that for any instance of learning intersections of two halfspaces under factorizable distributions, 
if the distribution $D$ does not satisfy the 
$(\alpha,m)$-moment matching condition 
and $D$ satisfies the $(\alpha^2d^{-c}/2^m,t)$-moment matching condition 
for any $t\le m-1, m \le 3$ and a sufficiently large universal constant $c$, 
then $\mathcal{A}$ draws 
$\poly(d, 1/\alpha)$ \iid samples from $D$, 
runs in time $\poly(d, 1/\alpha)$, 
and outputs a unit vector $\bu\in S^{d-1}$ 
such that $\norm{\bu_W} = O(\alpha)$ with probability $2/3$. 
\end{theorem}
In the rest of the section, we provide an overview 
of the proof of \Cref{thm:csq-direction-extraction}. 
Suppose that $\E_{\bx\sim D^+}(\bx_V^{\otimes t})$ 
and $\E_{\bx\sim D^-}(\bx_V^{\otimes t})$ 
are different, for some $t=m\in \N$. 
Then the tensor 
$T=(\E_{\bx\sim D^+}-\E_{\bx\sim D^-})(\bx_V^{\otimes m})$ 
is nonzero and $T\cdot \bu^{\otimes m}=0$ for all points $\bu \in W$. 
Therefore, the function $f(\bu) = \bu^{\otimes m}\cdot T$  
obtains a local maximum/minimum only if $\bu^{\otimes m}\in V^{\otimes m}$, 
which implies that $\bu\in V$. 
Unfortunately, we are not able to solve 
such an optimization problem exactly. 
As we will show in the proof, an approximate solution 
(that can be efficiently found via a gradient-descent method) 
suffices for our purposes.
Moreover, there is still a technical challenge to implement this approach.
Since $V$ is unknown to us, it is impossible for us to estimate $T$. 
However, if $\E_{\bx\sim D^+}(\bx^{\otimes t})$ and 
$\E_{\bx\sim D^-}(\bx^{\otimes t})$ are the same for all $t \le m-1$, 
then $(\E_{\bx\sim D^+}-\E_{\bx\sim D^-})(\bx^{\otimes m}) = T$, for which we can efficiently estimate with samples. 
We will show that if we take $m$ to be the smallest index such that 
$D$ does not satisfy the $(\alpha,m)$-moment matching condition, 
but satisfies the $(\alpha^2d^{-c}/2^m,t)$-moment matching condition, 
then we are able to take $\poly(d^m,1/\alpha)$ examples 
from $D^+,D^-$ to estimate $T$, 
so that any approximate solution to the estimated function 
gives a direction close to $V$. In particular, to use this approach 
to learn an intersection of two halfspaces, we only need $m\le 3$.






%% file: localization.tex

\section{Localization with the Relevant Direction and Learning Intersections of Halfspaces}\label{sec local}
In the previous sections, 
we have shown that for every factorizable distribution $D$ 
that is consistent with an instance of learning intersections of two halfspaces with a margin, we are able to efficiently find one direction $\bw$ 
that is close to the relevant subspace $V$.
 Based on these results, a natural attempt is to find 
 the next relevant direction so that we can approximately 
 find the relevant subspace $V$; 
 and do a brute-force search over all intersections of two halfspaces over $V$. 
 However, the structural result we obtained in \Cref{sec structure} 
 only allows us to find one direction. 
 Furthermore, since we make no distributional assumptions over $D_V$, 
 to make the brute-force search method succeed, 
 a small mismatch between $V$ and the approximate subspace we found 
 could lead to a large error. On the other hand, instead 
 of trying to recover the relevant subspace, 
 we make the following observation.
\begin{lemma} \label{lem:banding-PTF}
    Let $D$ be a joint distribution of $(\bx,y)$ on $\B^d(1)\times \{\pm 1\}$ that is consistent with an intersection of halfspaces with $\gamma$-margin and $\bw\in S^{d-1}$ such that $\norm{\bw_V}_2 \le c \gamma$, for some small constant $c$, where $V$ is the relevant subspace of the intersection of halspaces. Then for any band $B_t:=\{\bx\in \B^d(1)\mid \bx\cdot \bw\in [t,t+c\gamma]\}$ where $t\in \R$ and $c$ is a sufficiently small constant, 
    the distribution of $(\bx,y)$ conditioned on $\bx\in B_t$ is consistent with an instance of learning a degree-$2$ polynomial threshold function with $\Omega(\gamma^2)$-margin.
\end{lemma}

\Cref{lem:banding-PTF} states that for any instance of learning 
an intersection of two halfspaces with a margin, 
if we cut the space into bands $B_i, i=1,2,\dots$, along a direction $\bu$ 
that is close to $V$, then the distribution in each band 
is consistent with a degree-$2$ PTF. This means that some correlational 
statistical query of the form $q(\bx)= p(\bx) \Ind(\bx \in B_i)$ 
can be used to detect the correlation between $D_X$ and $D_Y$, 
and allows us to efficiently find a weak hypothesis $h$ 
with $\poly(\gamma)$-advantage. 
We state the algorithmic result for the weak learning algorithm 
in \Cref{thm:main-weak-learner}.

\begin{theorem} \label{thm:main-weak-learner}
    There is an algorithm $\A$ such that for every instance of learning intersections of two halfspaces with $\gamma$-margin, given $\bw\in S^{d-1}$ such that $\norm{\bw_W}_2 \leq c\gamma$ where $c$ is a sufficiently small constant, 
    $\A$ draws $\poly(d,1/\gamma)$ examples from $D$, 
    runs in $\poly(d,1/\gamma)$ time, and outputs a hypothesis 
    $h: \B^d(1) \to \{\pm 1\}$ such that with probability at least $2/3$, 
    $\err(h) \le 1/2-\Omega(\gamma)$.
    
\end{theorem}

\noindent We emphasize that \Cref{thm:main-weak-learner} holds {\em without the assumption 
that $D$ is factorizable}. This immediately implies that we are able 
to get an efficient strong learning algorithm via boosting algorithms~\citep{schapire2013boosting}. 
Due to space limitations, 
we defer the proofs in this section to Appendix~\ref{app local}.

\section{Conclusion}\label{sec conclusion}
The question of whether the intersection 
of two halfspaces with a margin can be 
learned in fully polynomial time is a 
central problem in Computational 
Learning Theory that has been open for 
over two decades. 
Our work makes progress on this problem 
by bypassing the previously known limitations 
through a novel algorithmic framework, 
yielding new techniques and structural insights. 
While our approach does not resolve 
the problem in full generality, 
the case of factorizable distributions 
that we consider is a fairly challenging setting 
and we expect that some of the ideas introduced here 
will be useful even beyond our factorization assumptions. 
The key and most difficult step in learning intersections of halfspaces is finding statistical queries that enable weak learning. Our approach to this is to identify a direction that is close to the relevant subspace. 
Most known learning lower-bounds involving statistical learning algorithms are based on hiding a subspace among irrelevant directions. 
Our results show that one can efficiently address these cases 
using SQ algorithms and establish that common SQ lower-bound 
constructions are not applicable to our setting. Thus, even if 
an SQ  
lower bound exists, it would require 
a novel construction with non-factorizable distributions. 
Furthermore, our algorithmic result establishes a strong 
separation between CSQ and SQ algorithms for \emph{weakly} 
realizable PAC learning. 
While it is known that SQ is needed for efficient 
strong realizable PAC learning, 
our work gives the first natural setting 
where SQ is even necessary for efficient weakly learning. 
Our learning framework builds on such a separation, and we 
expect understanding such a separation may lead to faster 
algorithms for learning other hypothesis classes. 




%% file: appendix.tex

\appendix

\section*{Appendix}

\paragraph{Structure of Appendix}
We give an overview of the structure of the appendix. 
In Appendix~\ref{sec notation}, we provide a complete list of notations and preliminaries. In Appendix~\ref{app structure}, we provide missing proofs and discussions in \Cref{sec structure}. In Appendix~\ref{app direction}, we provide omitted proofs in \Cref{sec direction}. In Appendix~\ref{app local}, we provide omitted proofs in \Cref{sec local}. In Appendix~\ref{app main}, we give a complete description of our main algorithm and provide the proof of \Cref{th main}. In Appendix~\ref{app csq}, we give a complete proof of \Cref{thm:csq-lb}, the CSQ lower bound for learning intersections of two halfspaces under factorizable distributions. 

\section{Preliminaries and Notations}\label{sec notation}
In this section, we present a complete list of notations, preliminaries and related background on the statistical learning model.

\paragraph{Basic Notations}
In this paper, we use small boldface characters for vectors and use capital lightface characters for subspaces, matrices and tensors. For $n \in \Z_+$, we denote by $[n]:=\{1,\dots,n\}$. For $\bx \in \R^d$, and $i \in [d]$, we use $\bx_i$ to denote the $i$-coordinate of $\bx$. For $i \in [d]$, we denote by $\be_i$ the $i$-th standard basis of $\R^d$.
Let $V \subseteq \R^d$ be a subspace, we denote by $\bx_V: = \proj_V (\bx)$, the projection of $\bx$ onto the subspace $V$ and denote by $V^\perp$ the orthogonal complement of $V.$ For $\bu,\bv \in \R^d$, we use $\bu \cdot \bv$ to denote the inner product of $\bu$ and $\bv$ and we use $\norm{\bu}_2$ to denote the $\ell_2$ norm of $\bu$. We use $S^{d-1}=\{\bx\in\R^n:\|\bx\|_2=1\}$ to denote the $d$-dimensional unit sphere and $B^d(r)$ the $d$-dimensional ball with radius $r$.

For any distribution $D$, we use $\E_{\bx\sim D}(\bx)$ to denote the expectation of $D$. Let $D$ be a distribution of $(\bx,y)$ over $\R^d \times \{\pm 1\}$. We use $D_X$ to denote the marginal distribution of $D$ over $\R^d$.
For any subspace $V \subseteq \R^d$, we use $D_V$ to denote the marginal distribution of $D$ for $\bx_V$. For $z \in \{\pm 1\}$, we use $D^z_V$ to denote the marginal distribution of $D$ over $\bx_V$ condition on $y=z$ and use $D^z$ to denote the marginal distribution $D_X$ over $\bx$ condition on $y=z$
For a distribution $D_X$ over $\R^d$, we say $D_X$ has an isotropic covariance matrix if there is some $\alpha\ge 0$ such that $\E_{\bx \sim D_X}\bx\bx^\intercal = \alpha I$, where $\alpha$ is called the scale of $\E_{\bx \sim D_X}\bx\bx^\intercal$.

For tensors, 
we will consider a $k$-tensor to be an element in $(\mathbb{R}^n)^{\otimes k}\cong\mathbb{R}^{n^k}$.
A symmetric tensor is a tensor that is invariant under a permutation of its vector arguments. We use $\norm{T}_F$ to denote the Frobinius norm of $T$.
We will use $T_{i_1,\ldots,i_k}$ to denote the coordinate of a $k$-tensor $T$ indexed by the $k$-tuple $(i_1,\ldots,i_k)$. 
For a tensor $T\in (\R^d)^{\otimes m}$ and $\pi:[m]\to [m]$ be a permutation of indices,
    we use $\pi(T)$ to denote the tensor permuted by $\pi$ defined as 
    $\pi(T)_{(i_1),\cdots,(i_m)}=T_{\pi(i_1),\cdots,\pi(i_m)}$.
    We define $\sym(T)$ as $\frac{1}{m!}\sum_{\pi\in \Pi}\pi(T)$, where $\Pi$ is the set of all possible permutations of $[m]$.
By abuse of notation,
we will sometimes treat a tensor $T\in (\R^d)^{\otimes m}$ as a linear mapping, i.e., for $\bv\in \R^d$, we use $T\cdot\bv$ to denote applying the linear mapping $T:\R^d\to (\R^d)^{\otimes m-1}$ specified by $T$ on $\bv$.
For a vector $\bv\in\R^n$, we denote by $\bv^{\otimes k}$ to be a vector (linear object) in $\R^{n^k}$.
For a matrix $M\in\R^{n\times m}$, we denote by $\|M\|_2,\|M\|_F$ to be the operator norm and Frobenius norm respectively.


We present the following fact that will be useful in the analysis of our algorithms.
\begin{fact} \label{fct:lb-tensor-random-rank1-correlation}
    Let $T\in (\R^d)^{\otimes m}$ and $\|T\|_F=1$ be a symmetric tensor for $m\leq 3$, then $\max_{\bu\in S^{d-1}} \bu^{\otimes m}\cdot T\geq 1/\poly(d)$.
\end{fact}


\begin{proof}[Proof of \Cref{fct:lb-tensor-random-rank1-correlation}]

    The statement trivially holds for $m=1,2$. Therefore, we only need to consider the case $m=3$. 
    We first show that any symmetric $T\in (\R^d)^{\otimes 3}$ can be written as 
    $
    T=\sum_{i=1}^N \alpha_i \bu_i^{\otimes 3}\; ,
    $
    where $N=\poly(d)$, $\sum_{i=1}^N |\alpha_i|=\poly(d)$ and each $\bu_i$ is a unit vector.
    Suppose we have shown that this is true, then it is easy to see that 
    $1= T\cdot T =\sum_{i=1}^N \alpha_i {\bu_i}^{\otimes 3}\cdot T\leq \poly(d)\max_i ( {\bu_i}^{\otimes 3}\cdot T)$, therefore at least one $\bu_i$ satisfies ${\bu_i}^{\otimes 3}\cdot T=1/\poly(d)$.
    
    Therefore, we just need to show that the statement above about decomposition is true.
    Since $\sym(\bv_i\otimes \bv_j\otimes \bv_k)$ where $\bv_i, \bv_j,\bv_k$ are standard basis vectors span the space of symmetric tensors.
    Therefore, it suffices for us to show that the statement holds true for any $T=\sym(\bv_i\otimes \bv_j\otimes \bv_k)$.
    Notice that for any $\bu,\bv\in \R^d$
    \begin{align*}
    &(\bu+\bv)^{\otimes 3}-\bu^{\otimes 3}-\bv^{\otimes 3}\\
    =& \bu\otimes \bv^{\otimes 2}+\bv\otimes \bu\otimes \bv+\bv^{\otimes 2}\otimes \bu\\
    &+\bv\otimes \bu^{\otimes 2}+\bu\otimes \bv\otimes \bu+\bu^{\otimes 2}\otimes \bv\\
    =&3\sym(\bu\otimes \bv^{\otimes 2})+3\sym(\bv\otimes \bu^{\otimes 2})\; , 
    \end{align*}
    and
    \begin{align*}
    &(\bu-\bv)^{\otimes 3}-\bu^{\otimes 3}+\bv^{\otimes 3}\\
    =& \bu\otimes \bv^{\otimes 2}+\bv\otimes \bu\otimes \bv+\bv^{\otimes 2}\otimes \bu\\
    &-(\bv\otimes \bu^{\otimes 2}+\bu\otimes \bv\otimes \bu+\bu^{\otimes 2}\otimes \bv)\\
    =&3\sym(\bu\otimes \bv^{\otimes 2})-3\sym(\bv\otimes \bu^{\otimes 2})\; . 
    \end{align*}
    Taking the difference of the above two equations shows that the decomposition statement is true for any $T=\sym(\bv\otimes \bu^{\otimes 2})$.
    Therefore, we just need to decompose $\sym(\bv_i\otimes \bv_j\otimes \bv_k)$  as linear combination of tensors of the form $\bv^{\otimes 3}$ and $\sym(\bv\otimes \bu^{\otimes 2})$.
    Then notice that since $\sym$ is a linear operator 
    \begin{align*}
    &\sym(\bv_i\otimes (\bv_j+\bv_k)^{\otimes2})-\sym(\bv_i\otimes \bv_j^{\otimes 2})-\sym(\bv_i\otimes \bv_k^{\otimes 2})\\
    =&\sym(\bv_i\otimes (\bv_j+\bv_k)^{\otimes2}-\bv_i\otimes \bv_j^{\otimes 2}-\bv_i\otimes \bv_k^{\otimes 2})\\
    =&\sym(\bv_i\otimes \bv_j\otimes \bv_k+\bv_i\otimes \bv_k\otimes\bv_j)\\
    =&\sym(\bv_i\otimes \bv_j\otimes \bv_k)\; .        
    \end{align*}
    This completes the proof.
\end{proof}


\paragraph{Background on Statistical Query Model}
SQ algorithms are a class of algorithms that are allowed
to query expectations of bounded functions on the underlying distribution 
through an (SQ) oracle
rather than directly access
samples. The model was introduced by \cite{Kearns:98} as a natural restriction of the PAC 
model~\citep{Valiant:84} in the context of learning Boolean functions. 
Since then, the SQ model has been extensively studied in a range of settings, including
unsupervised learning~\citep{Feldman16b}.
The class of SQ algorithms is broad and captures a range of known 
algorithmic techniques in machine learning including spectral techniques,
moment and tensor methods, local search (e.g., EM),
and many others (see, e.g.,~\cite{FeldmanGRVX17, FeldmanGV17} 
and references therein).

\begin{definition}[SQ Model] \label{def:sq}
Let $D$ be a distribution on $X$. 
A \emph{statistical query} is a bounded function $q:X\rightarrow[-1,1]$. 
We define $\mathrm{STAT}(q,\tau)$ to be the oracle that given any such query $q$, outputs a value $v$ such that $|v-\E_{\bx\sim D}[q(\bx)]|\leq\tau$, where $\tau>0$ is the \emph{tolerance} parameter of the query.
A \emph{statistical query (SQ) algorithm} is an algorithm 
whose objective is to learn some information about an unknown 
distribution $D$ by making adaptive calls to the corresponding $\mathrm{STAT}(q,\tau)$ oracle.
\end{definition}


\paragraph{Basics of Correlational Statistical Query(CSQ) Model}

In particular, given $D$ is a distribution on $X\times \{-1, 1\}$, we can define the Correlational Statistical Query (CSQ) model as a further restriction of the SQ model. 

\begin{definition} [CSQ Model] \label{def:csq}
Let $D$ be a distribution on $X\times \{-1,1\}$. 
A \emph{correlational statistical query} is a bounded function $q:X\times \{-1,1\}\rightarrow[-1,1]$. 
We define $\mathrm{CSTAT}(\tau)$ to be the oracle that given any such query $q$, outputs a value $v\in [-1,1]$ such that $|v-\E_{(\bx,y)\sim D}[yq(\bx)]|\leq\tau$, where $\tau>0$ is the \emph{tolerance} parameter of the query.
A \emph{statistical query (SQ) algorithm} is an algorithm 
whose objective is to learn some information about an unknown 
distribution $D$ by making adaptive calls to the corresponding $\mathrm{STAT}(q,\tau)$ oracle. 
\end{definition}

\begin{definition} [Function Representation of Distribution for CSQ]
Let $D$ be a joint distribution of $(\bx,y)$ supported on $\R^d\times \{\pm 1\}$ where $D^+$ and $D^-$ has probability density functions $P_{D^+}, P_{D^-}:\R^d\to \R_+$.
Let $D_0$ be a distribution on $\R^d$ with density function $P_{D_0}:\R^d\to \R+$ where the support of $D$ contains the support of $D^+$ and $D^-$.
Then, the function representation of $D$ for CSQ w.r.t. $D_0$ is defined as a function $f_{D,D_0}:\R^d\to \R$ such that $f_{D,D_0}(\bx)=(P_{D^+}(\bx)-P_{D^-}(\bx))/P_{D_0}(\bx)$.
\end{definition}

\begin{definition} [Pairwise Correlation]
For functions $f,g:\R^d\mapsto \R_+$, we defined the correlation between $f$ and $g$ under the distribution $D_0$ to be the expectation 
$\E_{\bx\sim D_0}[f(\bx)g(\bx)]$.
\end{definition}

\begin{definition}
We say that a set of functions $F$ mapping $\R^d\to \R$ is 
$(\gamma,\beta)$-correlated relative to a distribution $D_0$ if for any $f_i,f_j\in F$,
the correlation $\E_{\bx\sim D_0}[f_i(\bx)f_j(\bx)]\leq  \gamma$ for all $i\neq j$
and $\E_{\bx\sim D_0}[f_i(\bx)f_j(\bx)]\leq  \beta$ for $i=j$.
\end{definition}

\begin{definition} [Decision Problem over Distributions] \label{def:decision-problem}
Let $D$ be a fixed distribution and $\cal D$ be a distribution family. 
We denote by $\mathcal{B}(\D,D)$ the decision problem in which 
the input distribution $D'$ is promised to satisfy either (a) $D'=D$ or (b) $D'\in \D$,
and the goal is to distinguish the two cases with high probability.
\end{definition}

\begin{definition} [Correlational Statistical Query Dimension]
For $\beta,\gamma>0$, a decision problem $\mathcal{B}(\D,D)$, where $D$ is a fixed distribution and 
$\D$ is a family of distribution both over $X\times \{\pm 1\}$, and $f_{D,D_0}\equiv 0$.
Let $s$ be the maximum integer such that
there exists a finite set of distributions $\D_{D}\subseteq\D$ such that 
$\{f_{D, D_0}\mid D\in \D_{D}\}$ is $(\gamma,\beta)$-correlated relative to $D_0$ and $|\D_{D}|\geq s$. 
The Correlational Statistical Query dimension with pairwise correlations $(\gamma,\beta)$ of $\mathcal{B}$ is defined to be $s$, 
and denoted by $s=\mathrm{CD}(\mathcal{B},\gamma,\beta)$.
\end{definition}

\begin{lemma}\label{lem:sq-lb}
Let $\mathcal{B}(\D,D)$ be a decision problem, where $D$ is 
the reference distribution and $\D$ is a class of distribution. For $\gamma,\beta>0$, 
let $s=\mathrm{CD}(\mathcal{B},\gamma,\beta)$. 
For any $\gamma'>0$, any CSQ algorithm for $\mathcal{B}$ requires queries of tolerance 
at most $\sqrt{\gamma+\gamma'}$ or makes at least $s\gamma'/(\beta-\gamma)$ queries.
\end{lemma}

\section{Omitted Proofs from \Cref{sec structure}}\label{app structure}
In this section, we present missing details in \Cref{sec structure}.

\subsection{Proof of \Cref{th one-side}}\label{app one-side}
In this section, we give the proof of \Cref{th one-side}. For convenience, we restate \Cref{th one-side} below.

\begin{theorem}[restatement of \Cref{th one-side}]\label{th one-side re}
For any $d,m\in \N$, $C\subseteq\R^d$, $T_{i}\in \sym((\R^d)^{\otimes i})$ for $i\in [m]$ and $\tau\in \R_{\geq 0}$, at most one of the following conditions can be satisfied:
\begin{enumerate}
    \item [a)] there exists a distribution $D$ supported on $C$ such that $\|\E_{\bx\sim D}(\bx^{\otimes i})-T_{i}\|_F\leq \tau$ for any $i\in [m]$; 
    \item [b)] there exists a degree-$m$ polynomial $p:\R^d\to \R$ defined as $p(\bx)=\sum_{i=1}^m  A_i\cdot\bx^{\otimes i}$ where 
    $p(\bx)\geq 0$ for any $\bx\in C$,
    $\sum_{i=1}^m\|A_i\|_F\leq 1$ and $\sum_{i=1}^m  A_i\cdot T_i<-\tau$.
\end{enumerate}
We call such a polynomial above a one-sided approximation polynomial for $C$ w.r.t. to moments information $T_i$ and tolerance $\tau$.
\end{theorem}

\begin{proof}[Proof of \Cref{th one-side re}]
We prove \Cref{th one-side re} by contradiction. Suppose that the two conditions in \Cref{th one-side re} are satisfied simultaneously. Since for every $\bx \in C$, $p(x)\ge 0$, we know that $\E_{\bx\sim D}p(\bx)\ge 0$. On the other hand, we have 
\begin{align*}
    \E_{\bx \sim D} p(\bx) & = \sum_{i=1}^m  A_i\cdot\E_{\bx \sim D} \bx^{\otimes i} = \sum_{i=1}^m  (A_i\cdot T_i) + \sum_{i=1}^m  A_i \cdot (\E_{\bx \sim D} \bx^{\otimes i}-T_i) \\
    &< -\tau + \sum_{i=1}^m  A_i \cdot (\E_{\bx \sim D} \bx^{\otimes i}-T_i) \le 0.
\end{align*}
This gives a contradiction.    

\end{proof}

\subsection{Discussion on Structural Assumptions made in \Cref{sec structure}} \label{app parameter}


In this section, we explain \Cref{as parameter} as well as other structural assumptions made in \Cref{sec structure} can be made without loss of generality.

We first argue that we can assume $V$, the subspace spanned by $\bu^*,\bv^*$ is exactly equal to $\textbf{span}\{\be_1,\be_2\}$.
If $\bu^*,\bv^*$ are parallel to each other, then the problem degenerates to the problem of learning a single halfspace or learning a degree-2 polynomial threshold function. Furthermore, since any rotation matrix $U$ will not change the inner product between two points in $\R^d$, if $V \neq \textbf{span}\{\be_1,\be_2\}$, we can apply a rotation matrix $U$ that maps $\bu^*,\bv^*$ to $\textbf{span}\{\be_1,\be_2\}$, and every example $U^\intercal\bx$ still satisfies $\gamma$-margin assumption and has the same label as $\bx$. Based on this, in the rest of the section, we argue that \Cref{as parameter} can be made without loss of generality.

\begin{assumption}[restatement of \Cref{as parameter}]\label{as parameter re}
    Given an intersection of two halfspaces $h^*=\sign(\bu^*\cdot \bx + t_1) \wedge \sign(\bv^* \cdot \bx + t_2)$ and a distribution $D$ over $\B^2(1) \times \{\pm 1\}$ that consistent with $h^*$ with the $\gamma$-margin condition, we parameterize $h^*$ by $\theta \in (0,\pi/2), t \ge 0, \sigma\ge 0$ where $\bu^* = \sin \theta \be_1 - \cos \theta \be_2, t_1 = t\sin\theta$ and $\bv^* = \sin \theta \be_1 + \cos \theta \be_2, t_2 = (1+\sigma)t\sin\theta$. Furthermore, we assume $\norm{\E_{\bx\sim D^+} \bx}_2 \le \gamma^c$, $t_1,t_2 \ge \gamma$, $|t_1|,|t_2|\leq 1$, where $c$ is some large constant.
\end{assumption}

First, we argue that we can without loss of generality assume $\norm{\E_{\bx\sim D^+} \bx}_2 \le \gamma^c$, for any large constant $c$. Assuming $\norm{\E_{\bx\sim D^+} \bx}_2 \ge \gamma^c$ instead, by drawing $\poly(d/\gamma)$ positive examples from $D_X^+$, we are able to estimate some $\hat{\bx} \in V$ such that $\norm{\hat{\bx}-\E_{\bx\sim D^+} \bx}_2 \le \gamma^c/2$. Since each example $\bx$ has $\norm{\bx}_2 \le 1$, we know that $\norm{\bx-\hat{\bx}}_2 \le 2$, thus by rescaling, $(\bx-\hat{\bx})/2$ satisfies $\gamma/2$-margin assumption and the resulting positive example has mean $\gamma^c/2$-close to the origin. 

Consider the target halfspace $h^*=\sign(\bu^*\cdot \bx + t_1) \wedge \sign(\bv^* \cdot \bx + t_2)$. We furthermore argue we can without loss of generality make the following two assumptions on $h^*$
\begin{enumerate}
    \item Under the assumption $\norm{\E_{\bx\sim D^+} \bx}_2 \le \gamma^c$, $t_1,t_2 \ge 0$. This is because if $t_1 \le 0$ (without loss of generality), then every positive example $\bx$ satisfies $\bu^*\cdot \bx\ge \gamma$, which implies $\norm{\E_{\bx\sim D^+} \bx}_2 \ge \gamma$. 
\item $|t_1|,|t_2|\leq 1$. If this is not the case, then the problem degenerates to learning a single halfspace and can be solved trivially.

\end{enumerate}

Given the above assumptions, it will be convenient for us to 
parameterize $h^*=\sign(\bu^*\cdot \bx + t_1) \wedge \sign(\bv^* \cdot \bx + t_2)$ to be described by $t\in [0,\infty)$, $\theta\in [0,\pi/2)$ and $\sigma\in [0,\infty)$ where we have $\bu^* = \sin \theta \be_1 - \cos \theta \be_2$, $t_1 = t\sin\theta$ and $\bv^* = \sin \theta \be_1 + \cos \theta \be_2$, $t_2 = (1+\sigma)t\sin\theta$ (as illustrated in \Cref{fig geometry re}). This will be convenient for later calculations.

\begin{figure}
    \centering
    \includegraphics[width=0.8\linewidth]{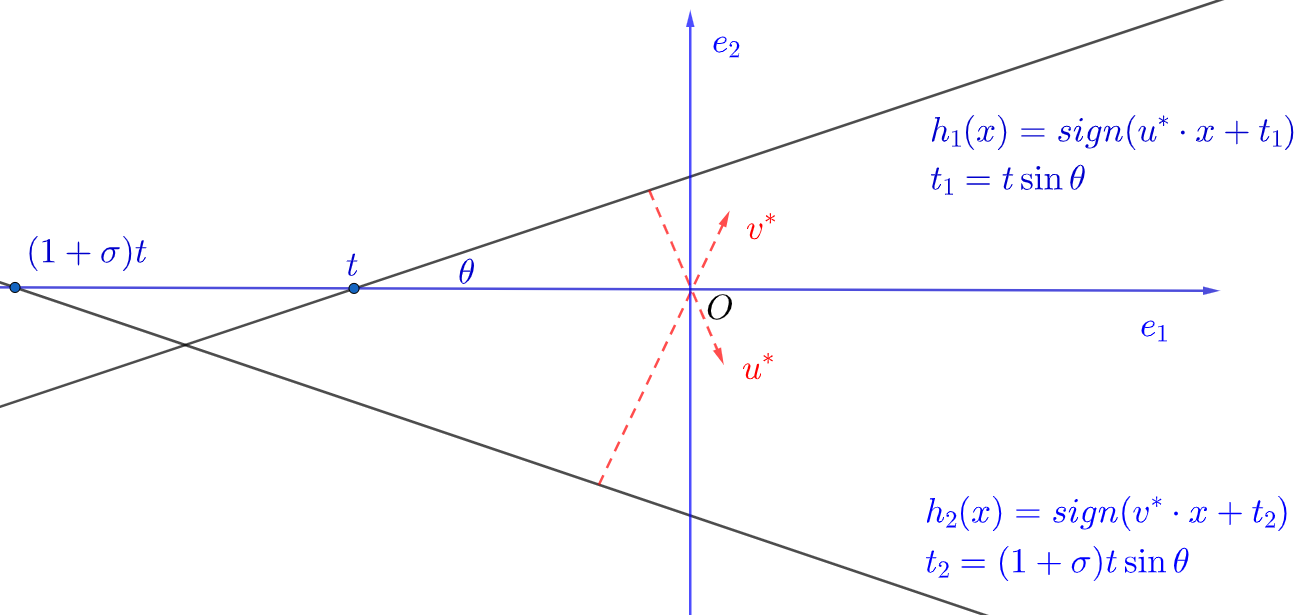}
    \caption{Geometrical Illustration of \Cref{as parameter}. Two halfspaces $h_1=\sign(\bu^*\cdot \bx+t_1)$ and $h_2=\sign(\bv^*\cdot \bx+t_2)$ are colored in black. Red dashed lines represent the directions of weight vectors $\bu^*,\bv^*$.}
    \label{fig geometry re}
\end{figure}

Notice that for every $\bu^*,\bv^*$ such that $\theta(\bu^*,\bv^*)=\pi-2\theta$, there is a rotation matrix $U$ such that $U\bu^*=\sin \theta \be_1 - \cos \theta \be_2, \bv^*=\sin \theta \be_1 + \cos \theta \be_2$. Since rotation matrix $U$ maintains the $\gamma$-margin assumption, parameterizing $h^*$ in such a way does not lose the generality.

\subsection{Proof of \Cref{lm 2nd}}\label{app lm 2nd}

In this section, we present the proof of \Cref{lm 2nd}. For convenience, we restate \Cref{lm 2nd} as follows.


\begin{lemma}[restatement of \Cref{lm 2nd}]\label{lm 2nd re}
    Let $D$ be a distribution over $\B^2(1) \times \{\pm 1\}$ that is consistent with an instance of learning intersections of two halfspaces with $\gamma$-margin assumption, where $\gamma$ is smaller than some sufficiently small constant. Let $c>0$ be any suitable large constant.
    Suppose 
        \begin{enumerate}
            \item $\norm{\E_{\bx\sim D^+} \bx}_F\le \gamma^c, \norm{\E_{\bx\sim D^-} \bx}_F\le \gamma^c$, $\norm{\left(\E_{\bx\sim D^+}-\E_{\bx \sim D^-}\right) \bx}_F\le \gamma^c$
            \item $\norm{ \left(\E_{\bx\sim D^+} - \E_{\bx\sim D^-}\right)\bx\bx^\intercal}_F \le \gamma^c$,
    \end{enumerate}
    then $\norm{(\Sigma^+)^{-1}}_2= O(1/\gamma^{4})$ and $\norm{(\Sigma^-)^{-1}}_2= O(1/\gamma^{4})$, where $\Sigma^+:=\E_{\bx\sim D^+}\bx\bx^\intercal$ and $\Sigma^-:=\E_{\bx\sim D^-}\bx\bx^\intercal$.
\end{lemma}

We first give some high-level intuition behind \Cref{lm 2nd re}.
For the purpose of contradiction, we assume that $\norm{(\Sigma^+)^{-1}}_2> \Omega(1/\gamma^{c})$ or $\norm{(\Sigma^-)^{-1}}_2> \Omega(1/\gamma^{c})$.
Therefore, there must be a unit vector $\bv$ such that $\bv^\intercal\Sigma^+\bv\le O(\gamma^{c})$ or $\bv^\intercal\Sigma^+\bv\le O(\gamma^{c})$.
However, since $\Sigma^+$ and $\Sigma^-$ are close to each other in Frobinous norm, it must be that $\bv^\intercal\Sigma^+\bv\le O(\gamma^{c})$ and  
$\bv^\intercal\Sigma^-\bv\le O(\gamma^{c})$. 
Roughly speaking, this means most of the samples are inside a thin band along the direction of $\bv^\perp$, i.e., inside the band region $B:=\{\bx\in \R^2\mid |\bx\cdot \bv|\leq  \gamma/2\}$.
Let $f^*(\bx)=\sgn(\bu^*\cdot x+t_1) \wedge \sgn(\bv^* \cdot x+t_2)$ be the true concept function, and let $A^+$, $A^-$ be the region of $\{\bx\in \R^2\mid f^*(\bx)=1\}$ and $\{\bx\in \R^2\mid f^*(\bx)=-1\}$ respectively.
Notice that there are two cases for this band region $B$(see \Cref{fig:covariance} for illustration): either
\begin{enumerate}
    \item  $\sign(\bu^*\cdot \bv^\perp)= \sign(\bv^*\cdot \bv^\perp)$. In this case, we show that the first moments of $D^+, D^-$ are not close along the direction of $v^\perp$.
    \item $\sign(\bu^*\cdot \bv^\perp)\neq \sign(\bv^*\cdot \bv^\perp)$.
    In this case, we show that the moment information must differ by giving a degree-2 polynomial $p$ that $\E_{\bx\sim D^+}[p(\bx)]$ and $\E_{\bx\sim D^-}[p(\bx)]$ differs from each other, where we choose this $p(\bx):=(\bu^*\cdot x+t_1)(\bv^* \cdot x+t_2)$.
\end{enumerate}
In both cases, this contradicts the assumption that $D$ is a moment-matching distribution.
We give the formal proof of \Cref{lm 2nd re} below.

\begin{figure}
    \centering
    \includegraphics[width=0.8\linewidth]{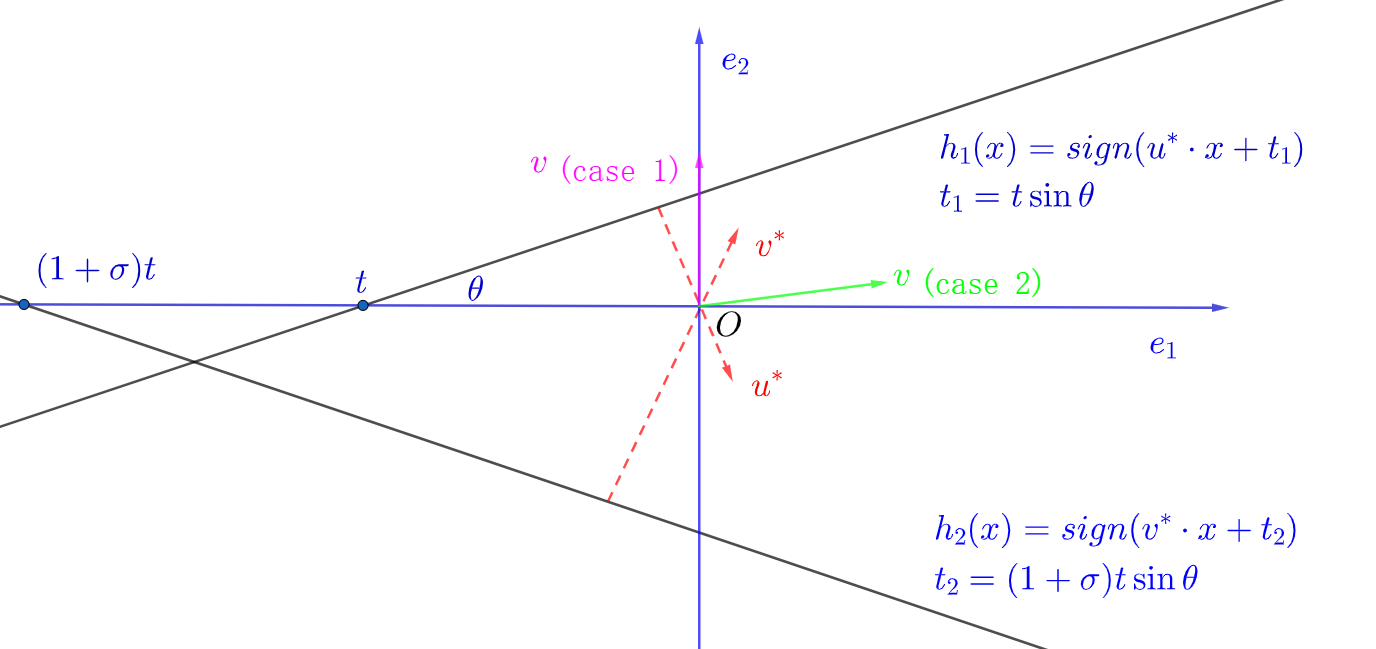}
    \caption{Geometrical illustration for the proof of \Cref{lm 2nd}. The vector colored in purple corresponds to case 1 in the proof and the vector colored in green corresponds to case 2 in the proof.}
    \label{fig:covariance}
\end{figure}
\begin{proof}[Proof of \Cref{lm 2nd re}]
To prove the statement, it suffices for us to show that there exists a universal constant $c'>0$, given 
    \begin{enumerate}
        \item $\norm{\E_{\bx\sim D^+} \bx}_F\le c'\gamma^4/100, \norm{\E_{\bx\sim D^-} \bx}_F\le c'\gamma^4/100$ and
        \item $\norm{\E_{\bx\sim D^+}\bx\bx^\intercal - \E_{\bx\sim D^-}\bx\bx^\intercal}_F \le c'\gamma^4/100$,
    \end{enumerate}
then for any $\bv\in S^{1}$, $\bv^\intercal\Sigma^- \bv\geq c'\gamma^4$.  
Suppose we can prove the above. 
Then given $c$ is a sufficiently large constant and $\gamma$ is at most a sufficiently small constant in \Cref{lm 2nd re}, the assumption of the above statement must be satisfied.
Therefore, we must have 
$\bv^\intercal\Sigma^- \bv\geq c'\gamma^4$ and $\bv^\intercal\Sigma^-\bv=\bv^\intercal\Sigma^+\bv -\bv^\intercal(\Sigma^- -\Sigma^+)\bv\ge c'\gamma^4/2$, 
which implies that $\norm{(\Sigma^+)^{-1}}_2=O(1/\gamma^{4})$ and $\norm{(\Sigma^-)^{-1}}_2=O(1/\gamma^{4})$.

We will prove $\bv^\intercal\Sigma^- \bv\geq c'\gamma^4$ for two cases.
Let $\bv^{\perp}$ be the unique unit vector up to negation that $\bv^{\perp}\cdot \bv=0$.
We consider the cases that:
\begin{enumerate}
    \item $\sign(\bu^*\cdot \bv^\perp)= \sign(\bv^*\cdot \bv^\perp)$ or (either $\bu^*\cdot \bv^\perp$ or $\bv^*\cdot \bv^\perp=0$), and \label{cas:ill-condition-same-sign}
    \item $\sign(\bu^*\cdot \bv^\perp)\neq \sign(\bv^*\cdot \bv^\perp)$.\label{cas:ill-condition-different-sign}
\end{enumerate}

For the case $\sign(\bu^*\cdot \bv^\perp)= \sign(\bv^*\cdot \bv^\perp)$, let $\bv^\perp$ be the unit direction that $\bv^{\perp}\cdot \bv=0$, $\bu^*\cdot \bv^\perp\ge 0$ and $\bv^*\cdot \bv^\perp\ge 0$.
Take $c'>0$ to be a sufficiently small constant and
assume for the purpose of contradiction that $\bv^\intercal\Sigma^- \bv\leq c'\gamma^3$.
Let $B:=\{\bx\in \B^2(1)\mid |\bx\cdot \bv|\leq  \gamma/2\}$.
Then notice that by Markov's inequality,
\[\pr_{\bx\sim D^-}[\bx\not \in B]\leq \E_{\bx\sim D^-}[(\bx\cdot \bv)^2]/(\gamma^2/4)\leq \bv^\intercal\Sigma^-\bv/(\gamma^2/4)\leq c\gamma\; ,\]
where $c$ is a sufficiently small constant.
Now, we show that for any $\bx\in B$ such that $\bx\cdot \bu^*+t_1\leq -\gamma$, $\bx\cdot \bv^{\perp}\leq -\gamma/2$.
First notice that 
\begin{equation} \label{eq:ill-condition-decomposition}
\bx\cdot\bu^*=\bx\cdot \proj_\bv\bu^*+\bx\cdot \proj_{\bv^{\perp}} \bu^*=
(\bx\cdot \bv)(\bv\cdot \bu^*)+(\bx\cdot \bv^{\perp})(\bv^{\perp}\cdot \bu^*)\; .
\end{equation}
Suppose $\bu^*\cdot \bv^\perp=0$, then we immediately get $\bx\cdot\bu^*+t_1\ge -|\bx\cdot\bv|\ge -\gamma/2 $. Therefore, no such $\bx$ exists, and we can assume without loss of generality that
$\bu^*\cdot \bv^\perp>0$.
Furthermore, notice that
for any $\bx\in B$ and $\bx\cdot \bv^*+t_1\leq -\gamma$,
given \Cref{eq:ill-condition-decomposition}, we get
\[
\bx\cdot \bv^{\perp}=\frac{\bx\cdot\bu^*-(\bx\cdot \bv)(\bv\cdot \bu^*)}{\bv^{\perp}\cdot \bu^*}
\leq \frac{-\gamma-(\bx\cdot \bv)(\bv\cdot \bu^*)}{\bv^{\perp}\cdot \bu^*}
\leq -\gamma-(\bx\cdot \bv)(\bv\cdot \bu^*)\leq -\gamma/2\; ,\]
where the third from the last inequality follows from $\bx\cdot \bu^*\leq -\gamma-t_1\leq -\gamma$
and the last inequality follows from $\bx\in B$ and $|\bv\cdot \bu^*|\leq 1$.
Similarly, we also have that
for any $\bx\in B$ and $\bx\cdot \bv^*+t_2\leq -\gamma$, $\bx\cdot \bv^{\perp}\leq -\gamma/2$.
Combining the above two and the $\gamma$-margin condition gives that, 
for any $\bx\in B$ and $\bx\in \mathbf{supp}(D^-)$, $\bx\cdot \bv^{\perp}\leq -\gamma/2$.
Then we get 
\begin{align*}
\E_{\bx\sim D^-} [\bv^{\perp}\cdot \bx]=&\E_{\bx\sim D^-} [\bv^{\perp}\cdot \bx|\bx\in B]\pr_{\bx\sim D^-}[\bx\in B]+\E_{\bx\sim D^-} [\bv^{\perp}\cdot \bx|\bx\not\in B]\pr_{\bx\sim D^-}[\bx\not\in B]\\
\le & (-\gamma/2)(1-c\gamma)+c\gamma\\
\le & -\gamma/4+\gamma/8
\le -\gamma/8\; ,
\end{align*}
where that last inequality follows from that $c$ is a sufficiently small constant.
This contradicts that $\norm{\E_{\bx\sim D^-} \bx}_F\le c'\gamma^4/100$ in the assumption.
Therefore, we must have $\bv^\intercal\Sigma^- \bv\geq c'\gamma^3\geq c'\gamma^4$. This proves the statement for Case~\ref{cas:ill-condition-same-sign}.

For the case that $\sign(\bu^*\cdot \bv^\perp)\neq \sign(\bv^*\cdot \bv^\perp)$,
let $B:=\{\bx\in \B^2(1) \mid \abs{\bv\cdot \bx} \le \gamma/2 \}$. Let the $c'$ in the statement be a sufficiently small constant.
Suppose we can prove that $\pr_{\bx\sim D^-}[\bx\not\in B]= \Omega(\gamma^2)$, then we are done, since from the definition of $B$, we immediately get 
\[
\bv^\intercal\Sigma^-\bv\ge\pr_{\bx\sim D^-}[(\bx\cdot \bv)^2\mid \bx\not\in B]\pr_{\bx\sim D^-}[\bx\not \in B]=\Omega(\gamma^4)\; .
\]
To show that $\pr_{\bx\sim D^-}[\bx\not\in B]= \Omega(\gamma^2)$, we consider the 
degree-2 polynomial $p(\bx):={(\bu^*\cdot \bx+t_1)(\bv^*\cdot \bx+t_2)}$. 
By \Cref{as parameter}, we know that $t_2 \le 1, t_1 \le 1$. 
Therefore,
\begin{align*}
    &\abs{\E_{\bx\sim D^+}p(\bx) - \E_{\bx\sim D^-}p(\bx)}\\ 
    = &\abs{(\bu^*)^\intercal(\Sigma^+-\Sigma^-)\bv^* + (t_2 \bu^*+t_1 \bv^*) \cdot (\E_{x\sim D^+}x_V-\E_{x\sim D^-}x_V)} \\
    = &\abs{\|(\bu^*)^\intercal\bv^*\|_F\norm{\Sigma^+-\Sigma^-}_F + \|t_2 \bu^*+t_1 \bv^*\|_2 \norm{\E_{x\sim D^+}x_V-\E_{x\sim D^-}x_V}_2} 
    \le c\gamma^4\; ,
\end{align*}
where $c$ is a sufficiently small constant.
From the $\gamma$-margin assumption, 
we have $\E_{\bx\sim D^+}p(\bx)\geq \gamma^2$. Therefore, we must have $\E_{\bx\sim D^-}p(\bx)\geq \gamma^2/2$.
We show that in order to satisfy $\E_{\bx\sim D^-}p(\bx)\geq \gamma^2/2$, we must have 
$\pr_{\bx\sim D^-}[\bx\not\in B]= \Omega(\gamma^2)$.
Notice that for any $\bx\in \mathbf{supp}(D^-)$ and $\bx\in B$, 
we must have either $\bu^*\cdot \bx+t_1\leq -\gamma$ or $\bv^*\cdot \bx+t_2\leq -\gamma$.
Suppose that $\bu^*\cdot \bx+t_1\leq -\gamma$, then we have $\bu^* \cdot \bx\leq -\gamma-t_1\leq -\gamma$ ($t_1\geq 0$ from \Cref{as parameter}).
Combining the above with $\bu^*\cdot \bx=\proj_{\bv} \bu^*\cdot \bx+\proj_{\bv^\perp}\bu^*\cdot \bx$ and $|\proj_{\bv} \bu^*\cdot \bx|\leq |\bv\cdot \bx|\leq \gamma/2$, we get
$\proj_{\bv^\perp}\bu^*\cdot \bx\leq 0$.
Notice that $\proj_{\bv^\perp}\bu^*\cdot \bx=(\bu^*\cdot \bv^\perp)(\bv^\perp\cdot \bx)\leq 0$
and we assumed $\sign(\bv^\perp\cdot \bu^*)\neq \sign(\bv^\perp\cdot \bv^*)$,
then $\proj_{\bv^\perp}\bv^*\cdot \bx=(\bv^*\cdot \bv^\perp)(\bv^\perp\cdot \bx)\geq 0$.
Plug it into the equation below, we get
\[
\bv^*\cdot \bx+t_2=\proj_{\bv} \bv^*\cdot \bx+\proj_{\bv^\perp} \bv^*\cdot \bx+t_2\geq (\bv \cdot \bv^*)(\bv\cdot \bx)+(\bv^*\cdot \bv^\perp)(\bv^\perp\cdot \bx)+t_2\geq -\gamma/2+\gamma\ge\gamma/2\; ,
\]
where the second from the last inequality comes from $\bx\in B$ and $t_2\geq \gamma$ in \Cref{as parameter}.
Similarly, we can also show that for any $\bx\in \mathbf{supp}(D^-)$ and $\bx\in B$, if $\bu^*\cdot \bx+t_1\leq -\gamma$, then
$\bv^*\cdot \bx+t_2\geq \gamma/2$.
Combining the two cases, we get that for any $\bx\in \mathbf{supp}(D^-)$ and $\bx\in B$,
\[
p(\bx)=(\bu^*\cdot \bx+t_1)(\bv^*\cdot \bx+t_2)\leq -\gamma^2/2\; .
\]
Therefore, we get
\begin{align*}
\E_{\bx\sim D^-}[p(\bx)]
=&\E_{\bx\sim D^-}[p(\bx)\mid \bx\in R]\pr_{\bx\sim D^-}[\bx\in R]+\E_{\bx\sim D^-}[p(\bx)\mid \bx\not \in R]\pr_{\bx\sim D^-}[\bx\not\in R]\\
\leq & -\gamma^2/2\pr_{\bx\sim D^-}[\bx\in R]+\pr_{\bx\sim D^-}[\bx\not\in R]\; .
\end{align*}
Combining the above with that $\E_{\bx\sim D^-}p(\bx)\geq \gamma^2/2$,
we get $\pr_{\bx\sim D^-}[\bx\not\in R]\geq \gamma^2/3$.
This completes the proof.

\end{proof}

\subsection{Proof of \Cref{lm positive polynomial} and \Cref{lm negative polynomial}} \label{app exact match}

In this section, we present the proof of \Cref{lm positive polynomial} and \Cref{lm negative polynomial}. For convenience, we restate the lemmas as follows.

\begin{lemma}[restatement of \Cref{lm positive polynomial}]\label{lm positive polynomial re}
    Let $h^*=\sign(\bu^*\cdot \bx + t_1) \wedge \sign(\bv^* \cdot \bx + t_2)$ be the target hypothesis of an instance of the problem of learning intersections of two halfspaces and $D$ be a distribution that is consistent with $h^*$
    Under \Cref{as parameter}, the following polynomial 
\begin{align*}
    f^*(\bx) = \frac{1}{t\sin\theta}(\bu^*\cdot \bx-t\sin\theta)^2(\bu^*\cdot \bx+t\sin \theta)
\end{align*}
satisfies $f^*(\bx) \ge 0, \forall \bx \in \textbf{supp}(D^+)$.
\end{lemma}

For a clear intuition, we plot the contour of the polynomial constructed in \Cref{lm positive polynomial} in \Cref{fig: pos}.

\begin{proof}[Proof of \Cref{lm positive polynomial re}]
    For every positive example $\bx$, we have $\bu^*\cdot x+t_1 = \bu^*\cdot x+t \sin\theta \ge 0$. Thus, $\forall x\in \textbf{supp}(D^+),$
    \begin{align*}
    f^*(\bx) = \frac{1}{t\sin\theta}(\bu^*\cdot \bx-t\sin\theta)^2(\bu^*\cdot \bx+t\sin \theta) \ge 0.
\end{align*}

\end{proof}

\begin{figure}
    \centering
    \includegraphics[width=0.45\linewidth]{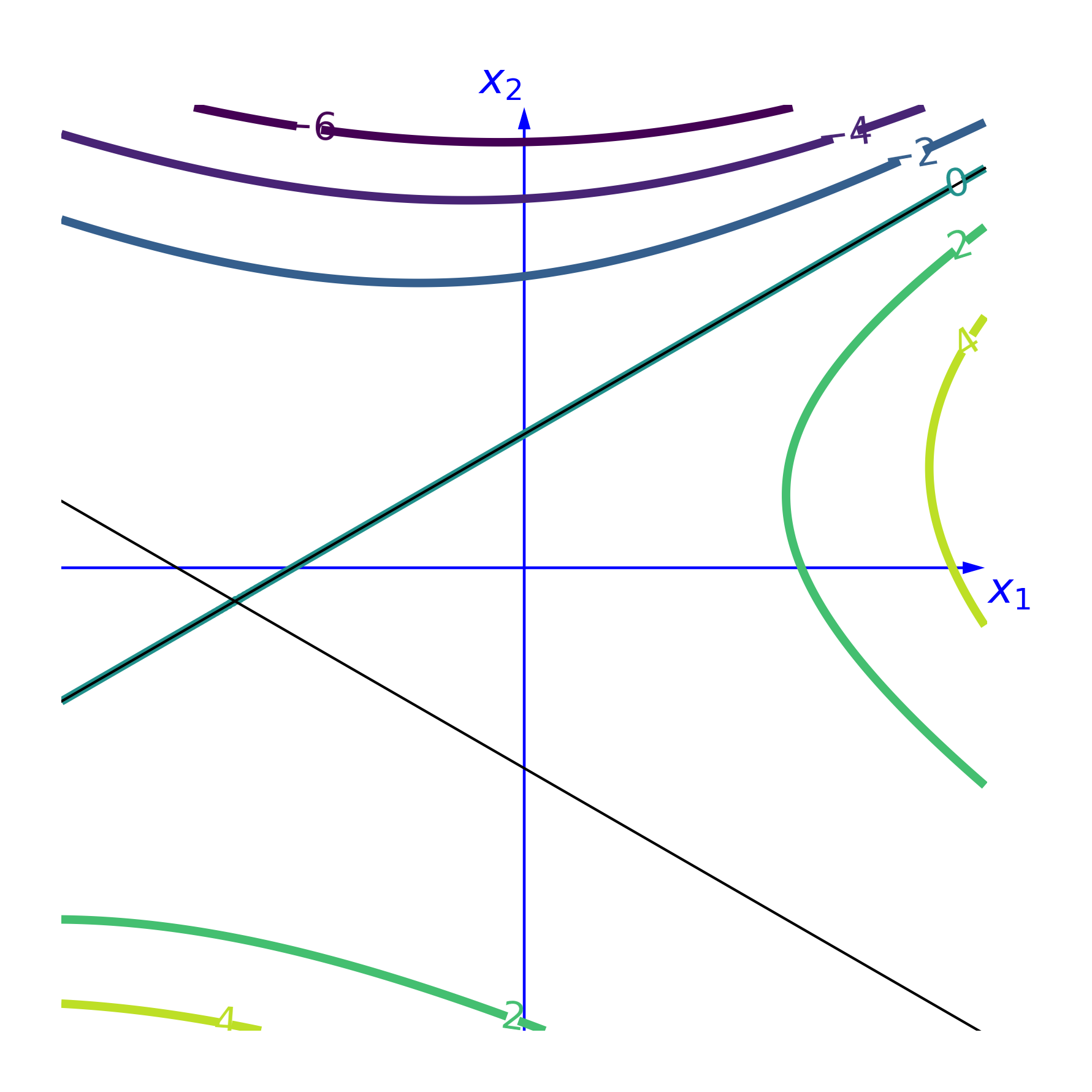}
    \caption{Illustration for \Cref{lm positive polynomial}. The target intersection of two halfspace $h^*$ is plotted in black. Colored lines represent the contours of the polynomial $f^*$. $f^*(\bx)>0$ for every example $\bx$ labeled positive by $h^*$.}
    \label{fig: pos}
\end{figure}

\begin{lemma}[restatement of \Cref{lm negative polynomial}]\label{lm negative polynomial re}
    Let $h^*=\sign(\bu^*\cdot \bx + t_1) \wedge \sign(\bv^* \cdot \bx + t_2)$ be the target hypothesis of an instance of the problem of learning intersections of two halfspaces and $D$ be a distribution that is consistent with $h^*$
    Under \Cref{as parameter}, the following polynomial 
\begin{align*}
    f^*(\bx) & = a_0 + a_1 \bx_1 + a_2 \bx_2 - \bx_2^2 \\   
    a_0 & = (1+\sigma)\tan^2 \theta t^2 \\
    a_1 & = (2+\sigma)\tan^2\theta t \\
    a_2 & =-\sigma \tan \theta t
\end{align*}
satisfies $f^*(\bx) \le 0, \forall \bx \in \textbf{supp}(D^-)$.
\end{lemma}

For a clear intuition, we plot the contour of the polynomial constructed in \Cref{lm negative polynomial} in \Cref{fig: pos}.

\begin{proof}[Proof of \Cref{lm negative polynomial re}]

To show for every $\bx\in \R^d$ such that $h^*(\bx)=-1$, $f^*(\bx)\le 0$., we partition the region of negative examples $N:=\{\bx \mid h^*(\bx)=-1\}$ into regions $N_1:=\{\bx \in V \mid \bx_2 \ge -\sigma t \tan\theta/2, \bu^*\cdot \bx+ t_1 \le 0\}$ and $N_2:=\{\bx \in V \mid \bx_2 \le -\sigma t \tan\theta/2, \bv^*\cdot \bx+ t_2 \le 0\}$, and show that in each region $f^*(\bx)\le 0$.

We first consider the region $N_1:=\{\bx \in V \mid \bx_2 \ge -\sigma t \tan\theta/2, \bu^*\cdot \bx+ t_1 \le 0\}.$ To start with, we focus on examples that are on the boundary of $N_1$. Let $\bx$ be any example on the decision boundary $\{\bx \mid \bu^*\cdot \bx + t_1 =0\}$. By \Cref{as parameter}, we know that $\bx$ satisfies $\bx_2 = \tan\theta(\bx_1+t)$. So,
\begin{align}\label{eq boundary 1}
    f^*(\bx) & = -\tan^2\theta \bx^2_1+(a_1+a_2\tan\theta-2t\tan^2\theta) \bx_1 + (a_0+a_2\tan\theta t-\tan^2\theta t^2) \notag\\
    & = -\tan^2\theta \bx_1^2 \le 0.
\end{align}
Thus, $f^*(\bx) \le 0$ for every $\bx$ that satisfies $\bu^*\cdot \bx+t_1=0$. Based on \eqref{eq boundary 1}, we show that $f^*(\bx) \le 0$ holds for every example $\bx$ with $\bx_2 = -\sigma t \tan\theta/2, \bx_1 \le -(1+\sigma/2)\tan\theta t$. 

For every fixed $\bx_2$, the partial derivative of $f^*(\bx)$ with respect to $\bx_1$ is 
\begin{align} \label{eq derivative 1}
    \frac{\partial f^*(\bx)}{\partial \bx_1} = a_1 = (2+\sigma)\tan^2 \theta t>0.
\end{align}
By \eqref{eq boundary 1}, we know that the point $\bx':=(-(1+\sigma/2)\tan\theta t,-\sigma t \tan\theta/2)$, the only vertex of the region $N_1$ satisfies $f^*(\bx')\le 0$. This implies that $f^*(\bx) \le 0$ holds for every example $\bx$ with $\bx_2 = -\sigma t \tan\theta/2, \bx_1 \le -(1+\sigma/2)\tan\theta t$.

So far, we have shown that $f^*(\bx) \le 0$ for every example $\bx$ on the boundary of $N_1$.
We next show that $f^*(\bx)\le 0$ holds for every example $\bx$ in the interior of $N_1$. Fix any $\bx_1$, the partial derivative of $f^*(\bx)$ with respect to $\bx_2$ is 
\begin{align} \label{eq derivative 2}
    \frac{\partial f^*(\bx)}{\partial \bx_2} = -2 \bx_2 + a_2 = -2 \bx_2 -\sigma\tan\theta t\le \sigma\tan\theta t-\sigma\tan\theta t=0,
\end{align}
when $\bx_2 \ge -\sigma t \tan\theta/2$. Since $f^*(\bx) \le 0$ holds for every $\bx$ on the boundary of $N_1$, \eqref{eq derivative 2} implies that $f^*(\bx) \le 0$ for every $\bx$ in the interior of $N_1$.

In the rest of the proof, we show that $f^*(\bx) \le 0$ for every $\bx \in N_2= \{\bx \in V \mid \bx_2 \le -\sigma t \tan\theta/2, \bv^*\cdot \bx+ t_2 \le 0\}.$ Let $\bx$ be any example on the decision boundary $\{\bx \mid \bv^*\cdot \bx+t_2=0\}$. By \Cref{as parameter} of $\bu^*$ and $t_1$, we know that $\bx$ satisfies $\bx_2 = -\tan\theta(\bx_1+(1+\sigma)t)$.

\begin{align*}
    f^*(\bx) & = -\tan^2\theta \bx^2_1 + (a_1-a_2\tan\theta-2(1+\sigma)t\tan^2\theta) \bx_1\\ 
    &+ (a_0-(1+\sigma)\tan\theta ta_2-(1+\sigma)^2\tan^2\theta t^2)\notag\\
    & = -\tan^2\theta \bx_1^2 \le 0.
\end{align*}

Thus, $f^*(\bx) \le 0$ for every $\bx$ that satisfies $\bv^*\cdot \bx+t_2=0$. Recall that $f^*(\bx) \le 0$ holds for every example $\bx$ with $\bx_2 = -\sigma t \tan\theta/2, \bx_1 \le -(1+\sigma/2)\tan\theta t$. Thus, $f^*(\bx) \le 0$ for every example $\bx$ on the boundary of $N_2$.
We next show that $f^*(\bx)\le 0$ holds for every example $\bx$ in the interior of $N_2$. Fix any $\bx_1$, the partial derivative of $f^*(\bx)$ with respect to $\bx_2$ is 
\begin{align} \label{eq derivative 3}
    \frac{\partial f^*(\bx)}{\partial \bx_2} = -2 \bx_2 + a_2 = -2 \bx_2 -\sigma\tan\theta t\ge \sigma\tan\theta t-\sigma\tan\theta t=0,
\end{align}
when $\bx_2 \le -\sigma t \tan\theta/2$. Since $f^*(\bx) \le 0$ holds for every $\bx$ on the boundary of $N_2$, \eqref{eq derivative 3} implies that $f^*(\bx) \le 0$ for every $\bx$ in the interior of $N_2$.
\end{proof}
\begin{figure}
    \centering
    \includegraphics[width=0.45\linewidth]{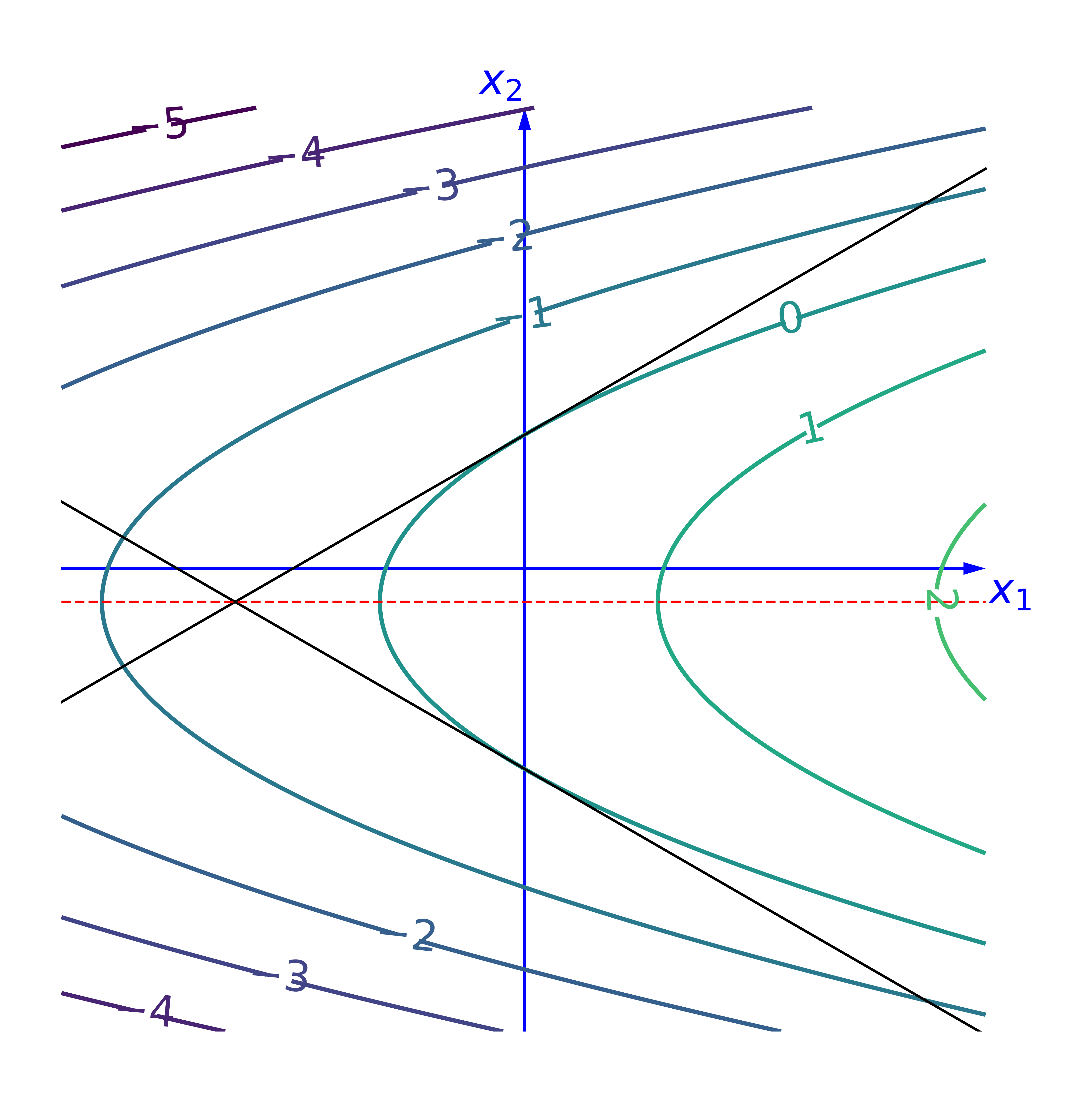}
    \caption{Illustration for \Cref{lm negative polynomial}. The target intersection of two halfspace $h^*$ is plotted in black. $h^*$ is symmetric according to the red dashed line $\bx_2= -\sigma t \tan\theta/2$. The red dashed line partitions the region of negative examples into two regions  $N_1:=\{\bx \in V \mid \bx_2 \ge -\sigma t \tan\theta/2, \bu^*\cdot \bx+ t_1 \le 0\}$ and $N_2:=\{\bx \in V \mid \bx_2 \le -\sigma t \tan\theta/2, \bv^*\cdot \bx+ t_2 \le 0\}$. 
    Colored lines represent the contours of the polynomial $f^*$. $f^*(\bx)<0$ for every example $\bx$ labeled negative by $h^*$.}
    \label{fig: neg}
\end{figure}

\subsection{Proof of \Cref{fact pos} and \Cref{fact neg}}\label{app fact pos neg}

\begin{fact}[restatement of \Cref{fact pos}]
\label{fact pos re}
    Let $h^*=\sign(\bu^*\cdot \bx + t_1) \wedge \sign(\bv^* \cdot \bx + t_2)$ be the target hypothesis of an instance of the problem of learning intersections of two halfspaces and $D$ be a distribution that is consistent with $h^*$
    Under \Cref{as parameter}, if $\E_{\bx\sim D^+}  (\bx) = 0, \E_{\bx\sim D^+} \bx^{\otimes 3} =0$ and for every $\bv \in S^{d-1}\cap V, \E_{\bx\sim D^+}(\bv\cdot \bx)^2 = \alpha^2$, then $\alpha^2 \le t^2 \sin^2\theta$.
\end{fact}

\begin{proof}[Proof of \Cref{fact pos}]

For every $\bx$ that is labeled positive by $h^*$, denote by $p(\bx)$ the variable of the density of a distribution $D^+$ over $\R^d$. Notice that any distribution $D^+$ that satisfies the statement of \Cref{fact pos}, gives a feasible solution to the following LP~\eqref{LP max re}. Thus, to upper bound the variance of $D^+$, it is equivalent to upper bound the optimal value of LP~\eqref{LP max re}.
\begin{align}\label{LP max re}
\begin{split}
\max &\  \alpha^2  \\
\text{s.t. } & \sum_{\bx}p(\bx)\bx = 0\\
& \sum_\bx p(\bx)\bx\bx^\intercal = \alpha^2 I \\
& \sum_\bx p(\bx)\bx^{\otimes 3} = 0 \\
& \sum_\bx p(\bx) =1 \\
& p(\bx) \ge 0 \quad \forall \bx
\end{split}
\end{align}

To do this, we use an LP duality argument to derive a tight upper bound for the optimal value of LP~\eqref{LP max re}.

Write $\bx\in V$ as $\bx = \bx_1 e_1 + \bx_2 e_2$ and $\bx_0=1\in \R$ for the simplicity of notation.
Let $f(\bx)=\sum_{i,j,k=0}^2 a_{ijk}\bx_i\bx_j\bx_k$ be a degree-3 polynomial defined over $V=\textbf{span}\{e_1,e_2\}$. The coefficient of $f(\bx)$ for the monomial $\bx_i\bx_j\bx_k$ is denoted by $a_{ijk}$. The dual linear program to LP~\eqref{LP max} is defined by LP~\eqref{LPdual max}, whose variable is defined over the coefficients of $f(\bx)$.
\begin{align}\label{LPdual max re}
\begin{split}
\min &\  a_{0}  \\
\text{s.t. } & f(\bx) \ge 0, \quad \forall \bx \in \textbf{supp}(D^+)\\
& a_{11}+a_{22} =-1
\end{split}
\end{align}
Here, $a_{11},a_{22}$ are coefficients of polynomial $f$ with respect to monomials $\bx_1^2,\bx_2^2$.

Every feasible solution to \eqref{LPdual max re} defines a degree-3 polynomial $f(\bx)$ such that $f(\bx)\ge 0$ for every example $\bx$ that is labeled positive by $h^*$. In particular, by LP duality theory \citep{bertsimas1997introduction, Shapiro2001}, the constant term $a_{0}$ of any feasible polynomial $f(\bx)$ to LP~\eqref{LPdual max re} gives an upper bound for the optimal value $\alpha^2$ to LP~\eqref{LP max re}. We explicitly construct the following polynomial feasible to LP~\eqref{LPdual max re}, with a small constant term.
\begin{align*}
    f^*(\bx) = \frac{1}{t\sin\theta}(\bu^*\cdot \bx-t\sin\theta)^2(\bu^*\cdot \bx+t\sin \theta)
\end{align*}
Notice that the constant term $a_{0} = f^*(0) = t^2\sin^2\theta$, which means if $f^*(\bx)$ gives a feasible solution to LP~\eqref{LP max re}, then $\alpha^2 \le f^*(0) = t^2\sin^2\theta$. So, in the rest of the proof, we show that $f^*(\bx)$ gives a feasible solution to LP~\eqref{LPdual max re}. By \Cref{lm positive polynomial}, we know that for every $\bx$ such that $h^*(\bx) = +1$, $f^*(\bx) \ge 0$.

On the other hand, we show that the sum of coefficients of $f^*(\bx)$ for monomials $\bx_1^2,\bx_2^2$ is equal to $-1$. Notice that 
\begin{align*}
    f^*(\bx) = \frac{1}{t\sin\theta}\left( (\bu^*\cdot \bx)^3-t\sin\theta(\bu^*\cdot \bx)^2 - (t\sin\theta)^2(\bu^*\cdot \bx) + (t\sin\theta)^3  \right),
\end{align*}
where the sum of coefficients of $f^*(\bx)$ for monomials $\bx_1^2,\bx_2^2$ is $-(\bu^*_1)^2-(\bu^*_2)^2=-\norm{\bu^*}^2=-1.$ This proves $f^*(\bx)$ gives a feasible solution to \eqref{LPdual max re}.
    
\end{proof}

\begin{fact}[restatement of \Cref{fact neg}]\label{fact neg re}
        Let $h^*=\sign(\bu^*\cdot \bx + t_1) \wedge \sign(\bv^* \cdot \bx + t_2)$ be the target hypothesis of an instance of the problem of learning intersections of two halfspaces and $D$ be a distribution that is consistent with $h^*$ 
        Under \Cref{as parameter},
        if $\E_{\bx\sim D^-}  \bx = 0, \E_{\bx\sim D^-} \bx^{\otimes 3} =0$ and for every $\bv \in S^{d-1}\cap V, \E_{\bx\sim D^-}(\bv\cdot \bx)^2 = \beta^2$, then $\beta^2 \ge (1+\sigma)t^2\tan^2\theta$.
\end{fact}

\begin{proof}[Proof of \Cref{fact neg}]
   For every $\bx$ that is labeled negative by $h^*$, denote by $p(\bx)$ the variable of the density of a distribution $D^-$ over $\R^d$. Notice that any distribution $D^-$ that satisfies the statement of \Cref{fact neg}, gives a feasible solution to the following LP~\eqref{LP min}. Thus, to lower bound the variance of $D^-$, it is equivalent to upper bound the optimal value of LP~\eqref{LP min}.

\begin{align}
\begin{split}\label{LP min}
\min &\  \beta^2  \\
\text{s.t. } & \sum_{\bx}p(\bx)\bx = 0\\
& \sum_\bx p(\bx)\bx\bx^\intercal = \beta^2 I \\
& \sum_\bx p(\bx)\bx^{\otimes 3} = 0 \\
& \sum_\bx p(\bx) =1 \\
& p(\bx) \ge 0 \quad \forall \bx
\end{split}
\end{align}
To do this, we use an LP duality argument to derive a tight lower bound for the optimal value of LP~\eqref{LP min}.

Write $\bx\in V$ as $\bx = \bx_1 e_1 + \bx_2 e_2$ and $\bx_0=1\in \R$ for the simplicity of notation.
Let $f(\bx)=\sum_{i,j,k=0}^2 a_{ijk}\bx_i\bx_j\bx_k$ be a degree-3 polynomial defined over $V=\textbf{span}\{e_1,e_2\}$. The coefficient of $f(x)$ for the monomial $\bx_i\bx_j\bx_k$ is denoted by $a_{ijk}$. The dual linear program to LP~\eqref{LP min} is defined by LP~\eqref{LPdual min}, whose variable is defined over the coefficients of $f(x)$.

\begin{align}
\begin{split}\label{LPdual min}
\max &\  a_{0}  \\
\text{s.t. } & f(\bx) \le 0, \quad \forall \bx \in \textbf{supp}(D^-)\\
& a_{11}+a_{22} =-1
\end{split}
\end{align}
Here, $a_{11},a_{22}$ are coefficients of polynomial $f$ with respect to monomial $x_1^2,x_2^2$.
Every feasible solution to \eqref{LPdual min} defines a degree-3 polynomial $f(\bx)$ such that $f(\bx)\le 0$ for every example $\bx$ that is labeled negative by $h^*$. In particular, by LP duality theory \citep{bertsimas1997introduction}, the constant term $a_{0}$ of any feasible polynomial $f(\bx)$ to LP~\eqref{LPdual min} gives a lower bound for the optimal value $\beta^2$ to LP~\eqref{LP min}. We explicitly construct the following polynomial feasible to LP~\eqref{LPdual min}, with a large constant term.
\begin{align*}
    f^*(\bx) & = a_0 + a_1 \bx_1 + a_2 \bx_2 - \bx_2^2 \\   
    a_0 & = (1+\sigma)\tan^2 \theta t^2 \\
    a_1 & = (2+\sigma)\tan^2\theta t \\
    a_2 & =-\sigma \tan \theta t
\end{align*}
Notice that the sum of coefficients of $f^*(\bx)$ for monomials $\bx_1^2,\bx_2^2$ is equal to $-1$ and by \Cref{lm negative polynomial}, $f(\bx) \le 0, \quad \forall \bx \in \textbf{supp}(D^-)$. As the constant term $a_0$ of $f^*(\bx)$ is equal to $(1+\sigma)\tan^2 \theta t^2$, this concludes the proof of \Cref{fact neg}.

\end{proof}


\subsection{Proof of \Cref{lm 3rd moment}}\label{app lm 3rd}

In this section, we give the Proof of \Cref{lm 3rd moment}. For convenience, we restate \Cref{lm 3rd moment} as follows.

\begin{lemma}[restatement of \Cref{lm 3rd moment}]\label{lm 3rd moment re}
    Let $D$ be a distribution over $\B^2(1) \times \{\pm 1\}$ that is consistent with an instance of learning intersections of two halfspaces with $\gamma$-margin assumption. Let $c>0$ be any suitably large constant. Suppose
    \begin{enumerate}
        \item $\norm{\E_{\bx\sim D^+} \bx}_F\le \gamma^c, \norm{\E_{\bx\sim D^-} \bx}_F\le \gamma^c$ and $\norm{(\E_{\bx\sim D^+}-\E_{\bx\sim D^-}) \bx}_F$.
        \item $\E_{\bx\sim D^+}\bx\bx^\intercal = \alpha^2 I+\Delta_+, \E_{\bx\sim D^-}\bx\bx^\intercal = \alpha^2 I+\Delta_-$, where $\Delta_+,\Delta_- \in \R^{2\times 2}$ are symmetric matrices such that $\norm{\Delta_+}_F \le \gamma^c, \norm{\Delta_-}_F \le \gamma^c$ and $\alpha^2 > 0$.
        \item $\norm{(\E_{\bx\sim D^+}- \E_{\bx\sim D^-})\bx^{\otimes 3}}_F\le \gamma^c
        $
    \end{enumerate}
    then $\norm{\E_{\bx\sim D^+}\bx^{\otimes 3}}_F \ge \Omega(\gamma^2), \norm{\E_{\bx\sim D^-}\bx^{\otimes 3}}_F \ge \Omega(\gamma^2)$.
\end{lemma}

\begin{proof}[Proof of \Cref{lm 3rd moment re}]
We prove \Cref{lm 3rd moment} by contradiction. Assuming $\norm{\E_{\bx_V\sim D^+}\bx_V^{\otimes 3}}_F \le O(\gamma^2)$ or  $\norm{\E_{\bx\sim D^-}\bx_V^{\otimes 3}}_F \le O(\gamma^2)$ holds. We show there is no $\alpha^2$ that can be used to fulfill the second condition in the statement of \Cref{lm 3rd moment}.


For every $\bx$ that is labeled positive by $h^*$, denote by $p(\bx)$ the variable of the density of a distribution $D^+$ over $\R^d$. Under margin assumption, every positive example $\bx$ satisfies $\bu^*\cdot \bx+t_1 \ge \gamma$ and $\bv^*\cdot \bx+t_2 \ge \gamma$. 
Denote by $S^+_\gamma:=\{\bx \in \R^2 \mid \bu^*\cdot \bx+t_1 \ge \gamma$ and $\bv^*\cdot \bx+t_2 \ge \gamma\}$.
Let $b=\gamma^c \ge 0$ be a small positive number that represents the level of perturbation for the LP~\eqref{LP max re}. Notice that Frobenius norm is always an upper bound of infinity norm. Thus,
any distribution $D^+$ that satisfies the statement of \Cref{lm 3rd moment} under the $\gamma$-margin assumption, gives a feasible solution to the following LP~\eqref{LPpert max}. 
\begin{align}
\begin{split}\label{LPpert max}
\max &\  \alpha^2  \\
\text{s.t. } & -b\le \sum_{\bx}p(\bx)\bx_1 \le b\\
& -b\le \sum_{\bx}p(\bx)\bx_2 \le b\\
&-b\le\sum_\bx p(\bx)\bx_1^2 - \alpha^2 \le b\\
&-b\le\sum_\bx p(\bx)\bx_2^2 - \alpha^2 \le b\\
&-b\le\sum_\bx p(\bx)\bx_1\bx_2 \le b\\
& -b\le \sum_\bx p(\bx)\bx_1^3  \le b \\
& -b\le \sum_\bx p(\bx)\bx_1^2\bx_2  \le b \\
& -b\le \sum_\bx p(\bx)\bx_1\bx_2^2  \le b \\
& -b\le \sum_\bx p(\bx)\bx_2^3  \le b \\
& \sum_\bx p(\bx) =1 \\
& p(\bx) \ge 0 \quad \forall \bx \in S^+_\gamma
\end{split}
\end{align}
Write $\bx\in V$ as $\bx = \bx_1 e_1 + \bx_2 e_2$ and $\bx_0=1$ for the simplicity of notation.
Let $f(\bx)=\sum_{i,j,k=0}^2 a_{ijk}\bx_i\bx_j\bx_k$ be a degree-3 polynomial defined over $V=\textbf{span}\{e_1,e_2\}$. The coefficient of $f(\bx)$ for the monomial $\bx_i\bx_j\bx_k$ is denoted by $a_{ijk}$. The dual linear program to LP~\eqref{LPpert max} is defined by LP~\eqref{LPdualpert max}, whose variable is defined over the coefficients of $f(\bx)$.
\begin{align}
\begin{split}\label{LPdualpert max}
\min &\  a_{0} + b(a_1+a_2+a_{11}+a_{22}+a_{111}+a_{112}+a_{122}+a_{222})  \\
\text{s.t. } & f(\bx) \le 0, \quad \forall \bx \in S^+_\gamma \\
& a_{11}+a_{22} \le -1
\end{split}
\end{align}
We construct an upper bound for the optimal value $\alpha^2$ by constructing a feasible solution to \eqref{LPdualpert max} with a small objective value. Consider the following polynomial 
\begin{align*}
    f^*(\bx) = \frac{1}{(t\sin\theta-\gamma/2)}(\bu^*\cdot \bx-(t\sin\theta-\gamma/2))^2(\bu^*\cdot \bx+(t\sin \theta-\gamma/2))
\end{align*}
Under margin assumption, every positive example $\bx$ satisfies $\bu^*\cdot \bx+t_1 \ge \gamma$ and $\bv^*\cdot \bx+t_2 \ge \gamma$. That is to say, $D^+$ is also consistent with an intersection of halfspaces $h'(\bx) = \sign(\bu^*\cdot \bx+t_1-\gamma/2) \wedge \sign(\bv^*\cdot \bx+t_2-\gamma/2)$. Thus, $f^*(\bx) \ge 0$ for every positive example by \Cref{lm positive polynomial}.

On the other hand,
\begin{align*}
    f^*(\bx) = \frac{1}{(t\sin\theta-\gamma/2)}\left( (\bu^*\cdot \bx)^3-(t\sin\theta-\gamma/2)(\bu^*\cdot \bx)^2 - (t\sin\theta-\gamma/2)^2(\bu^*\cdot \bx) + (t\sin\theta-\gamma/2)^3  \right),
\end{align*}
where the sum of coefficients of $f^*(\bx)$ for monomials $\bx_1^2,\bx_2^2$ is $-(\bu^*_1)^2-(\bu^*_2)^2=-\norm{\bu^*}^2=-1.$ This proves $f^*(\bx)$ gives a feasible solution to \eqref{LPdualpert max}. Notice that $t\sin\theta$, the distance between the origin and halfspace $h^*_1$ is at least $\gamma$, otherwise, the origin is not labeled positive by $h^*$. On the other hand, $t\sin\theta \le 1$, because otherwise, no example in $\B(1)$ is labeled negative and $D^-$ is not well-defined under the $\gamma$-margin assumption.
This implies the objective value corresponding to the solution to LP~\eqref{LPpert max} is 
\begin{align*}
    obj(f^*)& :=(t\sin\theta-\gamma/2)^2+\frac{b}{t\sin\theta-\gamma/2}((\bu^*_1)^3+(\bu^*_1)^2\bu^*_2+\bu^*_1(\bu^*_2)^2+(\bu^*_2)^3)-b(t\sin\theta-\gamma/2)(\bu^*_1+\bu^*_2) \\
    & \le t^2\sin^2\theta-\gamma/4+O(b/\gamma) \\
    & \le t^2\sin^2\theta - \gamma/8,
\end{align*}
when $b\le O(\gamma^2)$. This implies that $\alpha^2 \le t^2\sin^2\theta - \Omega(\gamma)$

On the other hand, we derive a lower bound for $\alpha^2$.
For every $\bx$ that is labeled negative by $h^*$, denote by $p(\bx)$ the variable of the density of a distribution $D^-$ over $\R^d$. 
Denote by $S^-_\gamma:=\{\bx \in \R^2 \mid \bu^*\cdot \bx+t_1 \le -\gamma$ or $\bv^*\cdot \bx+t_2 \le -\gamma\}$. 
Under $\gamma$-margin assumption, any example $\bx$ with $h^*(\bx)=-1$ satisfies $\bx \in S^-_\gamma$.
Let $b=\gamma^c \ge 0$ be a small positive number that represents the level of perturbation for the LP~\eqref{LP min}.
Notice that Frobenius norm is always an upper bound of infinity norm. 
Thus,
any distribution $D^-$ that satisfies the statement of \Cref{lm 3rd moment} under the $\gamma$-margin assumption. Thus, to lower bound the variance of $D^-$, it is equivalent to lower bound the optimal value of LP~\eqref{LPpert min}.

\begin{align}
\begin{split}\label{LPpert min}
\min &\  \alpha^2  \\
\text{s.t. } & -b\le \sum_{\bx}p(\bx)\bx_1 \le b\\
& -b\le \sum_{\bx}p(\bx)\bx_2 \le b\\
&-b\le\sum_\bx p(\bx)\bx_1^2 - \alpha^2 \le b\\
&-b\le\sum_\bx p(\bx)\bx_2^2 - \alpha^2 \le b\\
&-b\le\sum_\bx p(\bx)\bx_1\bx_2 \le b\\
& -b\le \sum_\bx p(\bx)\bx_1^3  \le b \\
& -b\le \sum_\bx p(\bx)\bx_1^2\bx_2  \le b \\
& -b\le \sum_\bx p(\bx)\bx_1\bx_2^2  \le b \\
& -b\le \sum_\bx p(\bx)\bx_2^3  \le b \\
& \sum_\bx p(\bx) =1 \\
& p(\bx) \ge 0 \quad \forall \bx \in S^-_\gamma\end{split}
\end{align}
To do this, we use an LP duality argument to derive a tight lower bound for the optimal value of LP~\eqref{LP min}.

Write $\bx\in V$ as $\bx = \bx_1 e_1 + \bx_2 e_2$ and $\bx_0=1\in \R$ for the simplicity of notation.
Let $f(\bx)=\sum_{i,j,k=0}^2 a_{ijk}\bx_i\bx_j\bx_k$ be a degree-3 polynomial defined over $V=\textbf{span}\{e_1,e_2\}$. The coefficient of $f(\bx)$ for the monomial $\bx_i\bx_j\bx_k$ is denoted by $a_{ijk}$. The dual linear program to LP~\eqref{LP min} is defined by LP~\eqref{LPdual min}, whose variable is defined over the coefficients of $f(\bx)$.

\begin{align}
\begin{split}\label{LPdualpert min}
\max &\  a_{0} - b(a_1+a_2+a_{111}+a_{112}+a_{122}+a_{222})  \\
\text{s.t. } & f(\bx) \le 0, \quad \forall \bx \in S^-_\gamma\\
& a_{11}+a_{22} \ge -1
\end{split}
\end{align}

We construct a lower bound for the optimal value $\alpha^2$ by constructing a feasible solution to \eqref{LPpert min} with a large objective value. Consider the following polynomial 
\begin{align*}
    f^*(\bx) & = a_0 + a_1 \bx_1 + a_2 \bx_2 - \bx_2^2 \\   
    a_0 & = (1+\sigma)\tan^2 \theta t^2 \\
    a_1 & = (2+\sigma)\tan^2\theta t \\
    a_2 & =-\sigma \tan \theta t
\end{align*}
By \Cref{lm negative polynomial},
we know that $f^*(\bx)$ gives a feasible solution to \eqref{LPdualpert min}. In the rest of the proof, we show the feasible solution corresponds to $f^*(\bx)$ has a large objective value. When $b=\gamma^c<O(\gamma^2)$, the objective value is 
\begin{align*}
    obj(f^*) & := (1+\sigma)\tan^2 \theta t^2 - b((2+\sigma)\tan^2\theta t-\sigma \tan \theta t) \\
    & \ge (1+\sigma)\tan^2 \theta t^2 - b(2+\sigma)\tan^2\theta t 
     = \tan^2\theta t^2 (1+\sigma-\frac{b(2+\sigma)}{t}) \\
    & \ge \tan^2\theta t^2 (1+\sigma-\frac{b(2+\sigma)}{t\sin\theta}) 
     \ge \tan^2\theta t^2 (1+\sigma-\frac{b(2+\sigma)}{\gamma}) \\
     & = \tan^2\theta t^2(1-O(\gamma)) + \sigma\tan^2\theta t^2(1-O(\gamma)) \ge \tan^2\theta t^2(1-O(\gamma)).
\end{align*}
Here, we use the fact that $t\sin \theta>\gamma$.
We consider two cases. In the first case, $\cos^2\theta\le 1-O(\gamma)$. In this case, we have $obj(f^*) \ge \sin^2\theta t^2$. In the second case, $\cos^2\theta \ge 1-O(\gamma)$. In this case, we have 
\begin{align*}
    obj(f^*) \ge \tan^2\theta t^2 -O(\gamma)\sin^2\theta t^2 \ge \sin^2\theta t^2 - O(\gamma) \ge \sin^2\theta t^2 - \gamma/16.
\end{align*}
Thus, we conclude $\alpha^2 \ge \sin^2\theta t^2 - \gamma/16$. 

To conclude the proof of \Cref{lm 3rd moment}, we without loss of generality to assume $\norm{E_{\bx\sim D^-}\bx^{\otimes 3}}_F \le O(\gamma^c)$. 
Since $E_{\bx\sim D^+}\bx^{\otimes 3}$ is close to $E_{\bx\sim D^-}\bx^{\otimes 3}$, we know that $D^+$ gives a feasible solution to LP~\eqref{LPpert max} with $\alpha^2 \le t^2\sin^2\theta - \Omega(\gamma)$ and $D^-$ gives a feasible solution to LP~\eqref{LPpert max} with $\alpha^2 \ge \sin^2\theta t^2 - \gamma/16$, which gives a contradiction.

\end{proof}

\subsection{Proof of \Cref{th:intersection-structure-main}}\label{app structure main}

In this section, present the full proof of \Cref{th:intersection-structure-main}. For convenience, we restate \Cref{th:intersection-structure-main} as follows.

\begin{theorem}\label{th:intersection-structure-main re}(restatement of \Cref{th:intersection-structure-main})

Let $D$ be a distribution over $\B^2(1) \times \{\pm 1\}$ that is consistent with an instance of learning intersections of two halfspaces with $\gamma$-margin assumption. Let $c>0$ be \emph{any} suitably large constant. 
Suppose 
\begin{enumerate}
    \item $\norm{\E_{\bx\sim D^+} \bx}_F, \norm{\E_{\bx\sim D^-} \bx}_F\le \gamma^c$. \label{cond:intersection-structure-main-1}
    \item $\norm{(\E_{\bx\sim D^+} - \E_{\bx\sim D^+})\bx\bx^\intercal}_F \le \gamma^c$
    \label{cond:intersection-structure-main-2}
    \item $\norm{(\E_{\bx\sim D^+} - \E_{\bx\sim D^-})\bx^{\otimes 3}}_F\le \gamma^c$,
    \label{cond:intersection-structure-main-3}
\end{enumerate}
then $\norm{\E_{\bx\sim D^+}\bx^{\otimes 3}}_F, \norm{\E_{\bx\sim D^-}\bx^{\otimes 3}}_F = \Omega(\gamma^{15})$.
\end{theorem}

\begin{proof}[Proof of \Cref{th:intersection-structure-main re}]
For every example $\bx \in V$, we consider the following linear transformation of $\bx$.
\begin{align*}
    \Tilde{\bx}: = \norm{(\Sigma^+)^{-1/2}}_2^{-1} (\Sigma^+)^{-1/2} \bx,
\end{align*}

Notice that 
\begin{align*}
 \norm{\Tilde{\bx}}_2 =\norm{(\Sigma^+)^{-1/2}}_2^{-1} \norm{(\Sigma^+)^{-1/2} \bx}_2 \le \norm{(\Sigma^+)^{-1/2}}_2^{-1} \norm{(\Sigma^+)^{-1/2}}_2\norm{\bx}_2 \le 1
\end{align*}
 Denote by $y(\Tilde{\bx}):=h^*(\bx)$, we will show that $\Tilde{\bx}$ is labeled by another intersections of two halfspaces $\Tilde{h}(\Tilde{\bx}) = \Tilde{h}_1(\Tilde{\bx}) \wedge \Tilde{h}_2(\Tilde{\bx})$ with $\poly(\gamma)$-margin assumption. Consider the first ground truth halfspace $h^*_1:=\sign(\bu^*\cdot \bx+t_1)$. We have 
\begin{align*}
    h^*_1(\bx) & = \sign\left(\bu^*\cdot \bx+t_1\right) = \sign\left((\Sigma^+)^{1/2}\bu^*\cdot (\Sigma^+)^{-1/2} \bx+t_1\right)\\
    &= \sign\left(\norm{(\Sigma^+)^{-1/2}}_2(\Sigma^+)^{1/2} \bu^*\cdot  \norm{(\Sigma^+)^{-1/2}}_2^{-1} (\Sigma^+)^{-1/2} \bx+t_1\right) \\
    & = \sign\left(\bu'\cdot  \Tilde{\bx}+t_1\right) = \sign\left(\Tilde{\bu}\cdot  \Tilde{\bx}+t_1/\norm{\bu'}_2\right) = \Tilde{h}_1(\Tilde{\bx}).
\end{align*}
Here $\bu':=\norm{(\Sigma^+)^{-1/2}}_2\Sigma^{1/2} \bu^*$ and $\Tilde{\bu} = \bu'/\norm{\bu'}_2$. Since $\norm{(\Sigma^+)^{1/2}}_2 \le 1$ and by \Cref{lm 2nd}, $\norm{(\Sigma^+)^{-1/2}}_2 \le \gamma^{-2}$, we know that $\norm{\bu'}_2\le \gamma^{-2}$.
Since $\bx$ satisfies $\gamma$-margin assumption with respect to $h^*$ and $\norm{\bu'}_2 \le \gamma^{-2}$, we know that
\begin{align*}
   \abs{\Tilde{\bu}\cdot  \Tilde{\bx}+t_1/\norm{\bu'}_2} = \abs{\bu'\cdot  \Tilde{\bx}+t_1} / \norm{\bu'}_2 =  \abs{\bu^*\cdot \bx + t_1}/ \norm{\bu'}_2 \ge \gamma/ \norm{\bu'}_2 = \gamma^{3}.
\end{align*}
Similarly, for the second ground truth halfspace $h^*_2:=\sign(\bv^*\cdot \bx+t_2)$, we have 
\begin{align*}
    h^*_2(\bx) & = \sign\left(\bv^*\cdot \bx+t_2\right) = \sign\left((\Sigma^+)^{1/2}\bv^*\cdot (\Sigma^+)^{-1/2} \bx+t_2\right)\\
    &= \sign\left(\norm{(\Sigma^+)^{-1/2}}_2(\Sigma^+)^{1/2} \bv^*\cdot  \norm{(\Sigma^+)^{-1/2}}_2^{-1} (\Sigma^+)^{-1/2} \bx+t_2\right) \\
    & = \sign\left(\bv'\cdot  \Tilde{\bx}+t_2\right) = \sign\left(\Tilde{\bv}\cdot  \Tilde{\bx}+t_2/\norm{\bv'}_2\right) = \Tilde{h}_2(\Tilde{\bx}).
\end{align*}
Here $\bv':=\norm{(\Sigma^+)^{-1/2}}_2\Sigma^{1/2} \bv^*$ and $\Tilde{\bv} = \bv'/\norm{\bv'}_2$. Since $\norm{(\Sigma^+)^{1/2}}_2 \le 1$ and by \Cref{lm 2nd}, $\norm{(\Sigma^+)^{-1/2}}_2 \le \gamma^{-2}$, we know that $\norm{\bv'}_2\le \gamma^{-2}$.
Since $\bx$ satisfies $\gamma$-margin assumption with respect to $h^*$ and $\norm{\bv'}_2 \le \gamma^{-2}$, we know that
\begin{align*}
   \abs{\Tilde{\bv}\cdot  \Tilde{\bx}+t_1/\norm{\bv'}_2} = \abs{\bv'\cdot  \Tilde{\bx}+t_2} / \norm{\bv'}_2 =  \abs{\bv^*\cdot \bx + t_1}/ \norm{\bv'} \ge \gamma/ \norm{\bv'} = \gamma^{3}.
\end{align*}
This implies that $\Tilde{\bx}$ is labeled by an intersections of two halfspaces $\Tilde{h}(\Tilde{\bx}) = \Tilde{h}_1(\Tilde{\bx}) \wedge \Tilde{h}_2(\Tilde{\bx})$ with $\gamma^{3}$-margin assumption. 

Next, we show that the marginal distribution of $\Tilde{\bx}$ satisfies the conditions in the statement of \Cref{lm 3rd moment}. Recall that the linear transformation $\Tilde{\bx}$ preserves the labels of $\bx$. Consider the distributions of $\Tilde{\bx}$ with positive labels. We have 
\begin{align*}
   \norm{\E_{\bx\sim D^+} \Tilde{\bx}}_2 & = \E_{\bx\sim D^+} \norm{\norm{(\Sigma^+)^{-1/2}}_2^{-1} (\Sigma^+)^{-1/2} \bx} 
   = \norm{(\Sigma^+)^{-1/2}}^{-1}_2 \norm{ \E_{\bx\sim D^+} (\Sigma^+)^{-1/2} \bx}_2\\
   &\le \norm{(\Sigma^+)^{-1/2}}^{-1}_2 \norm{ (\Sigma^+)^{-1/2}}_2 \norm{ \E_{\bx\sim D^+} \bx}_2 =\norm{\E_{\bx\sim D^+}  \bx}_2 \le \gamma^c .
 \end{align*}
 Similarly, consider the distributions of $\Tilde{\bx}$ with positive labels. We have 
\begin{align*}
   \norm{\E_{\bx\sim D^-} \Tilde{\bx}}_2 & = \E_{\bx\sim D^-} \norm{\norm{(\Sigma^+)^{-1/2}}^{-1}_2 (\Sigma^+)^{-1/2} \bx}_2 
   = \norm{(\Sigma^+)^{-1/2}}^{-1}_2 \norm{ \E_{\bx\sim D^-} (\Sigma^+)^{-1/2} \bx}_2\\
   &\le \norm{(\Sigma^-)^{-1/2}}^{-1}_2 \norm{ (\Sigma^+)^{-1/2}}_2 \norm{ \E_{\bx\sim D^-} \bx}_2 =\norm{\E_{\bx\sim D^-}  \bx}_2 \le \gamma^c.
 \end{align*}
Thus, $\Tilde{\bx}$ satisfies the first condition of \Cref{lm 3rd moment}. We next show that $\Tilde{\bx}$ satisfies the second condition of \Cref{lm 3rd moment}. The covariance matrix of the positive $\Tilde{\bx}$ is 
\begin{align*}
    \E_{\bx\sim D^+}\Tilde{\bx}\Tilde{\bx}^\intercal =  \norm{(\Sigma^+)^{-1}}_2^{-1}(\Sigma^+)^{-1/2} \E_{\bx\sim D^+} \bx\bx^\intercal (\Sigma^+)^{-1/2} = \norm{(\Sigma^+)_2^{-1}}^{-1} I.
\end{align*}
On the other hand, the covariance of the negative $\Tilde{\bx}$ is 
\begin{align*}
    \E_{\bx\sim D^-}\Tilde{\bx}\Tilde{\bx}^\intercal & = \norm{(\Sigma^+)^{-1}}_2^{-1}(\Sigma^+)^{-1/2} \E_{\bx\sim D^-} \bx\bx^\intercal (\Sigma^+)^{-1/2} 
     = \norm{(\Sigma^+)^{-1}}_2^{-1}(\Sigma^+)^{-1/2}  \Sigma^- (\Sigma^+)^{-1/2} \\
     & = \norm{(\Sigma^+)^{-1}}_2^{-1}(\Sigma^+)^{-1/2}  \Sigma^+ (\Sigma^+)^{-1/2} - \norm{(\Sigma^+)^{-1}}_2^{-1}(\Sigma^+)^{-1/2}  (\Sigma^+-\Sigma^-) (\Sigma^+)^{-1/2} \\
     & = \norm{(\Sigma^+)^{-1}}_2^{-1} I - \norm{(\Sigma^+)^{-1}}_2^{-1}(\Sigma^+)^{-1/2}  (\Sigma^+-\Sigma^-) (\Sigma^+)^{-1/2}.
\end{align*}
Since $\norm{(\Sigma^+-\Sigma^-)}_F\le \gamma^c$, we know that 
\begin{align*}
    \norm{(\E_{\bx\sim D^+}-\E_{\bx\sim D^-})\Tilde{\bx}\Tilde{\bx}^\intercal}_F = \norm{\norm{(\Sigma^+)^{-1}}_2^{-1}(\Sigma^+)^{-1/2}  (\Sigma^+-\Sigma^-) (\Sigma^+)^{-1/2}}_F \le \norm{(\Sigma^+-\Sigma^-)}_F \le \gamma^c.
\end{align*}
Thus, the marginal distribution of $\Tilde{\bx}$ satisfies the conditions in the statement of \Cref{lm 3rd moment}.

Finally, we show the third condition of \Cref{lm 3rd moment} holds. We have 
\begin{align*}
    \norm{(\E_{\bx\sim D^+}-\E_{\bx\sim D^-})\Tilde{\bx}^{\otimes 3}}_F & = \norm{\norm{(\Sigma^+)^{-1/2}}_2^{-3} ((\Sigma^+)^{-1/2})^{\otimes 3}\left( \E_{\bx\sim D^+}\bx^{\otimes 3}-\E_{\bx\sim D^-}\bx^{\otimes 3} \right) }_F \\
    &\le O( \norm{\left( \E_{\bx\sim D^+}\bx^{\otimes 3}-\E_{\bx\sim D^-}\bx^{\otimes 3} \right)}_F) \le \gamma^c.
\end{align*}
Thus, the third condition in the statement of \Cref{lm 3rd moment} holds for $\Tilde{\bx}$. By \Cref{lm 3rd moment}, $\norm{\E_{\bx\sim D^+}\Tilde{\bx}^{\otimes 3}}_F \ge \Omega(\gamma^{6}), \norm{\E_{\bx\sim D^-}\Tilde{\bx}^{\otimes 3}}_F \ge \Omega(\gamma^{6})$. By \Cref{lm 2nd}, we know that $\norm{(\Sigma^+)^{-1/2}}_F \le O(\gamma^2)$. Thus, $\norm{\E_{\bx_V\sim D^+}\bx_V^{\otimes 3}}_F \ge \Omega(\gamma^{15}), \norm{\E_{\bx\sim D^-}\bx_V^{\otimes 3}}_F \ge \Omega(\gamma^{15})$.

\end{proof}

\section{Omitted Proofs from \Cref{sec direction}}\label{app direction}

\subsection{Proof of \Cref{th SQ efficient}}\label{app SQ efficient}

\begin{algorithm}[h]
		\caption{\textsc{SQ-Direction Finding} (SQ-efficient algorithm for finding relevant direction with matched moments)}\label{alg SQ}
		\begin{algorithmic} [1]
\State\textbf{Input:} $\gamma\in (0,1)$ and i.i.d. sample access to a distribution $D$ on $\B^d(1)\times \{\pm 1\}$ that is an instance of learning intersections of halfspaces under product distribution. Suppose that $D$ satisfies the conditions in the statement of \Cref{th SQ efficient}.
\State\textbf{Output:} $\mathcal{O}$, a list of $\poly(d)$, $\bw \in S^{d-1}$ such that at least one of $\bw \in \mathcal{O}$ satisfies $\norm{\proj_{V^\perp} \bw} \le \poly(\gamma)$.
\State Take $S_1$, a set of $m_1=\poly(d/\gamma)$ \iid samples from $D_X$ to estimate $\mu:=\E_{\bx\sim D_X} \bx $ with
\begin{align*}
     \hat{\mu}:=\frac{1}{m_1}\sum_{\bx\in S_1} \bx
\end{align*}
up to $\poly(\gamma/d)$ error 
\State Take $S$, a set of $N=\poly(d/\gamma)$ \iid samples from $D_X$ and estimate 
\begin{align*}
    \hat T=\frac{1}{N}\sum_{\bx\in S} (\bx-\hat{\mu})^{\otimes 3}
\end{align*}
\State Define $\hat{f}:S^{d-1}\to \R$ as $\hat{f}(\bu):= \hat T\cdot \bu^{\otimes 3} = \hat{\E}_{\bx\sim S}\left((\bx-\hat{\mu}) \cdot \bu\right)^3$.
\State $\mathcal{O}:=\emptyset$
\For{$t=1,\dots,T=\poly(d)$}
\State Find a $(\gamma/d)^{c'}$-approximate solution $\bu_t$ to $\hat{f}$ such that for every $\bu \in \mathcal{O}, \abs{\bu_t\cdot \bu} \le \poly(\gamma/d)$, where $c'>0$ is a large constant.
\State $\mathcal{O} \gets \mathcal{O} \cup \{\bu_t\}$. (If no such $\bu_t$ is found, return $\mathcal{O}$)
\EndFor
\State \Return $\mathcal{O}$
\end{algorithmic}
\end{algorithm}

In this section, we present the proof of \Cref{th SQ efficient}. The algorithm we will analyze is \Cref{alg SQ}. 
For convenience, we restate \Cref{th SQ efficient} below. 

\begin{theorem}[restatement of \Cref{th SQ efficient}]\label{th SQ efficient re}
There is a learning algorithm $\A$ such that for every $c$, a suitably large constant and any instance of learning intersections of two halfspaces under factorizable distribution with $\gamma$-margin assumption if the input distribution $D$ satisfies 
\begin{enumerate}
    \item $\norm{(\E_{\bx\sim D^+}-\E_{\bx\sim D^-}) \bx}_F \le \gamma^c$. 
    \item $\norm{(\E_{\bx\sim D^+} - \E_{\bx\sim D^-})\bx\bx^\intercal}_F \le \gamma^c$
    \item $\norm{(\E_{\bx\sim D^+} - \E_{\bx\sim D^-})\bx^{\otimes 3}}_F\le \gamma^c$
\end{enumerate}
then
it makes $\poly(d,1/\gamma)$ many statistical queries, where each query has tolerance at most $\poly(1/d,\gamma)$ and outputs a list of $\poly(d/\gamma)$ directions $\bw \in \R^d$ such that one $\bw$ in the list satisfies $\norm{\bw_W}_2 \le \poly(\gamma/d)$.

\end{theorem}

Notice that \Cref{alg SQ} only uses $\poly(d/\gamma)$ \iid unlabeled examples drawn from $D_X$ to estimate the mean and third-moment tensor of $D_X$ up to $\poly(\gamma/d)$ accuracy, thus these steps can be implemented via $\poly(d)$ SQs, each of which has tolerance $\poly(\gamma/d)$. For detailed background about the SQ model, we refer the reader to \Cref{sec notation}.

Here, we give an overview of the proof of \Cref{th SQ efficient}. To simplify the notation, we define $\eta$-approximate solution for a polynomial function defined over the unit sphere.
\begin{definition}[$\eta$-approximate solution]\label{def sq local optimal}
Let $f : S^{d-1} \to \R$ be a polynomial function. For $\eta>0$ and $\bu \in S^{d-1}$, we say a point $\bu \in S^{d-1}$ is a $\eta$-approximate solution if
    \begin{enumerate}
    \item $f(\bu) \ge \eta_0 = \Omega(\eta)$
    \item $\norm{\proj_{\bu^\perp} \nabla f(\bu)} \le \eta_1 = O(\eta^2)$
    \item $\max_{\bz \in S^{(d-1)} \cap \bu^\perp} \bz^\intercal \nabla^2 f(\bu) \bz \le \eta_2 = O(\eta^2)$
\end{enumerate}
\end{definition}

Recall the observation obtained by \cite{vempala2011structure} summarized as \Cref{fact ica}.
\begin{fact}[Lemma 4 in \cite{vempala2011structure}]\label{fact ica}
    Let $D_X = D_V \times D_W$ be factorizable distribution over $\R^d$ such that $D_X$ has the same $m-1(m \ge 3)$ moments as Gaussian but has a different $m$th moment. Then any local maximum(minimum) $\bu^*$ of $f^*(\bu^*)$ over $S^{d-1}$ with $f^*(\bu^*)> \gamma_m(f^*(\bu^*)< \gamma_m)$ must be either in $V$ or $W$. Here $f^*(\bu) = \E_{\bx \sim D}(\bu\cdot \bx)^m$ and $\gamma_m$ is the $m$-th moment of a standard 1-dimensional Gaussian distribution.
\end{fact}

The first ingredient of the proof is to show a robust version of \Cref{fact ica} for the polynomial $f^*(\bu) = T^* \cdot \bu^{\otimes 3}=\E_{\bx\sim D_X}\left(\bx \cdot \bu\right)^3$. That is to say, given $\E_{\bx\sim D_X}(\bx)$ is close to $0$, any $\poly(\gamma/d)$ approximate solution to $f^*$ must be $\poly(\gamma/d)$-close to $V$ or $W$. 
In particular, by \Cref{th:intersection-structure-main}, the moment tensor of $D_V$ must have norm at least $\gamma^{c}$ for some constant $c$, which implies that some point in $V$ must be a $\poly(\gamma/d)$-approximate solution to $f^*$.
On the other hand, since a polynomial function is completely described by its moment tensor. By estimating each entry of the moment tensor up to error $\poly(\gamma/d)$ error, the approximate objective function $\hat{f}$ has a similar structure as the one of $f^*$ and optimizing $\hat{f}$ is enough to give us an approximate solution to $f^*$.

The second ingredient is the key difference from the prior work \citep{frieze1996learning,vempala2011structure}. After obtaining the first approximate solution $\bu_1$, instead of restricting $\hat{f}$ over the subspace $(\bu_1)^{\perp}$, we will look at a small band $B=\{\x \in S^{d-1} \mid \abs{\bu_1 \cdot x} \le \poly(\gamma/d)\}$. As long as no $\bu$, a $\poly(\gamma/d)$ approximate solution close to $V$ is added to the list $\mathcal{O}$, we must be able to add another point to $\mathcal{O}$. 
By \cite{ALON200331}, if each pair of $\bu,\bu' \in \mathcal{O}$ satisfies $\abs{\bu\cdot \bu'}\le \poly(\gamma/d)$, the size of $\mathcal{O}$ is at most $\poly(d/\gamma)$. This implies that \Cref{alg SQ} can terminate. Now we are able to present the full proof.


\begin{proof}[Proof of \Cref{th SQ efficient re}]
To start with, we consider the case where we have the exact access to the moment tensor of $D_X$. Without loss of generality, we assume $\norm{\mu := \E_{\bx \sim D_X}\bx}_2\le o((\gamma/d)^{2c_1})$ for some large constant $c_1>c$, 
because otherwise, we can estimate $\mu$ up to error $(\gamma/d)^{2c_1}$ and shift $D_X$ to $\mu$ and rescale each shifted example by a factor of $2$ to make $\norm{\bx}_2 \le 1$. This will only decrease the margin $\gamma$ by a factor of at most $2$.
Let $T^*: = \E_{\bx\sim D_X} \bx^{\otimes 3}$ be the 3rd moment tensor of $D_X$ and $f^*(\bu) = T^* \cdot \bu^{\otimes 3} = \E_{\bx\sim D_X}\left(\bx \cdot \bu\right)^3$. 

Consider any $\bu \in S^{(d-1)}$. We write $\bu = s\bu_V^0 + t\bu_W^0$, where $s \ge 0, t \ge 0, s^2+t^2 =1$ and $\bu_V^0: = \bu_V/\norm{\bu_V}, \bu_W^0: = \bu_W/\norm{\bu_W}$. We show that if $\bu$ is an $\eta:=(\gamma/d)^{c_1}$ approximate solution to $f^*$,
    then $\norm{\bu_W} \le \poly(\gamma/d)$ or $\norm{\bu_W} \le \poly(\gamma/d)$. 
Observe that 
\begin{align*}
    f^*(\bu) &= \E_{\bx\sim D_X}\left(\bx \cdot \bu\right)^3 = \E_{\bx\sim D_X}\left(\bx_V \cdot \bu_V + \bx_W\cdot \bu_W\right)^3 \\
    &= \E_{\bx\sim D_X}(\bx_V \cdot \bu_V)^3 + (\bx_W\cdot \bu_W)^3 + 3(\bx_V \cdot \bu_V)(\bx_W\cdot \bu_W)^2 + 3(\bx_V \cdot \bu_V)^2(\bx_W\cdot \bu_W) \\
    & = g^*(\bu) + h^*(\bu).
\end{align*}
Here, $g^*(\bu) = \E_{\bx\sim D_X}(\bx_V \cdot \bu_V)^3 + (\bx_W\cdot \bu_W)^3$ and $h^*(\bu) = \E_{\bx\sim D_X} 3(\bx_V \cdot \bu_V)(\bx_W\cdot \bu_W)^2 + 3(\bx_V \cdot \bu_V)^2(\bx_W\cdot \bu_W)$. 

In particular, since $\norm{\mu}\le (\gamma/d)^{2c_1}$ and $D$ is factorizable, we know that $h^*$ is a degree-3 polynomial that can be characterized with a tensor with Frobenius norm at most $o((\gamma/d)^{c_1})$.

Write \[g^*(\bu) = \E_{\bx\sim D_X}(\bx_V \cdot \bu_V)^3 + (\bx_W\cdot \bu_W)^3 = s^3\E_{\bx\sim D_X}(\bx_V \cdot \bu^0_V)^3 + t^3\E_{\bx\sim D_X}(\bx_W\cdot \bu^0_W)^3.\] 
For simplicity, we define $a_V := \E_{\bx\sim D_X}(\bx_V \cdot \bu^0_V)^3, a_W:=\E_{\bx\sim D_X}(\bx_W\cdot \bu^0_W)^3$. By the symmetry of $g^*$, without loss of generality, we assume $a_V \ge 0$.

We consider two cases. In the first case, we assume $a_W \le 0$. 

Since $f^*(\bu) \ge \eta_0$ and $h^*(\bu) \le \eta$, we know that $a_V \ge \eta_0/2$ and $s^3 \ge \eta_0/2$. Consider point $\bz:=t\bu_V^0-s\bu_W^0 \in S^{(d-1)} \cap \bu^\perp$. Since $\nabla f^*(\bu) = \nabla g^*(\bu)+\nabla h^*(\bu)$ and 
\begin{align*}
   \nabla g^*(\bu) = 3\E_{\bx\sim D_X} \left((\bx_V\cdot \bu_V)^2\bx_V+(\bx_W\cdot \bu_W)^2\bx_W \right), 
\end{align*}
we have 
\begin{align*}
    \nabla f^*(\bu) \cdot \bz = \nabla g^*(\bu) \cdot \bz + \nabla h^*(\bu) \cdot \bz = 3(ts^2 a_V-st^2 a_W)+ \nabla h^*(\bu) \cdot \bz \ge 3ts^2 a_V - o((\gamma/d)^{c_1}).
\end{align*}
Since $\norm{\proj_{\bu^\perp} \nabla f^*(\bu)} \le \eta_1$, we know that 
\begin{align*}
3ts^2 a_V \le  \nabla f^*(\bu) \cdot \bz + o((\gamma/d)^{c_1}) \le O(\eta_1).
\end{align*}
Because $a_V \ge \gamma_0/2$ and $s^3 \ge \gamma_0/2$, this implies $s \ge 1-\poly(\gamma/d)$.  Thus, $\norm{u_W} \le \poly(\gamma/d)$.

In the second case, we have $a_W \ge 0$ and $a_V \ge 0$. By symmetry, without loss of generality, we assume $a_V \ge a_W\ge 0$. We will show that if $s\ge (\gamma/d)^{c_1}$ and $t \ge (\gamma/d)^{c_1}$, and $\norm{\proj_{\bu^\perp} \nabla f^*(\bu)} \le \eta_1$, then $\max_{\bz \in S^{(d-1)} \cap \bu^\perp} \bz^\intercal \nabla^2 f^*(\bu) \bz$ must be sufficiently large, which gives a contradiction. 

Observe that 
\begin{align*}
   \nabla^2 g^*(\bu) = 6\E_{\bx\sim D_X} \left((\bx_V\cdot \bu_V)\bx_V\bx_V^\intercal+(\bx_W\cdot \bu_W)\bx_W\bx_W^\intercal \right).
\end{align*}
Recall that $\bz:=t\bu_V^0-s\bu_W^0 \in S^{(d-1)} \cap \bu^\perp$. We have 
\begin{align*}
     \bz^\intercal \nabla^2 f^*(\bu) \bz & = \bz^\intercal \nabla^2 g^*(\bu) \bz+\bz^\intercal \nabla^2 h^*(\bu) \bz  = 6st^2 a_V + 6s^2ta_W + \bz^\intercal \nabla^2 h^*(\bu) \bz \\
     & = 6s^3(t^2/s^2) a_V + 6t^3(s^2/t^2)a_W + \bz^\intercal \nabla^2 h^*(\bu) \bz \\
     & \ge (\gamma/d)^{c_1}\eta_0/2 - o((\gamma/d)^{2c_1}),
\end{align*}
which gives a contradiction to the fact that $\bu$ is a $(\gamma/d)^{c_1}$-approximate solution. Thus, we have $t \le \poly(\gamma/d)$ and $s \ge 1-\poly(\gamma/d)$. 

So far, we have shown that if any point $\bu \in S^{(d-1)}$ that is a $\eta$-approximate solution to $f^*$ must be $\poly(\gamma/d)$-close to either $V$ or $W$. 
We next show that there must be a point that is close to $V$ and is a $\eta$-approximate solution to $f^*$. Since $D_X$ satisfies the statement of \Cref{th:intersection-structure-main}, we know that $\norm{\E_{\bx \sim D^+_X} \bx^{\otimes 3}_V}\ge \Omega(\gamma^{15}), \norm{\E_{\bx \sim D^-_X} \bx_V^{\otimes 3}}\ge \Omega(\gamma^{15})$. 
This implies that 
$\norm{\E_{\bx\sim D_V}\bx^{\otimes 3}} \ge \Omega(\gamma^c)$, 
which implies that $\max_{\bu \in S^{(d-1)} \cap V} f^*(\bu) \ge \Omega(\gamma^c)$. Here $c$ is a constant smaller than $c_1$.

Let $\bu^* \in S^{(d-1)} \cap V$ such that $g^*(\bu^*) = \max_{\bu \in S^{(d-1)} \cap V} g^*(\bu)$. Since $\bu^*$ is a maximal solution of $g^*(\bu)$ restricted at $V$, we know that $\norm{\proj_{(\bu^*)^\perp} \nabla g^*(\bu^*)}=0$ and $\max_{z \in S^{(d-1)} \cap (\bu^*)^\perp} z^\intercal \nabla^2 g^*(\bu^*) z=0$. Since $\norm{\E_{\bx\sim D_V}\bx^{\otimes 3}} \ge \Omega(\gamma^c)$, we know that $g^*(\bu^*)\ge \Omega(\gamma^c).$ Since $f^* = g^*+h^*$ and $h^*$ is a degree-3 polynomial with $o((\gamma/d)^{c_1})$-small magnitude, we know that $\bu^*$ must be an $\eta$-approximate solution to $f^*$. In particular, $f^*$ can be completely described by its moment tensor. Thus, by estimating 
each entry of the moment tensor of $D_X$ up to $\poly(\gamma/d)$ error, we obtain that $\hat{f}$ is a good approximation for $f^*$. Thus, any $(\gamma/d)^{c'}$-approximate solution for $\hat{f}$, with some $c'$ slightly larger than $c_1$, must be a $(\gamma/d)^{c_1}$-approximate solution for $f^*$. So, any $(\gamma/d)^{c'}$-approximate solution for $\hat{f}$ must be also close to $V$ or $W$.

Finally, we show that the output $\mathcal{O}$ has $\poly(d/\gamma)$ size and at least one of the solution in $\mathcal{O}$ is $\poly(\gamma/d)$ close to $V$. Here, we will make use of the following lemma.
\begin{fact}[Theorem 9.3 in \cite{ALON200331}]\label{fact alon}
    Let $A \in \R^{n \times n}$ such that for all $i \in [n]$, $A_{ii}=1$ and $\abs{A_{ij}} \le \eps$ with $1/\sqrt{n}\le \eps<1/2$ for all $i \neq j$. Then 
    \begin{align*}
        \textbf{rank}(A) \ge \Omega(\frac{1}{\eps^2\log(1/\eps)} \log(n)).
    \end{align*}
\end{fact}
Denote by $n$ the size of $\mathcal{O}$ and consider the matrix $A \in \R^{n \times n}$ where $A_{ij} = \bu_i \cdot \bu_j$. Since each $\bu_t \in \R^d$ for $t \in [d]$, we know the rank of $A$ is at most $d$. Furthermore, by the construction of $\mathcal{O}$, we know that for every $i,j \in [n]$, if $i=j$, then $A_{ij} = \bu_i \cdot \bu_j =1$ and if $i \neq j$, then $\abs{A_{ij}} = \abs{\bu_i \cdot \bu_j} \le \eps:= (\gamma/d)^{c_1}$. If $n \ge \eps^{-3} = (\gamma/d)^{-3c_1}$, then by \Cref{fact alon}, we know that
\begin{align*}
    \frac{1}{\eps^2\log(1/\eps)} \log(n) = \Omega((d/\gamma)^{c_1}) > d,
\end{align*}
which gives a contradiction. Thus, $n \le \poly(d/\gamma)$.

On the other hand, let $\bu_i$ be a vector that is $\poly(\gamma/d)$-close to $V$ and $\bu_j$ be a vector that is $\poly(\gamma/d)$-close to $W$, then $\abs{\bu_i\cdot \bu_j} \le \poly(\gamma/d)$. Thus, if $\mathcal{O}$ does not contain a point $v$ that is $\poly(\gamma/d)$-close to $V$, there must be another point $\bu$ that is nearly orthogonal to all points in $\mathcal{O}$ and will be added to $\mathcal{O}$ by \Cref{alg SQ}.
This proves \Cref{th SQ efficient}.

    \end{proof}

\subsection{Proof of \Cref{th tensor}}\label{app tensor}

In this section, we give the proof of \Cref{th tensor}. To begin with, we present the main algorithm and restate the \Cref{th tensor} for convenience.

\begin{theorem}[restatement of \Cref{th tensor}]\label{th tensor re}

There is a learning algorithm $\A$ such that for every $c$, a suitably large constant and any instance of learning intersections of two halfspaces under factorizable distribution with $\gamma$-margin assumption if the input distribution $D$ satisfies 
\begin{enumerate}
    \item $\norm{(\E_{\bx\sim D^+}-\E_{\bx\sim D^-}) \bx}_F \le \gamma^c$. 
    \item $\norm{(\E_{\bx\sim D^+} - \E_{\bx\sim D^-})\bx\bx^\intercal}_F \le \gamma^c$
    \item $\norm{(\E_{\bx\sim D^+} - \E_{\bx\sim D^-})\bx^{\otimes 3}}_F\le \gamma^c$
\end{enumerate}
$\A$ runs in $\poly(d,1/\gamma)$ time and outputs a list of $d$ unit vectors $\mathcal{O}$ such that at least one direction $\bw \in \mathcal{O}$ satisfies $\norm{\bw_W}_2 \le \poly(\gamma)$ with probability $\Omega(\gamma/d)$.    

\end{theorem}

The algorithm we will analyze is \Cref{alg tensor}.

\begin{algorithm}[h]
		\caption{\textsc{TensorDirectionFinding} (Efficient algorithm for finding relevant direction with matched moments)}\label{alg tensor}
		\begin{algorithmic} [1]
\State\textbf{Input:} $\gamma\in (0,1)$ and i.i.d. sample access to a distribution $D$ on $\B^d(1)\times \{\pm 1\}$ that is an instance of learning intersections of halfspaces under product distribution. Suppose that $D$ satisfies the conditions in the statement of \Cref{th tensor}.
\State\textbf{Output:} $\mathcal{O}$, a list of $\poly(d)$, $w \in S^{d-1}$ such that at least one of $w \in \mathcal{O}$ satisfies $\norm{\proj_{V^\perp} w} \le \poly(\gamma)$.
\State Take $S_1$, a set of $m_1=\poly(d/\gamma)$ \iid samples from $D_X$ to estimate $\mu:=\E_{x\sim D_X} x $ with
\begin{align*}
     \hat{\mu}:=\frac{1}{m_1}\sum_{x\in S_1} x
\end{align*}
up to $\poly(\gamma)$ error 
\State Take $S$, a set of $N=\poly(d/\gamma)$ \iid samples from $D_X$ and estimate 
\begin{align*}
    \hat T=\frac{1}{N}\sum_{x\in S} (x-\hat{\mu})^{\otimes 3}
\end{align*}
\State Let $v \sim N(0,\frac{1}{d}I)$ be a random Gaussian vector in $\R^d$
\State Define $\hat{M}: = \hat{T} \cdot v$.
\State Compute $\mathcal{O}$, the set of $d$ eigenvectors of $\hat{M}$ via eigen-decomposition algorithm.
\State \Return $\mathcal{O}$
\end{algorithmic}
\end{algorithm}

\begin{proof}[Proof of \Cref{th tensor re}]
We consider the central third moment tensor $T^*: = \E_{\bx \sim D_X} (\bx-\mu)^{\otimes 3}$. Since $\norm{\mu}_2 \le 1$, we can without loss of generality assume $\mu = 0$, because shifting $D_X$ to $\mu$ and rescaling the distribution will only decrease the margin assumption $\gamma$ by a factor of 2. Under this assumption, we have    
    \begin{align*}
        T^*=\E_{\bx \sim D_X} \bx^{\otimes 3} = \E_{\bx \sim D_X} (\bx_V+\bx_W)^{\otimes 3} =\E_{\bx \sim D_V} \bx_V^{\otimes 3} + \E_{\bx \sim D_W} \bx_W^{\otimes 3}.
    \end{align*}
Here the second equation follows by the fact that $D_V, D_W$ are independent and have zero-mean.    
To simplify the notation, we denote by $T_V= \E_{\bx \sim D_V} \bx_V^{\otimes 3}$ and $T_W=\E_{\bx \sim D_W} \bx_W^{\otimes 3}$. 

Let $\bv \sim N(0,\frac{1}{d}I)$ be a Gaussian vector. Define random matrix $M \in \R^{d\times d}$ as
\begin{align*}
M:=   T^* \cdot \bv & =\E_{\bx \sim D_V} \bx_V\bx_V^\intercal (\bx_V \cdot \bv) + \E_{\bx \sim D_W} \bx_W\bx_W^\intercal(\bx_W \cdot \bv)\\
&= \E_{\bx \sim D_V} \bx_V\bx_V^\intercal (\bx_V \cdot \bv_V) + \E_{\bx \sim D_W} \bx_W\bx_W^\intercal(\bx_W \cdot \bv_W).
\end{align*}
To simplify the notation, we denote by $M_V = \E_{\bx \sim D_V} \bx_V\bx_V^\intercal (\bx_V \cdot \bv_V)$ and $M_W = \E_{\bx \sim D_W} \bx_W\bx_W^\intercal(\bx_W \cdot \bv_W)$.
We write the random vector $\bv_V = \alpha \bv_V^0$, where the random variable $\alpha = \norm{\bv_V}$ and the random vector $\bv_V^0 = \bv_V/\norm{\bv_V}$ is drawn uniformly from $S^{(d-1)} \cap V$. Denote by $\sigma_1,\sigma_2$ be two eigenvalues of the random matrix $M_V/\alpha$ such that $\abs{\sigma_1} \ge \abs{\sigma_2}$ and denote by $\bu^{(1)},\bu^{(2)}$ the corresponding eigenvectors to $\sigma_1,\sigma_2$.
Notice that for $i \in [2]$, 
\begin{align*}
    M \bu^{(i)} = M_V \bu^{(i)} + M_W \bu^{(i)} = M_V \bu^{(i)} = \alpha\sigma_i \bu^{(i)}.
\end{align*}
This implies $\bu^{(i)}$ is an eigenvector of $M$ with eigenvalue $\alpha\sigma^{(i)}$. Furthermore, if $\alpha\sigma_i$ for some $i \in [2]$ is not close to any eigenvalue of $M_W$, then one of the eigenvectors of $M$ must be in $V$ and we can find it via PCA. Next, we show this holds for our case. By \Cref{th:intersection-structure-main}, we know that $\norm{T_V}_F \ge \gamma^c$, for some constant $c>0$, which implies there is some vector $\bu^* \in S^{(d-1)} \cap V$ such that 
\begin{align*}
    T_V \cdot (\bu^*)^{\otimes 3} = \E_{\bx\sim D_V}(\bx_V \cdot \bu^*)^3 \ge \Omega(\gamma^c).
\end{align*}
Since $\bv_V^0$ is drawn uniformly from $S^{d-1} \cap V$, we know that 
\begin{align*}
    \Pr\left( \norm{\bu^*-\bv_V^0} \le \gamma^c/10 \right) \ge \Pr\left( \sin\theta(\bu^*,\bv_V^0) \le \gamma^c/10\right) \ge \Omega(\gamma^c).
\end{align*}
Given this happens, we show that $\sigma_1 \ge \gamma^c$ by showing that $(\bv_V^0)^\intercal(M_V/\alpha) \bv_V^0 \ge \Omega(\gamma^c)$. We have 
\begin{align*}
    (\bv_V^0)^\intercal M_V \bv_V^0 & = T^* \cdot (\bv_V^0)^{\otimes 3} = T^* \cdot (\bu^*)^{\otimes 3}- T^* \cdot \left((\bv_V^0)^{\otimes 3}-(\bu^*)^{\otimes 3}\right) \\
    & = T^* \cdot (\bu^*)^{\otimes 3} - \E_{\bx \sim D_V}\left((\bx \cdot \bv^0_V)^3-(\bx \cdot \bu^*)^3\right) \\ 
    &  \ge  \Omega(\gamma^c),
\end{align*}
where the inequality follows the fact that the polynomial function $\E_{\bx \sim D_V}(\bx \cdot \bu)^3$ is $O(1)$-Lipschitz with respect to $\bu$. Thus, with a probability at least $\Omega(\gamma^c)$, $\sigma_1 \ge \gamma^c/2$. However, we have no structural result that can guarantee $\sigma_2$ is also large. Thus, in the rest of the proof, we consider two cases and argue that in each case, the eigenvalues of $M_V$ are far from the eigenvalues of $M_W$.

In the first case, we assume that $\sigma_1-\sigma_2 \ge \gamma^c/4$. Since $M_V$ and $M_W$ are independent, we consider the $d-2$ eigenvalues of $M_W$, $\sigma_3,\dots,\sigma_d$. Recall that the two eigenvalues of $M_V$ are $\alpha \sigma_1$ and $\alpha \sigma_2$. For each $i \in \{3,\dots,d\}$, we associate each $\sigma_i$ and interval $I_i = [\sigma_i - \xi,\sigma_i+\xi ]$, where $\xi>0$ will be determined later. Notice that $\alpha \sigma_1 \in I_i$ if and only if $\sqrt{d}\alpha \in [\sqrt{d}\sigma_i/\sigma_1-\sqrt{d}\xi/\sigma_1,\sqrt{d}\sigma_i/\sigma_1+\sqrt{d}\xi/\sigma_1]$.
Since $\bv \sim N(0,\frac{1}{d}I)$, and $\bv_V$ is the 2-dimensional projection onto the subspace $V$, we know that $d\alpha^2 \sim \chi(2)$ is a $\chi$-distribution with degree of freedom $2$. Recall that $\chi(2)$ is a distribution supported over $\R^+$ with density function $p(x) = x\exp(-x^2/2) \le 1, \forall x \ge 0$. This implies that
\begin{align}\label{eq anti}
    \Pr(\alpha \sigma_1 \in I_i) \le 2\sqrt{d}\xi/\sigma_1, \forall i \in \{3,\dots,d\}
\end{align}
On the other hand, 
\begin{align*}
    \Pr(\abs{\alpha \sigma_1-\alpha \sigma_2} \le \xi) =  \Pr(\alpha \le \xi/\abs{\sigma_1-\sigma_2}) \le O(\xi/\abs{\sigma_1-\sigma_2}).
\end{align*}
By choosing $\xi = O(\gamma^{2c}/d)$, we know that with probability at least $1-O(\gamma^{2c})$, $\alpha \sigma_1$, the largest eigenvalue of $M_V$ is $\xi$-far from any other eigenvalues of $M$.

In the second case,  we assume $\sigma_1-\sigma_2 \le \gamma^c/4$, which implies that $\sigma_2 \ge \gamma^c/4$. Recall that the two eigenvalues of $M_V$ are $\alpha \sigma_1$ and $\alpha \sigma_2$. With the same proof as \eqref{eq anti}, we know that 
\begin{align}
    \Pr(\alpha \sigma_2 \in I_i) \le 2\sqrt{d}\xi/\sigma_2, \forall i \in \{3,\dots,d\}.
\end{align}
By choosing $\xi = O(\gamma^{2c}/d)$, we know that with probability at least $1-O(\gamma^{2c})$, $\alpha \sigma_1, \alpha\sigma_2$, the two eigenvalues of $M_V$ are $\xi$-far from any other eigenvalues of $M_W$.

To finish the proof of \Cref{th tensor}, we make use of the well-known Davis-Kahan $\sin \theta$ theorem.

\begin{theorem}[Davis-Kahan $\sin \theta$ Theorem]\label{th sin}
    Let $A = E_0 A_0 E_0^\intercal + E_1 A_1 E_1^\intercal, A+H = F_0 B_0 F_0^\intercal + F_1 B_1 F_1^\intercal \in \R^{d\times d}$ be symmetric matrices, where $(E_0,E_1), (F_0,F_1)$ are orthogonal matrices and $A_0,A_1,B_0,B_1$ are diagonal matrices. If every eigenvalues of $A_0$ are $\delta$-far from the eigenvalues of $A_1$, then 
    \begin{align*}
        \norm{F_0^\intercal E_1}_2 \le \frac{\norm{H}_2}{\delta}.
    \end{align*}
\end{theorem}
Since $\bv\sim N(0,\frac{1}{d}I)$, by \cite{Ver18}, we know that with probability at least $1-\gamma^d$, $\norm{\bv}_2 \le 1+\log(1/\gamma)$. Since by taking $\poly(d/\gamma)$ samples from $D_X$, we are able to estimate each entry of $T^*$ with $\hat{T}$ up to error $\poly(\gamma/d)$. Thus, given $\norm{\bv}_2 \le 1+\log(1/\gamma)$, $\norm{M-\hat{M}}_2 \le (\gamma/d)^{3c}$.

Recall that we have discussed two cases for the behavior of the eigenvalues of $M_V$. In the first case, the largest eigenvalue $\alpha \sigma_1$ is $\gamma^2c/d$ far from any other eigenvalues of $M$. Using \Cref{th sin} by taking $A=M, A+H=\hat{M}, E_0=\bu^{(1)}$, we know that there is an eigenvector of $\hat{M}$, $\hat{\bu^{(1)}}$ such that $\norm{\hat{\bu^{(1)}}_W} \le \gamma^c$. In the second case the two eigenvalues  $\alpha \sigma_1, \alpha \sigma_2$ are $\gamma^2c/d$ far from any other eigenvalues of $M_W$, using \Cref{th sin} by taking $A=M, A+H=\hat{M}, E_0=[\bu^{(1)},\bu^{(2)}]$, we know that there is an eigenvector of $\hat{M}$, $\hat{\bu^{(1)}}$ such that $\norm{\hat{\bu^{(1)}}_W} \le \gamma^c$. Thus, with probability at least $\Omega(\gamma^c)$, \Cref{alg tensor} outputs a list of $d$ unit vectors such that at least one of them is $\gamma^c$ close to $V$.

\end{proof}

\paragraph{Remark} In general, given the estimated moment tensor $\hat{M}$, there is no polynomial time algorithm that can \emph{exactly} compute all eigenvectors of $\hat{M}$ due to the roundoff error. However, as long as the roundoff error is at most $\poly(\gamma/d)$, one can compute in polynomial time an eigen-decomposition that is $\poly(\gamma/d)$ close to $\hat{M}$. We refer the readers to \cite{dhillon2006design} for detailed guarantee of efficient algorithms for eigen-decomposition that takes the roundoff error into consideration.


\section{Omitted Proofs from \Cref{sec:direction-extraction-csq}} \label{app csq}

This section is dedicated to proving \Cref{{thm:csq-direction-extraction}}. The main algorithm we will analyze is \Cref{alg csq}.

\begin{algorithm} 
    \caption{\textsc{DirectionFinding} (Relevant direction extraction with Mismatched Moments)} \label{alg csq}
    \begin{algorithmic}[1]
    \State \textbf{Input:} $\alpha\in (0,1), m\in \{1,2,3\}$ and i.i.d. sample access to a distribution $D$ on $\B^d(1)\times \{\pm 1\}$ that is an instance of learning intersections of halfspaces under product distribution. Suppose $D$ does not satisfy the $(\alpha,m)$-moment matching condition and $D$ satisfies the $(\alpha^2/d^c,t)$-moment matching condition for any $t\le m-1$ and a sufficiently large universal constant $c$. 
        
    \State \textbf{Output:} With $2/3$ probability, the algorithm outputs a unit vector $\bu\in S^{d-1}$ such that $\norm{\bu_W}=O(\alpha)$. 
    
    \State Define $T=\E_{\bx\sim D^+}\bx_V^{\otimes m}-\E_{\bx_V\sim D^-}\bx_V^{\otimes m}$. We will use $\hat{T}$ as an empirical estimation of $T$. Set $\zeta=\alpha/d^{c-1}$. Draw $N=\poly(d,1/\alpha)$ i.i.d. samples $S^+$ and $S^-$ from $D^+$ and $D^-$ respectively. Then estimate 
    \[
    \hat T=\frac{1}{N}\sum_{\bx\in S^+} \bx^{\otimes m}-\frac{1}{N}\sum_{\bx\in S^-} \bx^{\otimes m}\; .
    \]
    
    \State Define $f:S^{d-1}\to \R$ as $f(\bu):= \hat T\cdot \bu^{\otimes m}$ and apply the following standard gradient descent steps. 
    Let $T=(d/\alpha)^{c'}$ for a sufficiently large constant $c'$ depending on $c$.
        \begin{enumerate}
            \item \label{step:csq-direction-extraction-initialization} Initialize a random $\bu_0\sim_u S^{d-1}$. If $f(\bu_0)< \zeta$, then reinitialize. If reinitialized $T$ times, then output failure.
            \item 
            Repeat the following for at most $T$ many iterations.
            For the $t$-th iteration, calculate the gradient $\bg=\proj_{\bu_t^{\perp}}\nabla f(\bu_t)$, and update $\bu_{t+1}=\frac{\bu_t+\lambda \bg}{\left \|\bu_t+\lambda \bg\right \|_2}$ where the stepsize $\lambda=c''\min(1,1/\left \|\proj_{\bu} \nabla f(\bu_t)\right \|_2)$ and $c''$ is a sufficiently small universal  constant. Repeat this step until getting a unit vector $\bu_t$ that is a $(\alpha,\eta)$-approximate solution such that $\frac{\eta + \alpha^2/d^c}{\alpha'-\alpha^2/d^c}\geq c''\alpha$, then output such $\bu_t$. If the condition is not satisfied in $T$ many iterations, then output failure.
        \end{enumerate}
    \end{algorithmic}
\end{algorithm}

\subsection{Approximate Local Optimum Implies Relevant Direction}
The following definition plays a core role of the proof.
\begin{definition}[$(\alpha,\eta)$-approximate solution]
Let $f: \R^{d} \to \R$ be a real-valued function and $\alpha,\eta>0$. 
We say that $\bw \in S^{d-1}$ is an $(\alpha,\eta)$-approximate solution to the problem $\max_{\bw \in S^{d-1}} f(\bw)$ if satisfies the following conditions: (1) $\norm{\proj_{\bw^\perp} \nabla f(\bw)}_2 \le \eta$, and (2) $ \abs{f(\bw)} \ge \alpha$.
\end{definition}

Suppose we have access to the exact function $f(\bu)=\bu^{\otimes m}\cdot T$ where $T=\E_{\bx\sim D^+}[\bx_V^{\otimes m}]-\E_{\bx\sim D^-}[\bx_V^{\otimes m}]$.
If the first $m-1$ moments are highly matched, then finding an approximate local maximum $\bu^*$ of $f$ implies that $\bu^*$ is close to $V$.


\begin{lemma}\label{lm csq-approximate}
    Let $\alpha,\eta,\xi>0$, $m\in \Z_+$ and $D$ be a joint distribution of $(\bx,y)$ on $\B^d(1)\times \{\pm 1\}$ that is an instance of learning intersections of halfspaces under factorizable distribution with $\gamma$-margins. 
    Suppose $D$ does not satisfy $(\xi,m)$-moment matching condition.
    Let $f(\bu)=\bu^{\otimes m}\cdot T$ where $T=\E_{\bx\sim D^+}[\bx_V^{\otimes m}]-\E_{\bx\sim D^-}[\bx_V^{\otimes m}]$.
    Then any $\bu^*$ that is an $(\alpha,\eta)$-approximate solution to $\max_{\bu \in S^{d-1}} f(\bu)$ must satisfy $\norm{\bu^*_W}_2 \le \eta/(m\alpha)$
\end{lemma}

\begin{proof}[Proof of \Cref{lm csq-approximate}]
Let $\bu \in S^{d-1}$ and $W=V^{\perp}$. To simplify the notation, we write $\Bar{\bu}_V = \bu_V/\norm{\bu_V}_2$, $\Bar{\bu}_W = \bu_W/\norm{\bu_W}_2$ and $\bu = s \Bar{\bu}_V + t \Bar{\bu}_W$, where $s,t \ge 0$ and $s^2 + t^2 =1$. The gradient of $f$ at $\bu$ is 
\begin{align*}
    \nabla f(\bu) & = m \left( \E_{\bx\sim D^+} [(\bu_V\cdot \bx_V)^{m-1} \bx_V] - \E_{\bx\sim D^-} [(\bu_V\cdot \bx_V)^{m-1} \bx_V] \right) \\
                & = ms^{m-1}\left( \E_{\bx\sim D^+} [(\Bar{\bu}_V\cdot x_V)^{m-1} x_V] - \E_{\bx\sim D^-} [(\Bar{\bu}_V\cdot x_V)^{m-1} x_V ]\right).
\end{align*}
Consider $\Tilde{\bu} = t \Bar{\bu}_V - s \Bar{\bu}_W \in S^{d-1} \cap \bu^\perp$, we have 
\begin{align*}
    \nabla f(\bu) \cdot \Tilde{\bu} & = mts^{m-1}\left( \E_{\bx\sim D^+} [(\Bar{\bu}_V\cdot \bx_V)^{m-1} \bx_V] - \E_{\bx\sim D^-} [(\Bar{\bu}_V\cdot \bx_V)^{m-1} \bx_V] \right) \cdot \Bar{\bu}_V \\
    & = mts^{m-1}\left( \E_{\bx\sim D^+} [(\Bar{\bu}_V\cdot \bx_V)^{m}]  - \E_{\bx\sim D^-} [(\Bar{\bu}_V\cdot \bx_V)^{m}] \right) \\
    & = \frac{tm}{s}s^{m}\left( \E_{\bx\sim D^+} (\Bar{\bu}_V\cdot \bx_V)^{m}  - \E_{\bx\sim D^-} (\Bar{\bu}_V\cdot \bx_V)^{m}  \right) \\
    & = \frac{tm}{s}\left( \E_{\bx\sim D^+} [(\bu_V\cdot \bx_V)^{m}]  - \E_{\bx\sim D^-} [(\bu_V\cdot \bx_V)^{m}]  \right) = \frac{tm}{s} f(\bu)
\end{align*}

If $\bu^*$ is an $(\alpha,\eta)$-approximate solution, then 
\begin{align*}
    t \le \abs{\frac{\eta s}{m f(u)}} \le \frac{\eta}{m\alpha}.
\end{align*}
Thus, we have $\norm{\bu^*_W}_2 \le \eta/(m\alpha)$.
    
\end{proof}

However, since $V$ is unknown to us and we only have sample access to the distribution $D$, we cannot hope to know the exact function $f$, which requires exact moment information of the distribution over $V$. 
To overcome these problems, we first show that
if we replace $T$ with its empirical estimation of the moments,
the statement of \Cref{lm csq-approximate} still holds.

\begin{lemma}\label{lm stability-f}
Let $\alpha,\eta,\xi>0$, $m\in \Z_+$ and $D$ be a joint distribution of $(\bx,y)$ on $\B^d(1)\times \{\pm 1\}$ that is consistent with an instance of learning intersections of two halfspaces under factorizable distribution with $\gamma$-margin assumption. 
Suppose $D$ does not satisfy $(\xi,m)$-moment matching condition.
Let $\hat{f}(\bu)=\bu^{\otimes m}\cdot\hat{T}$ where $\|\hat{T}-T\|_F\leq \eps$ and $T=\E_{\bx\sim D^+}[\bx_V^{\otimes m}]-\E_{\bx\sim D^-}[\bx_V^{\otimes m}]$.
Then any $\bu^*$ that is an $(\alpha,\eta)$-approximate solution to $\max_{\bu \in S^{d-1}} \hat{f} (\bu)$ must satisfy $\norm{ \bu^*_W} \le \frac{(\eta + \epsilon)}{m (\alpha-\epsilon)}$.
\end{lemma}

\begin{proof}[Proof of \Cref{lm stability-f}]
    Let $\bu \in S^{d-1}$ and $W=V^{\perp}$. To simplify the notation, we write $\Bar{\bu}_V = \bu_V/\norm{\bu_V}$, $\Bar{\bu}_W = \bu_W/\norm{\bu_W}$ and $\bu = s \Bar{\bu}_V + t \Bar{\bu}_W$, where $s,t \ge 0$ and $s^2 + t^2 =1$. We compute the gradient of $\hat{f}$ at $\bu$.
    Notice that $\hat{f}(\bu)=\hat{T}\cdot\bu^{\otimes m}$ and 
    \[
    \nabla \hat{f}(\bu)= \nabla f(\bx)+\nabla(\hat{T}-T^*\cdot\bu^{\otimes m})\; .
    \]
    Therefore, 
    \begin{align*}
        \nabla \hat{f}(\bu)\cdot\Tilde{\bu}=&\nabla f(\bu)\cdot\Tilde{\bu}+\nabla(\hat{T}-T^*\cdot\bu^{\otimes m})\cdot \Tilde{\bu}\\
        =&\frac{tm}{s} f(\bu)  +\nabla(\hat{T}-T^*\cdot\bu^{\otimes m})\cdot \Tilde{\bu}\\
        =&\frac{tm}{s} \hat{f}(\bu)+ \left (\frac{tm}{s} f(\bu)-\frac{tm}{s} \hat{f}(\bu)\right ) +\nabla(\hat{T}-T^*\cdot\bu^{\otimes m})\cdot \Tilde{\bu}\\
        =&\frac{tm}{s} \hat{f}(\bu)+ \frac{tm}{s}(\hat{T}-T)\cdot\bu^{\otimes m}+\nabla(\hat{T}-T^*\cdot\bu^{\otimes m})\cdot \Tilde{\bu}\\
        =&\frac{tm}{s} \hat{f}(\bu)+ \frac{tm}{s}(\hat{T}-T)\cdot\bu^{\otimes m}+(\hat{T}-T^*)\cdot\bu^{\otimes m-1}\otimes \Tilde{\bu}\; .&
    \end{align*}
    Since $\norm{\hat{T}-T^*}_F \le \eps$, if $\bu$ is an $(\alpha,\eta)$-approximate solution, we have 
    $
        \eta \ge \abs{\nabla \hat{f}(\bu) \cdot \Tilde{\bu}} \ge \frac{tm}{s}(\alpha-\eps) -  \eps,
    $
    which implies
    $
        t \le \abs{\frac{(\eta + \epsilon)s}{m (\alpha-\epsilon)}} \le \frac{(\eta + \epsilon)}{m (\alpha-\epsilon)}.
    $
    Thus, we have $\norm{\bu^*_W}_2 \le \frac{(\eta + \epsilon)}{m (\alpha-\epsilon)}$.
\end{proof}

However, since $V$ is unknown to us, even if we have \iid sample access to $D$, it is still unclear how to estimate $T$ with samples. To overcome this difficulty, we show in the next lemma that if $D$ satisfies $(\eps,t)$-moment matching condition for $t \le m-1$, then we can efficiently estimate $T$ up to a desired accuracy.

\begin{lemma}\label{lm tensor estimate csq}
Let $D$ over $\B^d(1) \times \{\pm 1\}$ be a distribution that is consistent with an instance of learning intersections of two halfspaces under product distribution. 
Consider degree $m$-moment tensors over $\R^d$, $T^* = \E_{\bx \sim D_X^+} \bx_V^{\otimes m} - \E_{\bx \sim D_X^-} \bx_V^{\otimes m}$ and $\hat{T} = \frac{1}{n}\sum_{\bx \in S_+} \bx^{\otimes m} - \frac{1}{n}\sum_{\bx \in S_-} \bx^{\otimes m}$, where $n=\poly( d^m,1/\epsilon,\log(1/\delta))$. If for $i<m$, $\D$ satisfies the $(\epsilon/2^m,i)-$moment matching condition, then with probability at least $1-\delta$, $\norm{\hat{T}-T^*} \le  2\epsilon$.
\end{lemma}

\begin{proof}[Proof of \Cref{lm tensor estimate csq}]
Denote by $T:=\E_{\bx\sim D_X^+} \bx^{\otimes m}-\E_{\bx\sim D_X^-} \bx^{\otimes m} = \left(\E_{\bx\sim D_X^+}-\E_{\bx\sim D_X^-}\right)\bx^{\otimes m}$.
    \begin{align*}
    \norm{T-T^*}_F & = \norm{\left(\E_{\bx\sim D_X^+}-\E_{\bx\sim D_X^-}\right)\left((\bx_V+\bx_W)^{\otimes m} - \bx_V^{\otimes m}\right)}_F 
         \\
         &= \norm{\left(\E_{\bx\sim D_X^+}-\E_{\bx\sim D_X^-}\right)\left((\bx_V+\bx_W)^{\otimes m} - \bx_V^{\otimes m}-\bx_W^{\otimes m}\right)}_F \\
        & =  \norm{\left(\E_{\bx\sim D_X^+}-\E_{\bx\sim D_X^-}\right)\sum_{i=0}^{m}\sum_{\sigma_i} \left(\bigotimes_{j=1}^m \bx_{\sigma_i(j)}\right) - \bx_V^{\otimes m}-\bx_W^{\otimes m}}_F \\
         & = \norm{\left(\E_{\bx\sim D_X^+}-\E_{\bx\sim D_X^-}\right)\sum_{i=1}^{m-1}\sum_{\sigma_i} \left(\bigotimes_{j=1}^m \bx_{\sigma_i(j)} \right)}_F\\
        & \le \sum_{i=1}^{m-1}\sum_{\sigma_i}\norm{\left(\E_{\bx\sim D_X^+}-\E_{\bx\sim D_X^-}\right) \left(\bigotimes_{j=1}^m \bx_{\sigma_i(j)} \right)}_F \\
        & = \sum_{i=1}^{m-1}\binom{m}{i}\norm{\left(\E_{\bx\sim D_X^+}-\E_{\bx\sim D_X^-}\right) \left(\bx_V^{\otimes i}  \otimes \bx_W^{\otimes {m-i}} \right)}_F\\
         & = \sum_{i=1}^{m-1}\binom{m}{i}\norm{\left(\E_{\bx\sim D_X^+}\left(\bx_V^{\otimes i}  \otimes \bx_W^{\otimes {m-i}} \right)-\E_{\bx\sim D_X^-}\left(\bx_V^{\otimes i}  \otimes \bx_W^{\otimes {m-i}} \right)\right) }_F \\
         & = \sum_{i=1}^{m-1}\binom{m}{i}\norm{\left(\E_{\bx\sim D_X^+}-\E_{\bx\sim D_X^-}\right)\left(\bx_V^{\otimes i}\right)  \otimes E_{\bx\sim D_X}\bx_W^{\otimes {m-i}}  }_F\\ 
         & = \sum_{i=1}^{m-1}\binom{m}{i}\norm{\left(\E_{\bx\sim D_X^+}-\E_{\bx\sim D_X^-}\right)\left(\bx_V^{\otimes i}\right) }_F \norm{ E_{\bx\sim D_X}\bx_W^{\otimes {m-i}}  }_F \le (\eps/2^m)  \sum_{i=1}^{m-1}\binom{m}{i} \le \epsilon.
    \end{align*}
The second equation holds by $E_{\bx\sim D_X^+} \bx_W^{\otimes m} = E_{\bx\sim D_X^-} \bx_W^{\otimes m} = E_{\bx\sim D_X} \bx_W^{\otimes m}$. In the third equation, we denote by $\sigma_i:[m] \to \{V,W\}$ a map of combination such that $\card{\{j \in [m]: \sigma_i(j)=V\}}=i$. The fifth equation holds because the tensor $\left(\E_{\bx\sim D_X^+}-\E_{\bx\sim D_X^-}\right) \left(\bigotimes_{j=1}^m \bx_{\sigma_i(j)} \right)$ is symmetric according to the map of combination $\sigma_i$. In the seventh equation, we use the fact that $\bx_V$ and $\bx_W$ are independent. In the second last inequation, we use the fact that $D$ satisfies the $(\epsilon,i)-$moment matching condition and $\norm{\bx} \le 1$ for sure.

By Hoeffding's inequality, we know that 
\begin{align*}
    \Pr \left(\norm{\hat{T}-T}_F \ge \epsilon\right) & \le \sum_{i=1}^{d^m}\Pr \left(\abs{T_i-\hat{T}_i} \ge d^{-m/2} \epsilon\right) \\
   & \le 2d^m \exp\left(-n d^{-m} \epsilon^2\right) \le \delta,
\end{align*}
when $n = \poly(d^m,1/\epsilon,\log(1/\delta))$, where $T_i$ is the $i-$th entry of the vectorized $T$. 
Thus, we have 
\begin{align*}
    \norm{\hat{T}-T^*}_F \le \norm{\hat{T}-T}_F + \norm{T-T^*}_F \le 2 \epsilon.
\end{align*}

\end{proof}

\subsection{Proof of \Cref{thm:csq-direction-extraction}}
Given \Cref{lm stability-f}, we are now ready to analyze \Cref{alg csq} and prove \Cref{thm:csq-direction-extraction}.
For convenience, we restate \Cref{thm:csq-direction-extraction} statement as \Cref{thm:csq-direction-extraction-re}.

\begin{theorem}[restatement of \Cref{thm:csq-direction-extraction}] \label{thm:csq-direction-extraction-re}
Let $m\leq 3$, there is an algorithm $\mathcal{A}$ (\Cref{alg csq}) such that for any instance of learning intersections of two halfspaces under factorizable distributions, 
if the distribution $D$ does not satisfy the 
$(\alpha,m)$-moment matching condition 
and $D$ satisfies the $(\alpha^2d^{-c}/2^m,t)$-moment matching condition 
for any $t\le m-1$ and a sufficiently large universal constant $c$, 
then $\mathcal{A}$ draws 
$\poly(d, 1/\alpha)$ \iid samples from $D$, 
runs in time $\poly(d, 1/\alpha)$, 
and outputs a unit vector $\bu\in S^{d-1}$ 
such that $\norm{\bu_W} = O(\alpha)$ with probability $2/3$. 
\end{theorem}

\begin{proof} [Proof of \Cref{thm:csq-direction-extraction}]
By \Cref{lm tensor estimate csq}, the empirical estimation $\hat{T}$ satisfies $\|\hat{T}-T\|_F\leq \alpha^2/d^c$ with high probability for a sufficiently large constant $c$.
From \Cref{lm stability-f}, 
we have that for any $(\alpha',\eta)$-approximate solution,
$\norm{ \bu_W} \le \frac{\eta + \alpha^2/d^c}{m (\alpha'-\alpha^2/d^c)}$.
Therefore, it suffices for us to show that 
after at most $T$ steps of the gradient descent, we will with high probability find a $(\alpha',\eta)$-approximate solution such that $\frac{\eta + \alpha^2/d^c}{\alpha'-\alpha^2/d^c}=O(\alpha)$.

We start by showing that the initialization will with high probability give us a $\bu_0$ such that $f(\bu_0)=\Omega(\alpha/\poly(d))$.
Since the empirical estimation $\hat{T}$ satisfies $\|\hat{T}-T\|_F\leq \alpha^2/d^c$ with high probability,
we must have $\|\proj_{V^{\otimes 3}}(\hat{T})\|_F\geq\|\proj_{V^{\otimes 3}}(T)\|_F-\|\proj_{V^{\otimes 3}}(\hat{T}-T)\|_F\geq \|T\|_F-\|\hat{T}-T\|_F=\Omega(\alpha)$.
Given \Cref{fct:lb-tensor-random-rank1-correlation} and $\hat{T}$ is inside the subspace of $V^{\otimes 3}$ (isomorphic to $(\R^2)^{\otimes 3}$), we have that $\max_{\bu\in S^{d-1}} \bu^{\otimes m}\cdot\hat{T}=\max_{\bu\in S^{d-1}\land \bu\in V} \bu^{\otimes m}\cdot\hat{T}=\Omega(\alpha)$.
We then use the following fact, which is  
Lemma 12 from \cite{vempala2011structure} to show that with high probability, we will get an initialization with large $f(\bu_0)$.   
\begin{fact} \label{fct:robust-sz}
        Let $p$ be a degree-$m$ polynomial over $d$ variables and $K$ a convex body in $\R^d$. If there exists an $\bx\in K$ such that $|p(\bx)|>\eps(c'd)^m$, for some suitable constant $c'>0$, then for $l$ random points $\bs_i$, $\pr(\forall \bs_i:|p(\bs_i)|\leq \eps)\leq 2^{-l}$.
\end{fact}
We apply \Cref{fct:robust-sz} with $K=\{\bu\in\R^d\mid \|\bu\|_2\leq 1\}$. It is easy to see that if we sample the initialization $\bu_0$ uniformly from $S^{d-1}$ instead of $K$, $f(\bu_0)$ will be larger.
Therefore, we have that with probability at least $\Omega(1)$, $\pr_{\bu_0\sim S^{d-1}} \bu_0^{\otimes m}\cdot \hat{T} \geq\alpha/\poly(d)$. 
Therefore in Step~\ref{step:csq-direction-extraction-initialization}, with probability $\Omega(1)$, we will have $f(\bu_0)\geq \alpha/d^{c'}$, where $c'$ is a universal constant.

Suppose we have such a good initialization, we start analyzing the gradient descent step in \Cref{alg csq}. To do this, we prove the following two facts. The first fact is about the smoothness of the function $f$ we are optimizing.

\begin{fact}
    The function $f$ in \Cref{alg csq} satisfies $\norm{\nabla f(\bu)-\nabla f(\bu')} \le O(\norm{\bu-\bu'})$. i.e. $f$ is a $\Omega(1)$-smooth function.
\end{fact}
\begin{proof}
    Consider $g:(\R^d)^m\to\R$ and $h:\R^d\to \R^d$ defined as
    $g(\bu_1,\cdots,\bu_m)=T\cdot(\bu_1\otimes \cdots,\bu_m)$
    and $h(\bu)=\bu$.
    Then we have $f(\bu)=g(\bu_1,\cdots,\bu_m )$, where each $\bu_i=g(\bu)$.
    Therefore,
    \[
    \nabla f(\bu)=\nabla g(\bu_1,\cdots,\bu_m )=\sum_{i=1}^m \frac{\partial \bu_i}{\partial \bu}\frac{\partial g}{\partial\bu_i}=\sum_{i=1}^m I(T\bu^{\otimes m-1})=m T\cdot \bu^{\otimes m-1}\; , 
    \]
    where $\frac{\partial g}{\partial\bu_i}=T\bu^{\otimes m-1}$ follows from that $T$ is a symmetric tensor.


    Therefore, we get
    \begin{align*}
        \|\nabla f(\bu)-\nabla f(\bu')\|_2= &\|mT\cdot (\bu^{\otimes m-1}-\bu'^{\otimes m-1})\|_2\\
        \le & m\|T\|_F\sum_{i=1}^m \binom{m}{i} \|(\bu'-\bu)^{\otimes i}\bu^{\otimes m-i}\|_F\\
        \le & m\|T\|_F\sum_{i=1}^m \binom{m}{i} \|(\bu'-\bu)\|_2^i\|\bu\|_2^{m-i}\\
        \leq & m \|T\|_F\sum_{i=1}^m \binom{m}{i} 2^m \|(\bu'-\bu)\|_2\\
        \leq &m^2m^m\|T\|_F2^m\|(\bu'-\bu)\|_2=O(\|\bu'-\bu\|_2)\; ,
    \end{align*}
    In the first inequality, we use the symmetry of $\bu^{\otimes i}, i \in [m]$. In the third inequality, we use the fact that $\norm{\bu'-\bu} \le 2$.
    And in the last inequality follows from $m\leq 3$ and $T=\E_{\bx\sim D^+}\bx^{\otimes m}-\E_{\bx\sim D^-}\bx^{\otimes m}$, where $D^+$ and $D^-$ are supported on $\B^d(1)$.
\end{proof}

The next fact measures the progress made in each step during the optimization step.

\begin{fact} \label{fct:descent-iteration-progress}
Let $f:S^{d-1}\to R$ be an $L$-smooth function. For $\bu\in S^{d-1}$, $\bg=\proj_{\bu^{\perp}}\nabla f(\bu)$, and $\bu'=\frac{\bu+\lambda \bg}{\left \|\bu+\lambda \bg\right \|_2}$ where the stepsize $\lambda=c\min(1/L,1/\left \|\proj_{\bu} \nabla f(\bu)\right \|_2,1)$ and $c$ is a sufficiently small constant, we have
$f(\bu')-f(\bu)=\Omega(\lamdba\|\bg\|_2^2)$.
\end{fact}


\begin{proof}
Let the angle between $\bu$ and $\bu'$ be $\theta$. Then, $\lambda \|\bg\|_2=\tan(\theta)$ is at most a sufficiently small constant.
Notice that 
    \begin{align*}
        f(\bu')-f(\bu)
        \geq &(\bu'-\bu)\cdot\nabla f(\bu)-L\|\bu'-\bu\|^2\\
        = & \lambda \bg\cdot \bg+(\|\bu+\lambda \bg\|_2-1)\bu'\cdot \nabla f(\bu)-L(2\sin(\theta/2))^2\\
        \geq & \lambda \|\bg\|_2^2+(\|\bu+\lambda \bg\|_2-1)(\proj_\bu\bu'\cdot \proj_\bu \nabla f(\bu) +\proj_{\bu^{\perp}}\bu'\cdot \bg)-L\lambda^2\|\bg\|_2^2\\
        = & \lambda \|\bg\|_2^2+(1/\cos(\theta)-1)(\cos(\theta) \|\proj_\bu \nabla f(\bu)\|_2 -\sin(\theta)\|\bg\|_2)-L\lambda^2\|\bg\|_2^2\\    
        \geq & \lambda \|\bg\|_2^2+O(\sin\theta)^2(\cos(\theta) \|\proj_\bu \nabla f(\bu)\|_2 -\sin(\theta)\|\bg\|_2)-L\lambda^2\|\bg\|_2^2\\    
        \geq & \lambda \|\bg\|_2^2+O(\lambda^2\|\bg\|_2^2)\|\proj_\bu \nabla f(\bu)\|_2 -\sin(\theta)^2\lambda\|\bg\|_2^2-L\lambda^2\|\bg\|_2^2\\    
        \geq & \lambda \|\bg\|_2^2+O(\lambda\|\proj_\bu \nabla f(\bu)\|_2)\lambda\|\bg\|_2^2 -O(\lambda)\|\bg\|_2^2-(L\lambda)\lambda\|\bg\|_2^2\\
        =&\Omega(\lamdba\|\bg\|_2^2)\; ,
    \end{align*}
Here, in the first inequality, we use Taylor expansion for $f$ as well as the smoothness of $f$. In the second inequality, we use the fact that $\lambda\norm{\bg}=\tan \theta$.   
    And in the last equality follows from $\lambda=c\min(1/L,1/\left \|\proj_{\bu} \nabla f(\bu)\right \|_2,1)$.
\end{proof}

Now we assume for the purpose of contradiction that the algorithm did not terminate in $T$ iterations, which implies that any $\bu_t$ is not an $(\alpha',\eta)$-approximate solution we desire for any $t\leq T$. 
From \Cref{fct:descent-iteration-progress}, we have that $f(\bu_t)$ is monotone increaseing, therefore, we must have $f(\bu_t)\geq f(\bu_0)\geq \alpha/d^{c-1}$ for any $t$.
Since we know that any $\bu_t$ is not an $(\alpha',\eta)$-approximate solution we desire.
This implies that $\frac{\eta + \alpha^2/d^c}{\alpha/d^{c-1}-\alpha^2/d^c}= \Omega(\alpha)$, and therefore we must have $\eta=\|\proj_{\bu_t^{\perp}}\nabla f(\bu_t)\|_2\geq\alpha^2/d^c$.
This implies that we must make some progress for each step of the gradient descent.
Accoriding to \Cref{fct:descent-iteration-progress}, we must have that
\begin{align*}
f(\bu_{t+1})-f(\bu_t)= &\Omega(\min(1,1/\|\proj_{\bu_t}\nabla f(\bu_t)\|_2)\|g\|_2^2\\
= &\Omega(\min(1,1/\|\proj_{\bu_t}\nabla f(\bu_t)\|_2)\alpha^4/\poly(d)\\
= &\Omega(\min(1,1/(mf(\bu_t)))\alpha^4/\poly(d)=1/\poly(d/\alpha)\; ,
\end{align*}
where the second from the last inequality follows from we use the fact that $\proj_{\bu_t}\nabla f(\bu_t) \cdot \bu_t = m f(\bu_t)$.
Therefore, after $T$ iterations, we get $f(\bu_T)\geq T\alpha^2/\poly(d)+f(\bu_0)\geq T\alpha^2/\poly(d)=\omega(\alpha)$.
While for any $\bu\in S^{d-1}$, we should have $f(\bu)=\bu^{\otimes 3}\cdot \hat{T}\leq \|\hat{T}\|_F\leq \|T\|_F+\alpha^2/\poly(d)= O(\alpha)$.
This gives a contradiction and therefore, we must find a $(\alpha',\eta)$-approximate solution such that $\frac{\eta + \alpha^2/d^c}{\alpha'-\alpha^2/d^c}=O(\alpha)$ before $T$ iterations.
This completes the proof.
\end{proof}

\section{Omitted Proofs from \Cref{sec local}}\label{app local}

\subsection{Proof of \Cref{lem:banding-PTF}}\label{app band}
In this section, we give the proof of \Cref{lem:banding-PTF}. For convenience, we restate \Cref{lem:banding-PTF} below.

\begin{lemma} [Restatement of \Cref{lem:banding-PTF}] \label{lem:banding-PTF-re}
    Let $D$ be a joint distribution of $(\bx,y)$ on $\B^d(1)\times \{\pm 1\}$ that is consistent with an intersection of halfspaces with $\gamma$-margins and $\bw\in S^{d-1}$ such that $\norm{\bw_V}_2 \le c \gamma$, for some small constant $c$, where $V$ is the relevant subspace of the intersection of halspaces. Then for any band $B_t:=\{\bx\in \B^d(1)\mid \bx\cdot \bw\in [t,t+c\gamma]\}$ where $t\in \R$ and $c$ is a sufficiently small constant, 
    the distribution of $(\bx,y)$ conditioned on $\bx\in B_t$ is consistent with an instance of learning a degree-$2$ polynomial threshold function with $\Omega(\gamma^2)$-margin.
\end{lemma}

\begin{proof}[Proof of \Cref{lem:banding-PTF}]
    Let $h^*(\bx)=\sign(\bu^*\cdot \bx + t_1) \wedge \sign(\bv^* \cdot \bx + t_2)$ be the target hypothesis and $V$ be the subspace spanned by $\bu^*$ and $\bv^*$ and $\bw\in S^{d-1}$ such that $\norm{\bw_V}_2 \leq c\gamma$. 
    Notice that if $V$ is a one-dimensional subspace, then the statement trivially holds for every $\bw$. Therefore, without loss of generality, we assume that $V$ is a $2$-dimensional subspace.

    Without loss of generality, we assume that $|t_1|,|t_2|\leq 1$.
    Let $\bw^*=\bw_V/\norm{\bw_V}_2$.
    Given $V$ is a 2-dimensional subspace, take $\bw'\in S^{d-1}$ to be the unique direction that $\bw'\in V$ and $\bw\cdot\bw'=0$.
    Notice that for any $\bx\in B_t$, we have     
    \begin{align*}
     &\bu^*\cdot \bx + t_1\\
    =&\proj_{\bw^*}\bu^*\cdot \bx +\proj_{\bw'}\bu^*\cdot \bx+ t_1\\
    =&\sign(\bw^*\cdot \bu^*)\|\proj_{\bw^*}\bu^*\|_2\bw^*\cdot \bx +\proj_{\bw'}\bu^*\cdot \bx+ t_1\\
    =&\sign(\bw^*\cdot \bu^*)\|\proj_{\bw^*}\bu^*\|_2\bw\cdot \bx+\sign(\bw^*\cdot \bu^*)\|\proj_{\bw^*}\bu^*\|_2(\bw^*-\bw)\cdot \bx\\
    &+\proj_{\bw'}\bu^*\cdot \bx+ t_1\; .
    \end{align*}
    Notice that $\norm{\bw_V}\leq c\gamma$ implies that $\|\bw^*-\bw\|_2\leq 2c\gamma$.
    Thus, for the second term, 
    we have \[|\sign(\bw^*\cdot \bu^*)\|\proj_{\bw^*}\bu^*\|_2(\bw^*-\bw)\cdot \bx|\leq \|\proj_{\bw^*}\bu^*\|_2\|\bw^*-\bw\|_2\|\bx\|\leq 2c\gamma.\] 
    Since the distribution satisfies $\gamma$-margin condition, we get for any $\bx\in B_t$ and on the support of $D$,
    \begin{align*}
    \sign (\bu^*\cdot \bx + t_1)
    =&\sign\left (\sign(\bw^*\cdot \bu^*)\|\proj_{\bw^*}\bu^*\|_2\bw\cdot \bx
    +\proj_{\bw'}\bu^*\cdot \bx+ t_1\right )
    \\
    =&\sign\left (\proj_{\bw'}\bu^*\cdot \bx+ (t_1+\sign(\bw^*\cdot \bu^*)\|\proj_{\bw^*}\bu^*\|_2 t)\right )\\
    =&\sign\left (\sign(\bw'\cdot \bu^*)\bw'\cdot \bx+ (t_1+\sign(\bw^*\cdot \bu^*)\|\proj_{\bw^*}\bu^*\|_2 t)/\|\proj_{\bw'}\bu^*\|_2\right )
    \; .
    \end{align*}
    Since $\|\proj_{\bw'}\bu^*\|_2\leq 1$, we have for any $\bx\in B_t$ satisfies 
    \[\sign (\bu^*\cdot \bx + t_1)=\sign (\sign(\bw'\cdot \bu^*)\bw'\cdot \bx + t_1') \]
    with $\Omega(\gamma)$ margin,
    where $t_1'=(t_1+\sign(\bw^*\cdot \bu^*)\|\proj_{\bw^*}\bu^*\|_2 t)/\|\proj_{\bw'}\bu^*\|_2$.
    By symmetry, we also have $\sign (\bv^*\cdot \bx + t_2)=\sign (\sign(\bw'\cdot \bv^*)\bw'\cdot \bx + t_2')$.
    
    Therefore, it suffices for us to show that there is a degree-2 PTF function with $\Omega(\gamma^2)$ margin that is consistent with the intersection of two halfspaces
    $f'(\bx)=\sign (\sign(\bw'\cdot \bu^*)\bw'\cdot \bx + t_1')\land \sign (\sign(\bw'\cdot \bv^*)\bw'\cdot \bx + t_2')$ with $\Omega(\gamma)$ margins.  
    Without loss of generality, we can always assume that $|t_1'|\leq 1$, $|t_2'|\leq 1$ and $\sign(\bw'\cdot \bu^*)\neq \sign(\bw'\cdot \bv^*)$. Because, otherwise, $f'(\bx)$ is equivalent to either $\sign (\sign(\bw'\cdot \bu^*)\bw'\cdot \bx + t_1')$ or $\sign (\sign(\bw'\cdot \bv^*)\bw'\cdot \bx + t_2')$. 
    Without loss of generality, assume that $\sign(\bw'\cdot \bu^*)=1$ and $\sign(\bw'\cdot \bv^*)=-1$. 
    Furthermore, we can without loss of generality assume that $-t_1'\leq t_2'$, because otherwise, $f'$ is equivalent to the $-1$ constant function.
    Given the above assumptions, we can without loss of generality assume that the function $f'$ is equivalent to 
    \[
        f(\bx)=
        \begin{cases}
			-1, & \text{for $\bw'\cdot \bx\in  (-\infty,-t_1'-c\gamma]$}\\
            1, & \text{for $\bw'\cdot \bx\in  [-t_1'+c\gamma,t_2'-c\gamma]$}\\
            -1, & \text{for $\bw'\cdot \bx\in  [t_2'+c\gamma,\infty]$}\; ,
		\end{cases}
    \]
    where from the margin condition, for any $\bx\in B_t$ and from the support of $D$ we cannot have $\bw'\cdot \bx\in [-t_1'-c\gamma,-t_1'+c\gamma]\cup [t_2'-c\gamma,-t_1'+c\gamma]$ and $c$ is a sufficiently small constant.
    Therefore, simply take the degree-2 PTF as as $\sign(p(\bx))$, where $p(\bx)=(\sign(\bw'\cdot \bu^*)\bw'\cdot \bx + t_1')(\sign(\bw'\cdot \bv^*)\bw'\cdot \bx + t_2')$ will immediately give us a function that is consistent with $f'$ with $c\gamma^2$ margins and $c$ is a sufficiently small constant. This completes the proof.

\end{proof}


\subsection{Proof of \Cref{thm:main-weak-learner}}\label{app weak learn}

In this section, we give the main weak learning algorithm as \Cref{alg:banding-weak} and present the proof of \Cref{thm:main-weak-learner}. For convenience, we restate \Cref{thm:main-weak-learner} below.

\begin{theorem} \label{thm:main-weak-learner re}
    There is an algorithm $\A$ such that for every instance of learning intersections of two halfspaces with $\gamma$-margin assumption, given $\bw\in S^{d-1}$ such that $\norm{\bw_W}_2 \leq c\gamma$ where $c$ is a sufficiently small constant, $\A$ draws $\poly(d,1/\gamma)$ examples from $D$, runs in $\poly(d,1/\gamma)$ time and outputs a hypothesis $h: \B^d(1) \to \{\pm 1\}$ such that with probability at least $2/3$, $\err(h) \le 1/2-\Omega(\gamma)$.
\end{theorem}

The algorithm in \Cref{thm:main-weak-learner re} is provided as \Cref{alg:banding-weak}.
\begin{algorithm} [h]
    \caption{Weak Learning Intersection of Halfspaces under Product Distribution using a Relevant Direction}
    \label{alg:banding-weak}
     \textbf{Input:} i.i.d.\ sample access to a distribution $D$ on $\B^d(1)\times \{\pm 1\}$ that is an instance of learning intersections of halfspaces under product distribution with $\gamma$ margins and $\bu \in S^{d}(1)$ such that $\|\proj_{V^{\perp}}\bu\|_2\leq c\gamma$ where $c$ is a sufficiently small constant.
        
    \textbf{Output:} With at least a constant probability, the algorithm outputs a hypothesis $h$ such that $\pr_{(\bx,y)\sim D}[h(\bx)\neq y]\leq 1/2-\Omega(\gamma)$.
    \vspace{0.3cm}
    \begin{algorithmic} [1]
    \State Let discrete set $T\subseteq [-1,1]$ such that $|T|\leq 2/(c_1\gamma)$ and for any $t^*\in [-1,1]$, there exists a $t\in T$ such that $|t-t^*|\leq c_1\gamma$, where $c_1$ is a sufficiently small constant depending on $c$.
    \For {$t\in T$}
    \State Let the localization area be defined as $B_t=\{\bx\in \B^d(1)\mid \bx\cdot\bu\in [t-c_1\gamma,t+c_1\gamma]\}$. 
    \State \algparbox{ Estimate $\pr_{(\bx,y)\sim D}[\bx\sim B_t]$ with $\hat{P}_t:=\frac{1}{n}\sum_{(\bx,y) \in S}\Ind(\bx \in B_t)$ by drawing a set $S$ of $\poly(1/\gamma)$ from $D$. }
    \label{line:weak-learn-band-mass-estimation}
    \If{$\hat{P}_t\geq 1/(2|T|)$} \label{line:weak-learn-if}
     \State \algparbox{Let $D_t$ be the distribution of $(V(\bx),y)\sim D$ conditioned on $\bx\in B_t$ where $V:\R^d\to \R^{(d+1)^2}$ defined as $V(\bx)=[\bx,1]^{\otimes 2}$ is the degree-$2$ Veronese mapping.}
     \State \algparbox{Use rejection sampling to get sample access to $D_t$ and apply the standard perceptron to learn a LTF with $c_2\gamma^2$ margins to additive error $1/4$ with success probability $1/2$, where $c_2$ is a sufficiently small constant depending on $c_1$, and let $h':\R^{(d+1)^2}\to \{-1,1\}$ be the output hypothesis.}\label{line:perceptron}
     \State \algparbox{
     Select the best $c'\in \{-1,1\}$ and 
     return the hypothesis $h$ defined as 
        \begin{align*}
                h(\bx) &=
                \begin{cases}
                    h'(V(\bx)), & \text{if } \bx\in B_t\; ;\\
                    c'      ,& \text{otherwise,}
                \end{cases}
        \end{align*}            
        and terminate.}

    \EndIf
    \EndFor
    \end{algorithmic}
\end{algorithm} 

\begin{proof} [Proof of \Cref{thm:main-weak-learner re}]
    From Chernoff bound, we have that the empirical estimation $\hat{P}$ in Line~\ref{line:weak-learn-band-mass-estimation} has error at most $c/|T|$ with failure probability at most $c/|T|$, where c is a sufficiently small constant.   
    Therefore, with at least constant probability, all estimations in Line~\ref{line:weak-learn-band-mass-estimation} have error at most $c/|T|$.
    Since we only need to show that the algorithm succeeds with a constant probability,
    we assume that all estimations have error at most $c/|T|$ and the algorithm in Line~\ref{line:weak-learn-band-mass-estimation} succeeds for the rest of the proof.
    
    Notice that from the definition of $T$, we get $\B^d(1)\subseteq\bigcup_{t\in T}B_t$.
    Therefore, there must be a $t\in T$ such that $\pr_{(\bx,y)\sim D}[\bx\in B_t]\ge 1/|T|$.
    Thus, the ``if'' condition in Line~\ref{line:weak-learn-if} must be satisfied at some point.
    Suppose that it is satisfied for $B_t$, then we must have $\pr_{(\bx,y)\sim D}[\bx\in B_t]\geq 1/(3|T|)$.
    Furthermore, from \Cref{lem:banding-PTF}, we must have that there is a degree-2 PTF that is consistent with the distribution $D'_t$ with $\Omega(\gamma^2)$ margins,
    where $D'_t$ is defined as the distribution of $(\bx,y)\sim D$ conditioned on $\bx\in B_t$.
    Notice that such a polynomial $p:\R^d\to \{\pm 1\}$ can be written in the form $p(\bx)=\sgn(T \cdot [\bx,1]^{\otimes 2})$ for some $\|T\|_f\leq 1$.
    Therefore, we must have that there is an LTF $h:(\R^{d+1})^{\otimes 2}\to \{\pm 1\}$ defined as $h(\bx)=\sgn(T\cdot \bx)$ that is consistent with the distribution $D_t$ with $\Omega(\gamma^2)$ margins, where the constant where depends on $c_1$.
    Then from the correctness of the perceptron algorithm (see \cite{cristianini2000introduction}), we must have that with at least constant probability (when the perceptron algorithm succeeds),
    $\pr_{(\bx,y)\sim D_t}[h'(\bx)\neq y]\leq 1/4$ .
    From the definition of $D_t$ and $D'_t$, this also implies that 
    \[\pr_{(\bx,y)\sim D'_t}[h'(V(\bx))\neq y]\leq 1/4\; .\]
    Therefore, with at least constant probability, the error of the output hypothesis is
    \begin{align*}
    \pr_{(\bx,y)\sim D}[h(\bx)\neq y]=&\pr_{(\bx,y)\sim D}[h'(V(\bx))\neq y\land \bx\in B_t]+\pr_{(\bx,y)\sim D}[c'\neq y\land \bx\not\in B_t]\\
    \leq &\frac{1}{4}\pr_{(\bx,y)\sim D}[\bx\in B_t]+\pr_{(\bx,y)\sim D}[c'\neq y\land \bx\not\in B_t]\; .
    \end{align*}
    Since we are choosing the best constant $c'\in \{-1,1\}$, with at least constant probability, we have $\pr_{(\bx,y)\sim D}[c'\neq y\land \bx\not\in B_t]\leq \frac{1}{2}\pr_{(\bx,y)\sim D}[\bx\not\in B_t]$. Combing with the fact that $\pr_{(\bx,y)\sim D}[\bx\in B_t]\geq 1/(3|T|)$ and $|T|=O(1/\gamma)$, we get 
    \[
    \pr_{(\bx,y)\sim D}[h(\bx)\neq y]\le 1/2-\Omega(\gamma)\; .
    \]
    This completes the proof.
\end{proof}

\section{Proof of \Cref{th main}}\label{app main}
In this section, we present a detailed version of our main algorithm as \Cref{alg main} and the full proof of \Cref{alg main}.

\begin{algorithm}[h]
		\caption{\textsc{Learning Intersections of Two Halfspaces} (Computationally efficient algorithm for learning intersections of two halfspaces)}\label{alg main}
		\begin{algorithmic} [1]
        
\State  \textbf{Input:} $\eps,\delta,\gamma \in (0,1)$ and \iid sample access to a distribution $D$ on $\B^d(1)\times \{\pm 1\}$ that is an instance of learning intersections of halfspaces under product distribution with $\gamma$-margin assumption. 

\State \textbf{Output:} With probability at least $1-\delta$, the algorithm outputs a hypothesis $\hat{h}: \R^d \to \{\pm 1\}$, such that $\err(\hat{h}) \le \eps$.

\State Draw $m_0 = O(1/\eps,\log(1/\delta))$ examples $S_0:=\{(\bx^{(i)},y^{(i)})\}_{i=0}^{m_0}$ from $D$ and estimate $\hat{p}: = \frac{1}{m_0}\Ind (y^{(i)} = 1)$. If $\hat{p}<\eps$ or $\hat{p}>1-\eps$, return a constant hypothesis accordingly.
 
\State For $z \in \{\pm \}$, draw $m_z=\poly(d,\gamma,\log(1/\delta))$ \iid examples 
$S_z:=\{\bx^{(i)}\}_{i=0}^{m_z}$ from $D^z$ via rejection sampling.

\State For $z \in \{\pm\}$, let $\hat{D^z}$, the uniform distribution over $S_z$ be the empirical distribution of $D^z$.

\If{ $\forall t \in [3]$, $(\hat{D^+},\hat{D^-})$ satisfies $(\alpha_t\gamma^{c_t},t)$-moment matching condition, where $\alpha_1=1/64,\alpha_2=1/16,\alpha_3=1/2, c_1 = 4c, c_2=2c,c_3=c$.}
\State Run \Cref{alg tensor} over $D$, $\poly(d,1/\gamma,\log(1/\delta))$ times and denote by $\mathcal{O}$ the union of the outputs of running \Cref{alg tensor}.
\Else 
\State Find the first $t$ such that $(\hat{D^+},\hat{D^-})$ does not satisfy $(\alpha_t\gamma^{c_t}/\poly(d),t)$-moment matching condition.
\State Run \Cref{alg csq} with parameter $t$ over $D$, $O(\log(1/\delta))$ times and denote by $\mathcal{O}$ the union of the outputs of running \Cref{alg csq}.
\EndIf

\For{$\bw \in \mathcal{O}$}
\State Run Boosting algorithm to get a hypothesis $h_{\bw}$ using the weak learning algorithm used in \Cref{sec local} using $\bw$ as the input vector.
\EndFor

\State Draw $\poly(1/\eps, \log(d/(\gamma\delta)))$ \iid examples from $D$ and find the hypothesis $\hat{h}$ from $\hat{H}=\{h_{\bw} \mid \bw \in \mathcal{O}\}$ with smallest empirical error 

\State \Return $\hat{h}$.

\end{algorithmic}
\end{algorithm}

Before presenting the proof of \Cref{th main}, it is convenient to recall the well known Adaboost algorithm developed by \cite{schapire2013boosting}.

\begin{theorem}[AdaBoost \cite{schapire2013boosting}]\label{th boost}
    Let $H$ be a binary hypothesis class over a space of example $X$. A learning algorithm $\mathcal{A}$ is said to be an $\alpha$-weak learning algorithm for $H$ if for every distribution $D$ over $H$ such that there exists some $h^* \in H$ such that $\err_D(h^*)=0$, $\mathcal{A}$ outputs in $T(\mathcal{A})$ time a hypothesis $c: X \to \{\pm 1\}$ such that $\err_D(c) \le 1/2-\alpha$ by drawing \iid examples from $D$. If the hypothesis $c$ output by $\mathcal{A}$ belongs to a hypothesis class $\mathcal{C}$ with VC-dimension $d$, then there is an algorithm AdaBoost such that for every $\epsilon,\delta$ and for every distribution $D$ over $H$ such that there exists some $h^* \in H$ with $\err_D(h^*)=0$, it draws a set $S$ of $\poly(d,1/\epsilon,1/\alpha,\log(1/\delta))$ examples from $D$,  and outputs in $\poly(d,1/\epsilon,1/\alpha,\log(1/\delta),T(\mathcal{A}))$ time a hypothesis $\hat{h}:X \to \{\pm 1\}$ with $\err_D(\hat{h}) \le \epsilon$ with probability at least $1-\delta$, by running $\mathcal{A}$, $O(\log(1/\epsilon)/\alpha^2)$ times over distributions over $S$.
\end{theorem}

\begin{proof} [Proof of \Cref{th main}]
We first prove the correctness of \Cref{alg main}.
    Notice that if $h^*$ is $\epsilon$-close to any constant hypothesis, i.e. $\min\{\Pr_{\bx\sim D}(h^*(\bx) =+1),\min\{\Pr_{\bx\sim D}(h^*(\bx) =-1)\}\le \epsilon/2$, then by Hoeffding's inequality \cite{Ver18}, with probability at least $1-O(\delta)$, $\min\{\hat{p},1-\hat{p}\}<\eps$. In this case, \Cref{alg main} outputs a constant hypothesis with error $\eps/2$ in $\poly(1/\eps,\log(1/\delta))$ time. 
    In the rest of the proof, we assume $\min\{\Pr_{x\sim D}(h^*(x) =-1)\}> \epsilon/2$. 

     Since each $\bx \sim D_X$ has $\norm{\bx}\le 1$, we know that for $z \in \{\pm 1\}$ and $k=\{1,2,3\}$, the empirical distribution $\hat{D^z}$, constructed with $m_z=\poly(d,1/\gamma,\log(1/\delta))$ \iid examples from $D^z$, satisfies $\norm{\E_{\bx \sim \hat{D^z}}(\bx)-\E_{\bx \sim D^z}(\bx)}_F \le \frac{1}{100}(\gamma/d)^{10c}$
    Given this happens, we consider two cases for the empirical distribution $(\hat{D^+},\hat{D^-})$. In the first case, $\forall t \in [3]$, $(\hat{D^+},\hat{D^-})$ satisfies $(\alpha_t\gamma^{c_t}/\poly(d),t)$-moment matching condition, where $\alpha_1=1/64,\alpha_2=1/16,\alpha_3=1/2$ and $c_t\ge c$. By Hoeffding's inequality, we know that $D$ satisfies 
\begin{enumerate}
    \item $\norm{(\E_{\bx\sim D^+}-\E_{\bx\sim D^-}) \bx}_F \le \gamma^c$. 
    \item $\norm{(\E_{\bx\sim D^+} - \E_{\bx\sim D^-})\bx\bx^\intercal}_F \le \gamma^c$
    \item $\norm{(\E_{\bx\sim D^+} - \E_{\bx\sim D^-})\bx^{\otimes 3}}_F\le \gamma^c$
\end{enumerate}
By \Cref{th tensor}, we know that with probability at least $\Omega(\gamma/d)$, \Cref{alg tensor}, outputs a list of $d$ unit vectors $\bw$ such that at least one of the $\bw$ satisfies $\norm{\bw_V}_2 \le \poly(\gamma)$. Thus, by running \Cref{alg tensor} $\poly(d,1/\gamma,\log(1/\delta))$ times, with probability at least $1-O(\delta)$, one of these implementations satisfies the above guarantee, which implies that we have a list of unit vectors $\mathcal{O}$ of size $\poly(d,1/\gamma,\log(1/\delta))$ such that one of the unit vectors $\bw$ satisfies $\norm{\bw_V}_2 \le \poly(\gamma)$.

In the second case, there must be some $t \in [3]$ such that $(\hat{D^+},\hat{D^-})$ does not satisfy the $(\alpha_t\gamma^{c_t}/\poly(d),t)$-moment matching condition. In this case, we consider the smallest $t \in [3]$ that satisfies the above condition. By the choice of $\alpha_t$, it is always holds that $\alpha_{t-1} \le 2^{-t} \alpha_t$. This implies that $D$ does not satisfy $(\alpha_t\gamma^{c_t},t)$ moment matching condition but satisfies $(\alpha_t\gamma^{2c_t}/(2^t\poly(d)),t')$-moment matching condition, for every $t'\le t$. Since $(\gamma^{c_t}) = \gamma^{\Omega (c)}$, by \Cref{thm:csq-direction-extraction}, with probability $\Omega(1)$, \Cref{alg csq} outputs a vector $\bw$ such that $\norm{\bw_V}_2 \le \poly(\gamma)$. Thus, by running \Cref{alg csq}, $\O(\log(1/\delta))$ times, we obtain a list of $O(\log(1/\delta))$ unit vectors such that at least one of the $\bw$ satisfies $\norm{\bw_V}_2 \le \poly(\gamma)$.

By \Cref{thm:main-weak-learner}, we know that provided a unit vector $\bw$ such that $\norm{\bw_V}_2 \le \poly(\gamma)$, \Cref{alg:banding-weak} is a $\Omega(\gamma)$-weak learner. And the hypothesis output by \Cref{alg:banding-weak} is a polynomial threshold function restricted at some band, which has a VC dimension $O(d)$.
Given this happens, \Cref{th boost} implies Adaboost takes \Cref{alg:banding-weak} as a weak learner and outputs a hypothesis $h_\bw$ such that $\err(h_\bw) \le \eps$ with probability at least $1-O(\delta)$. Since one of the $\bw$ satisfies $\norm{\bw_V}_2 \le \poly(\gamma)$ and $\mathcal{O}$ has size at most $\poly(d,1/\gamma,\log(1/\delta))$, a standard hypothesis testing approach outputs some $\hat{h} \in \{h_\bw \mid \bw \in \mathcal{O}\}$ with $\err(\hat{h}) \le \eps$ with probability at least $1-O(\delta)$. By union bound, with probability at least $1-\delta$, \Cref{alg main} outputs a hypothesis $\hat{h}$ with $\err(\hat{h}) \le \eps$.

To conclude the proof of \Cref{th main}, we bound the sample complexity and the time complexity of \Cref{alg main}. Given $\min\{\Pr_{x\sim D}(h^*(x) =-1)\}> \epsilon/2$, sampling one example $\bx \sim D^z$, for $z \in \{\pm\}$ has a sample complexity $\Tilde{O}(1/\eps)$. Thus, constructing the empirical distribution $(\hat{D^+},\hat{D^-})$ takes $\poly(d,1/\gamma,\log(1/\delta))$ sample and time. On the other hand, by \Cref{th tensor} and \Cref{thm:csq-direction-extraction}, every time we run \Cref{alg tensor} and \Cref{alg csq}, it takes us $\poly(d,1/\gamma,\log(1/\delta))$ sample and time. Since we run \Cref{alg tensor} and \Cref{alg csq} at most $\poly(d,1/\gamma,\log(1/\delta))$ times, we know that constructing $\mathcal{O}$ takes $\poly(d,1/\gamma,\log(1/\delta))$ sample and time. Finally, since $\mathcal{O}$ has size at most $\poly(d,1/\gamma,\log(1/\delta))$, by \Cref{thm:main-weak-learner} and \Cref{th boost}, it takes $\poly(d,1/\gamma,\log(1/\delta))$ sample and time to create $\{h_\bw \mid \bw \in \mathcal{O}\}$. Finally, as a hypothesis testing approach over $\{h_\bw \mid \bw \in \mathcal{O}\}$ can be done efficiently. We know that the sample complexity and time complexity of \Cref{alg main} are both $\poly(d,1/\gamma,\log(1/\delta))$.
\end{proof}

%% file: CSQ-lb.tex

\section{CSQ Lower Bound on Learning Intersection of Margin Halfspaces under Factorizable Distributions}\label{app lb}
We give our main theorem for CSQ hardness as the following.

\begin{theorem} \label{thm:csq-lb-main}
    Let $\gamma>0$, $q\in \N$, $\tau\in (0,1)$ and $d'=\min(d,1/\gamma^2)$. 
    Any CSQ algorithm that learns intersections of two halfspaces on $d$-dimension
    with $\gamma$-margin under factorizable distributions 
    to error $1/2-\max(d'^{-\Omega(\log(1/\gamma))},2^{-d'^{\Omega(1)}})$ requires 
    $q$ queries of tolerance at most $\tau$, 
    where $q/\tau^2\geq \min(d'^{\Omega(\log(1/\gamma))},2^{d'^{\Omega(1)}})$.  
\end{theorem}

The high-level proof idea here follows the framework of Non-Gaussian Component Analysis (see \cite{DKS17-sq} and \cite{DKRS23}).
To prove the CSQ hardness, it suffices for us to prove a CSQ lower bound against an easier decision problem as defined in \Cref{def:decision-problem}. 
For convenience, instead of considering distributions on $\B^d\times \{\pm 1\}$, we will consider distributions on $\R^d\times \{\pm 1\}$.
As we will later see, the difference here is trivial as we will be able to rescale and truncate these distributions (for our construction) inside a ball at the cost of a very small total variation distance.
We show that given CSQ access to a joint distribution $D$ of $(\bx,y)$ 
supported on $\R^d\times \{\pm 1\}$,
it is hard to solve the problem $\mathcal{B}(\D,D)$ with the following distributions.
\begin{enumerate}[leftmargin=*]

\item [(a)] Null hypothesis: We have $\bx\sim \normal(0,\bI_d)$ and $y=1$ with probability $1/2$ independent of $\bx$.

\item [(b)] Alternative hypothesis: 
$D\in \D$, where $\D$ is a family of distributions such that 
for any distribution $D\in \D$, $D$ is close in total variation distance to a distribution $D'$ that is an instance 
of learning intersections of two halfspaces with $\gamma$-margin under factorizable distributions. 
\end{enumerate}

To construct the family of distribution $\D$, we first construct a distribution $D$ of $(\bx',y)$ supported on $\mathbb{B}^{2}(1)\times \{\pm 1\}$ that is consistent with an intersection of halfspaces with $\gamma$ margin and $\E_{(\bx,y)\sim D}[yp(\bx)]=0$ for any polynomial $p$ of degree at most $O(\log (1/\gamma))$. Being supported inside the ball here will be convenient for later truncation.
We give the following lemma for distribution $D$ where the extra third property here is needed for technical reasons.

 \begin{lemma} \label{lem:csq-dist}
    Let $\gamma>0$, then there exists a joint distribution $D$ of $(\bx,y)$ supported on $\B^2(1)\times \{\pm 1\}$ that satisfied the following conditions:
    \begin{enumerate}
        \item (Realizable by an intersection of two halfspaces with $\gamma$ margins) There exists an intersection of two halfspace $h^*$ such that $D$ is realizable by $h^*$ with $\gamma$ margins; \label{itm:realizable_ball}
        \item (Orthogonal with low-degree polynomial) For any polynomial $p:\R\to \R$ of degree at most $c\log (1/\gamma)$ where $c$ is a sufficiently small constant and $\E_{\bx\sim \normal_2}[p(\bx)]=0$, we have $\E_{(\bx,y)\sim D}[yp(\bx)]=0$; \label{itm:ortho_ball}
        \item (Bounded chi-squared distance with Gaussian) $\chi^2(D^+, \normal_2),\; \chi^2(D^-, \normal_2)=O(1/\gamma)$.\label{itm:chi_ball}
    \end{enumerate}
\end{lemma}
\begin{proof} 
    To construct such a distribution $D$, we first construct a distribution $D'$ supported on $\{\pm 1\}^n\times \{\pm 1\}$. Then we obtain $D$ by projecting $D'$ onto a 2-dimensional subspace and add $\normal(0, \sigma I_2)$ noise on $\bx$ where $\sigma=\Theta(\gamma)$.
    The purpose of the extra noise is to make the originally discrete distribution continuous so we can have bounded $\chi$-squared distance.
    We introduce the following fact about such a $D'$ from \cite{Sherstov:09}.
    \begin{fact} \label{fct:dist-hypercube}
    Let $n\in \N$, then there exists a joint distribution $D$ of $(\bx,y)$ supported on $\{\pm 1\}^n\times \{\pm 1\}$ that satisfies the following conditions.
    \begin{enumerate}
        \item (Realizable by an intersection of two halfspaces) There exists an intersection of two halfspace $c$ such that $D$ is realizable by $c$ where the weight is the halfspaces is $2^{O(\sqrt{n})}$. \label{itm:realizable_hypercube}
        \item  (Orthogonal with low-degree polynomial) For any polynomial $p:\R\to \R$ of degree at most $c\sqrt{n}$ where $c$ is a sufficiently small constant and $\E_{\bx\sim_u \{\pm 1\}^n}[p(\bx)]=0$, we have $\E_{(\bx,y)\sim D}[yp(\bx)]=0$.
        \label{itm:ortho_hypercube}
    \end{enumerate}
    \end{fact}    
Suppose $\|\bw_1\|_2,\|\bw_2\|_2\leq  2^{c_1\sqrt{n}}$ in \Cref{itm:realizable_hypercube} of \Cref{fct:dist-hypercube}.
Then we set $n=\log(1/\gamma)^2/(100c_1)^2$ in \Cref{fct:dist-hypercube}, 
and let $D'$, $\bw_1$ and $\bw_2$ be the corresponding distribution and weight for the intersection of halfspaces.
This implies that $\|\bw_1\|_2,\|\bw_2\|_2\leq (1/\gamma)^{1/100}$.
We then first define the distribution $D''$ as the distribution of $\left (\frac{1}{2(1/\gamma)^{1/100}} [\bx\cdot\bb_1,\bx\cdot\bb_2]^\intercal,y\right )$, where we sample $(\bx,y)\sim D'$ and $\bb_1$ and $\bb_2$ are orthornomal basis vectors that spans the subspace spaned by $\bw_1$ and $\bw_2$.
Then we defined a independent noise random vector $\bz\in \R^2$, sampled by having $\bz\sim \normal_2(0,\gamma I_2/100)$ and then conditioned on $\|\bz\|_2\leq \gamma$.
Finally, we define the desired distribution $D$ as the distribution of $(\bx+\bz,y)$ where $(\bx,y)\sim D''$ and $\bz$ as defined above.

Notice that $D''$ has at least $10\gamma$ margins for an intersection of halfspaces, which follows from $D'$ has $\Omega(1)$ margins (given $\gamma$ is sufficiently small). 
Then the extra noise $\bz$ has $\|\bz\|_2\leq \gamma$, therefore, $D$ still have $\gamma$ margin for an intersection of halfspaces.
Furthermore, $D$ is supported inside $\B^2(1)\times \{\pm 1\}$, which follows from $\|\bw_1\|_2,\|\bw_2\|_2\leq (1/\gamma)^{1/100}$ and $\|\bz\|_2\leq \gamma$.
This proves \Cref{itm:realizable_ball} of \Cref{lem:csq-dist}.

For \Cref{itm:ortho_ball} of \Cref{lem:csq-dist},
notice that $\E_{(\bx,y)\sim D''}[yp(\bx)]=0$ for any $p$ of degree-$c\sqrt{n}$ (for $c$ a sufficiently small constant) since $D''$ is transform from $D'$ through a linear transformation.
Then \Cref{itm:ortho_ball} of \Cref{lem:csq-dist} follows from that $\bz$ are sampled independently in $D$.

For \Cref{itm:chi_ball} of \Cref{lem:csq-dist}, notice that due to the extra noise $\bz$ smoothed out the discrete distribution $D''$,
\begin{align*}
    \chi^2(D^+, \normal_2)
    =& \E_{\bx\sim D^+}\left [P_{D^+}(\bx)/P_{\normal_2}(\bx)\right ]\\
    =&\E_{\bx\sim D^+}\left [\int_{\bx'\in \R^2} P_{\bz}(\bx-\bx')P_{D''^+}(\bx')d\bx'/P_{\normal_2}(\bx)\right ]\\
    =&O(1/\gamma)\; ,
\end{align*}
where the last inequality follows from that $P_{\bz}$ is bounded everywhere by $O(1/\gamma)$ and the fact that $D^+$ is supported inside $\B^2(1)$.
The exact same argument holds for $\chi^2(D^-, \normal_2)$.
\end{proof}

To construct the family of distribution $\D$ from the 2-dimensional distribution in \Cref{lem:csq-dist}, we will pick a large set of near orthogonal 2-dimensional subspaces. For each subspace $V\in \R^{2\times d}$ in the set, we embed the distribution $D$ supported on $\B^{2}(1)\times \{\pm 1\}$ along this subspace. Namely, the distribution we create is defined as the following.

\begin{definition}[Hidden-Subspace Distribution]\label{def:hd} 
For a distribution $A$ supported on $\R^m$ and a matrix $V\in\R^{n\times m}$ 
with $V^\intercal V=I_m$, we define the distribution $\p_V^A$ supported on $\R^n$ 
such that it is distributed according to $A$ in the subspace $\mathrm{span}(\bv_1,\ldots,\bv_m$) and is an independent standard Gaussian in the orthogonal directions, where $\bv_1,\ldots,\bv_m$ denote the column vectors of $V$.
In particular, if $A$ is a continuous distribution with probability density function $A(\by)$, 
then $\p_V^A$ is the distribution over $\R^n$ with probability density function 
\[ \p_V^A(\bx)=A(\bv_1\cdot\bx,\ldots,\bv_m\cdot\bx)\exp(-\|\bx-VV^\intercal\bx\|_2^2/2)/(2\pi)^{(n-m)/2} \;.\]

Furthermore, for a distribution $A$ of $(\bx',y')$ supported on $\R^m\times \{\pm 1\}$, we define the distribution $\p_V^A$ of $(\bx,y)$ as the distribution supported on $\R^n\times \{\pm 1\}$ that satisfies the following.
\begin{enumerate}
    \item $y$ and $y'$ has the same marginal distribution; and
    \item $\left (\p_V^A\right )^-$ is $\p_V^{(A^-)}$ and $\left (\p_V^A\right )^+$ is $\p_V^{(A^+)}$.
\end{enumerate}
\end{definition}

That is, $\p^{A}_{V}$ over $\R^n\times \{\pm 1\}$ is the product distribution whose orthogonal 
projection onto the subspace of $V$ and the label space $\{\pm 1\}$ is $A$,
and onto the subspace in $\R^n$ perpendicular to $V$ is
the standard $(n-m)$-dimensional normal distribution.
For our setting, we will consider the special case of $m=2$.
We will use the subspaces $\bV\in S$ where $S$ is exponential in size.
We give the following lemma for constructing $S$.

\begin{fact} [Near-orthogonal Subspaces: Lemma 2.5 from \cite{DKPZ21}] \label{lem:near-orthogonal}
Let $0<a, c<1/2$ and $m,n\in \Z_+$ such that $m\leq n^a$. There exists a set $S$ of $2^{\Omega(n^c)}$ matrices in $\R^{m\times n}$ such that every $U\in S$ satisfies $UU^{\intercal}=I_m$ and every pair $U,V\in S$ with $U\neq V$ satisfies $\|UV^{\intercal}\|_F\leq O(n^{2c-1+2a})$.
\end{fact}

Let $S$ be the set in $\Cref{lem:near-orthogonal}$. 
We let the alternative hypothesis distribution family be defined as $\D=\{\p_{V}^{A}|V\in S\}$ where $A$ is the distribution in $\Cref{lem:csq-dist}$.
We give the following lemma for $\D$. 

\begin{lemma} \label{lem:decision-csq-lb}
For any sufficiently small $\gamma>0$, 
there exists a distribution family $\D=\{D_V:V\in S\}$ supported on $\R^d\times \{\pm 1\}$ where $|\D|=2^{d^{\Omega(1)}}$ satisfying the following:
\begin{enumerate}[leftmargin=*]
	\item \label{itm:alternative_factor} For any $D_V\in \D$,
	both $(D_V)^+$ and $(D_V)^-$ are factorizable in $V$ and $V^{\perp}$ and both $(D_V)^+_{V^{\perp}}$ and $(D_V)^-_{V^{\perp}}$ are $\normal(0,I_{d-2})$;
	
	\item \label{itm:alternative_consis} For any $D_V\in \D$, there exists a intersection of halfspaces that is consistent with $D_V$ with $\gamma$ margins. Furthermore, both halfspaces are in the subspace $V$;
        
    \item \label{itm:alternative_bounded} For any $D_V\in \D$, $\pr_{(\bx,y)\sim D_V}\left [\proj_{\perp V}(\bx)\geq 2\sqrt{d}\right ]= 2^{-d^{\Omega(1)}}$; and
	
	\item \label{itm:alternative_correlation} Let $g_{V}:\R^d\to \R$ denote the function $g_V(\bx)=(P_{\bx\sim D_V^+}(\bx)-P_{\bx\sim D_V^-}(\bx))/P_{\normal_d}(\bx)$.
	Then for any $D_U,D_V\in \D$, $ (g_U\cdot g_V)_{\normal_d}=O(1/\gamma)$ if $U = V$ and 
	$|(g_U\cdot g_V)_{\normal_d}|=d^{-\Omega(\log{1/\gamma})}$ if $U \neq V$.
\end{enumerate} 
\end{lemma}
\begin{proof}
Let $S$ be the set in $\Cref{lem:near-orthogonal}$.
We let the alternative hypothesis distribution family be defined as $\D=\{\p_V^{A}\mid V\in S\}$, where $A$ is the distribution in \Cref{lem:csq-dist}.

\Cref{itm:alternative_factor} follows immediately from the definition of $\D$. \Cref{itm:alternative_consis} follows from \Cref{itm:realizable_ball} of \Cref{lem:csq-dist}. \Cref{itm:alternative_bounded} follows from the fact that $A$ is bounded inside $\mathbb{B}^{2}(1)\times \{\pm 1\}$ and the concentration of $l_2$ norm for Gaussian. 
Namely,
\[
\pr_{(\bx,y)\sim D_v}\left [\bx\geq 2\sqrt{d}\right ]\leq \pr_{\bx\sim \normal(0,I_{d-2})}\left [\bx\geq (2\sqrt{d}-1)\right ] =\pr_{t\sim \chi^2(d-2)}[t\geq 3d]\leq 2^{-\Omega(d)}\; ,\]
where $\chi^2(d-2)$ the chi-squared distribution with $d-2$ degrees of freedom.

For \Cref{itm:alternative_correlation}, if $U=V$, then we have 
\begin{align*}
(g_U\cdot g_V)_{\normal_d}=&\E_{\bx\sim \normal_d}[g_V(\bx)^2]\\
=&\E_{\bx\sim \normal_d}\left [\left (\left (P_{\bx\sim D_V^+}(\bx)-P_{\bx\sim D_V^-}(\bx)\right )/\normal_d(\bx)\right )^2\right ]\\
=&\E_{\bx\sim \normal_d}\left [\left (P_{\bx\sim D_V^+}(\bx)/\normal_d(\bx)\right )^2\right ]
+\E_{\bx\sim \normal_d}\left [\left (P_{\bx\sim D_V^-}(\bx)/\normal_d(\bx)\right )^2\right ]\\
&+2\E_{\bx\sim \normal_d}\left [\left (P_{\bx\sim D_V^+}(\bx)/\normal_d(\bx)\right )\left (P_{\bx\sim D_V^-}(\bx)/\normal_d(\bx)\right )\right ]\\
\le &\chi^2(D_V^+,\normal_d)+\chi^2(D_V^+,\normal_d)+2\sqrt{\chi^2(D_V^+,\normal_d)\chi^2(D_V^+,\normal_d)}=O(1/\gamma)\; .
\end{align*}
 For the case $U\neq V$, we will need the following fact.
\begin{fact} [Correlation Lemma: Lemma 2.3 from \cite{DKPZ21}] \label{lem:correlation-lemma}
Let $g:\R^m\mapsto \R$ and $U,V\in \R^{m\times d}$ with $m\leq d$ be linear maps such that 
$UU^{\intercal}=VV^\intercal=I$ where $I$ is the $m\times m$ identity matrix. 
Then, we have that 
\[\E_{\x\sim \normal_d}[g(U\x)g(V\x)]\leq \sum_{t=0}^{\infty}\| UV^{\intercal}\|_2^t\E_{\x\sim \normal_m}[(g^{[t]}(\x))^2] \;,\]
where $g^{[t]}$ denote the degree-$t$ Hermite part of $g$.
\end{fact}
Using the above fact and \Cref{itm:ortho_ball}, we have
\begin{align*}
    (g_U\cdot g_V)_{\normal_d}=d^{-\Omega(c\log(1/\gamma))}\E_{\bx\sim \normal_d}[g_V(\bx)^2]=d^{-\Omega(\log(1/\gamma))}O(1/\gamma)=d^{-\Omega(\log(1/\gamma))}\; .
\end{align*}
This completes the proof.
\end{proof}

Given \Cref{lem:decision-csq-lb}, we are now ready to prove our main theorem \Cref{thm:csq-lb-main}. 
\begin{proof} [Poof for \Cref{thm:csq-lb-main}]
    Without loss of generality, we will assume that $\min(d,1/\gamma^2)=d$. Since if this is not the case, i.e., $d>1/\gamma^2$, we can always give a lower bound for $d'=1/\gamma^2$ and the lower bound immediately applies to $d>d'$ by simply adding dummy coordinates that is always 0 on the hard instance. Therefore, we just need to show a $d^{\Omega(\log (1/\gamma))}$ lower bound given $d\leq 1/\gamma^2$.

    Let $\D'$ be the alternative hypothesis distribution set in \Cref{lem:decision-csq-lb}
    with the margin parameter in \Cref{lem:decision-csq-lb} taken as $\alpha=2\gamma \sqrt{d}$.
    We defined null hypothesis distribution $D'$ as the joint distribution of $(\bx,y)$ where $\bx\sim \normal_d$ and $y=1$ independently with probability $1/2$.
    Now, consider the decision problem ${\cal B}(\D',D')$.
    From \Cref{lem:decision-csq-lb}, we have $\mathrm{CD}({\cal B},d^{-\Omega(\log (1/\alpha))},O(1/\alpha))=2^{d^{\Omega(1)}}$.
    Notice that from \Cref{lem:sq-lb}, by taking the parameter $\gamma'$ in \Cref{lem:sq-lb} as
    $\gamma'=d^{-c\log(1/\alpha)}$ for some constant $c$, we get that any CSQ algorithm for solving ${\cal B}(\D',D')$ either requires a query of tolerance $d^{-\Omega(\log(1/\alpha))}$ or $2^{d^{\Omega(1)}}$ many queries.

    The problem here is that ${\cal B}(\D',D_0')$ is supported on $\R^d\times \{\pm 1\}$ instead of $\B^d(1)\times \{\pm 1\}$.
    To fix this problem, we first truncate the distributions.
    We now define $D''$ as the distribution of $\left (\bx,y\right )$ where $(\bx,y)\sim D'\mid (\|\bx\|_2\leq 2\sqrt{d})$. Similarly, we defined $\D''$ as the family of distribution, where for each distribution $D''_V\in \D''$, we take a $D'_V\in \D'$ and defined $D''_V$ as the distribution of $\left (\bx,y\right )$ where $(\bx,y)\sim D'_V\mid (\|\proj_{\perp V}\bx\|_2\leq 2\sqrt{d})$.    
    Notice that from the definition of $\D'$ and Property~\ref{itm:alternative_bounded} of \Cref{lem:decision-csq-lb}, the total variation distance between the truncated distribution and the untruncated distribution is bounded by $2^{-d^{\Omega(1)}}$. Therefore, given the CSQ lower bound on the untruncated $\B(\D',D')$, we have that any CSQ algorithm for solving the truncated $\B(\D'',D'')$ will either require a query of tolerance $d^{-\Omega(\log(1/\alpha))}+2^{-d^{\Omega(1)}}=\max(d^{-\Omega(\log(1/\alpha))},2^{-d^{\Omega(1)}})=\max(d^{-\Omega(\log(1/\gamma))},2^{-d^{\Omega(1)}})$ or $2^{d^{\Omega(1)}}$ many queries, which follows from the definition of the CSQ oracle.

    Now, since everything is bounded inside a radius $3\sqrt{d}$ ball, we just need to rescale it so everything is inside a unit ball. From the definition of the CSQ oracle, this does not change the CSQ lower bound.
    Defined $D$ be the distribution of $(\bx/(3\sqrt{d}),y)$ where $\bx,y\sim D''$ and $\D$ such that any $D_V\in \D$ is defined as $(\bx/(3\sqrt{d}),y)$ where $\bx,y\sim D_V''$ for some $D_V''\in \D''$.     
    It is immediate that any algorithm that solves ${\cal B}(\D,D)$ requires either a query of tolerance $\max(d^{-\Omega(\log(1/\gamma))},2^{-d^{\Omega(1)}})$ or $2^{d^{\Omega(1)}}$ many queries.
    
    Furthermore, notice that any $D_V\in \D$ is an instance 
    of learning intersections of two halfspaces with $\gamma$-margin under factorizable distributions. 
    Therefore, any algorithm for learning intersections of two halfspaces with $\gamma$-margin under factorizable distributions will output a hypothesis with $\eps$ error if given such $D_i\in \D$.
    While given the null hypothesis distribution $D$, no learning algorithm can learn any hypothesis with an error nontrivially better than $1/2$. 
    Therefore, given that such a learning algorithm can solve ${\cal B}(\D,D)$, it must require either a query of tolerance $\max(d^{-\Omega(\log(1/\gamma))},2^{-d^{\Omega(1)}})$ or $2^{d^{\Omega(1)}}$ many queries.
    This completes the proof.
\end{proof}